\ifdefined\usejmlrstyle

\documentclass[twoside,11pt]{article}

\usepackage[hyperref,abbrvbib]{jmlr2e}

\jmlrheading{1}{2017}{1-48}{11/17}{11/17}{DuchiNa17}{John Duchi and Hongseok Namkoong}

\ShortHeadings{Variance-based Regularization with Convex Objectives}{Duchi and Namkoong} \firstpageno{1}

\usepackage{jmlr-jcd}
\let\cite\citep

\else
\documentclass[11pt]{article}
\usepackage[numbers]{natbib}
\usepackage{fullpage}
\usepackage{hyperref}
\usepackage{jcd}
\fi

\usepackage{pgfplotstable}
\usepackage{graphicx}
\usepackage{subcaption}
\usepackage{float}
\usepackage{lipsum}

\newcommand\blfootnote[1]{%
  \begingroup
  \renewcommand\thefootnote{}\footnote{#1}%
  \addtocounter{footnote}{-1}%
  \endgroup
}

\providecommand{\comment}[1]{}
\usepackage{extra-macros}

\begin{document}

\ifdefined\usejmlrstyle

\title{Variance-based Regularization with Convex Objectives}

\author{\name John Duchi \email jduchi@stanford.edu \\
       \addr Department of Statistics and Electrical Engineering\\
       Stanford University\\
       Stanford, CA 94305, USA
       \AND
       \name Hongseok Namkoong \email hnamk@stanford.edu \\
       \addr Department of Management Science and Engineering\\
       Stanford University\\
       Stanford, CA 94305, USA}

\editor{David Blei and Bernhard Sch{\"o}lkopf}

\maketitle

\else
\begin{center}
  {\LARGE Variance-based regularization with convex objectives} \\
  \vspace{.5cm}
  {\Large John C.\ Duchi$^1$ ~~~~ Hongseok Namkoong$^2$} \\
  \vspace{.2cm}
  {\large Stanford University} \\
  Departments of $^1$Statistics,
  $^1$Electrical Engineering, and $^2$Management Science and Engineering \\
  \vspace{.2cm}
  \texttt{\{jduchi,hnamk\}@stanford.edu}
\end{center}
\fi 

\begin{abstract}
  We develop an approach to risk minimization and stochastic optimization that
  provides a convex surrogate for variance, allowing near-optimal and
  computationally efficient trading between approximation and estimation
  error. Our approach builds off of techniques for distributionally robust
  optimization and Owen's empirical likelihood, and we provide a number of
  finite-sample and asymptotic results characterizing the theoretical
  performance of the estimator. In particular, we show that our procedure
  comes with certificates of optimality, achieving (in some scenarios)
  faster rates of convergence than empirical risk minimization
  by virtue of automatically balancing bias and variance. We
  give corroborating empirical evidence showing that in practice, the
  estimator indeed trades between variance and absolute performance on a
  training sample, improving out-of-sample (test) performance over standard
  empirical risk minimization for a number of classification problems.
\end{abstract}

\ifdefined\usejmlrstyle

\fi


\section{Introduction}

\blfootnote{Short (NIPS) version of the paper is available at
  \url{https://goo.gl/o6Y3nF}. }

We propose and study a new approach to risk minimization that automatically
trades between bias---or approximation error---and variance---or estimation
error. Let $\statdomain$ be a sample space, $P_0$ a distribution on
$\statdomain$, and $\Theta$ a parameter space. For a loss function
$\loss : \Theta \times \statdomain \to \R$, consider the problem of finding
$\theta \in \Theta$ minimizing the risk
\begin{equation}
  \label{eqn:risk}
  \risk(\theta) \defeq \E[\loss(\theta, \statrv)] = \int \loss(\theta,
  \statval) dP(\statval)
\end{equation}
given a sample $\{\statrv_1, \ldots, \statrv_n\}$ drawn i.i.d.\ according to
the distribution $P$.  Under appropriate conditions on the loss $\loss$,
parameter space $\Theta$, and random variables $\statrv$, a number of
researchers~\cite{BartlettBoMe05,BartlettJoMc06,BoucheronBoLu05,Koltchinskii06a}
have shown results of the form that with high probability,
\begin{equation}
  \label{eqn:bias-variance-bernstein}
  \risk(\theta)
  \le \frac{1}{n}
  \sum_{i = 1}^n \loss(\theta, \statrv_i)
  + C_1 \sqrt{\frac{\var(\loss(\theta, \statrv))}{n}}
  +  \frac{C_2}{n}
  ~~ \mbox{for~all~} \theta \in \Theta
\end{equation} 
where $C_1$ and $C_2$ depend on the parameters of
problem~\eqref{eqn:risk} and the desired confidence guarantee.  Such bounds
justify empirical risk minimization (ERM), which chooses $\what{\theta}_n$ to
minimize $\frac{1}{n} \sum_{i = 1}^n \loss(\theta, \statrv_i)$ over
$\theta \in \Theta$. Further, these bounds showcase a tradeoff between bias
and variance, where we identify the bias (or approximation error) with the
empirical risk $\frac{1}{n} \sum_{i = 1}^n \loss(\theta, \statrv_i)$, while
the variance arises from the second term in the bound.

Given bounds of the form above and heuristically considering the classical
``bias-variance'' tradeoff in estimation and statistical learning, it is
natural to instead choose $\theta$ to directly minimize a quantity trading
between approximation and estimation error, say of the form
\begin{equation}
  \label{eqn:empirical-plus-variance}
  \frac{1}{n}
  \sum_{i = 1}^n \loss(\theta, \statrv_i)
  + C \sqrt{\frac{\var_\emp(\loss(\theta, \statrv))}{n}},
\end{equation}
where $\var_\emp$ denotes the empirical variance of its argument.
\citet{MaurerPo09} considered precisely this idea, giving a number of
guarantees on the convergence and good performance of such a procedure.
Unfortunately, even when the loss $\loss$ is convex in $\theta$, the
formulation~\eqref{eqn:empirical-plus-variance} is in general non-convex,
yielding computationally intractable problems, which has limited the
applicability of procedures that minimize the variance-corrected empirical
risk~\eqref{eqn:empirical-plus-variance}.  In this paper, we develop an
approach based on Owen's empirical likelihood~\cite{Owen01} and ideas from
distributionally robust optimization~\cite{Ben-TalGhNe09, BertsimasGuKa14,
  Ben-TalHaKoMa15} that---whenever the loss $\loss$ is convex---provides a
tractable \emph{convex} formulation that very closely approximates the
penalized risk~\eqref{eqn:empirical-plus-variance}, and we give a number of
theoretical guarantees and empirical evidence for its performance.


Before summarizing our contributions, we first describe our approach.  Let
$\phi : \R_+ \to \R$ be a convex function with $\phi(1) = 0$. Then the
\emph{$\phi$-divergence} between distributions $P$ and $Q$ defined on a space
$\statdomain$ is
\begin{equation*}
  \phidiv{P}{Q} = \int \phi\left(\frac{dP}{dQ}\right) dQ
  = \int_\statdomain
  \phi\left(\frac{p(\statval)}{q(\statval)}\right) q(\statval) d\mu(\statval),
\end{equation*}
where $\mu$ is any measure for which $P, Q \ll \mu$, and
$p = \frac{dP}{d\mu}$, $q = \frac{dQ}{d\mu}$. Throughout this paper, we use
$\phi(t) = \half (t - 1)^2$, which gives the
$\chi^2$-divergence~\cite{Tsybakov09}. Given $\phi$ and a sample
$\statrv_1, \ldots, \statrv_n$, we define the \emph{local neighborhood of the
  empirical distribution with radius $\tol$} by
\begin{equation*}
  \mc{P}_n \defeq \left\{\mbox{distributions~}P ~ \mbox{such that}~
  \phidiv{P}{\emp} \le \frac{\tol}{n} \right\},
\end{equation*}
where $\emp$ denotes the empirical distribution of the sample,
and our choice of
$\phi(t) = \half (t - 1)^2$ means that $\mc{P}_n$ consists of discrete
distributions supported on the sample $\left\{X_i\right\}_{i=1}^n$. We then
define the \emph{robustly regularized risk}
\begin{equation}
  \label{eqn:robust-risk}
  \risk_n(\theta, \mc{P}_n)
  \defeq \sup_{P \in \mc{P}_n} \E_P[\loss(\theta, \statrv)]
  = \sup_{P} \left\{
    \E_P[\loss(\theta, \statrv)]  : \phidivs{P}{\emp} \le \frac{\tol}{n}\right\}.
\end{equation}
As it is the supremum of a family of convex functions, the robust risk $\theta
\mapsto \risk_n(\theta, \mc{P}_n)$ is convex in $\theta$ whenever $\loss$ is
convex, no matter the value of $\tol \ge 0$. Given the robust empirical
risk~\eqref{eqn:robust-risk}, our proposed estimation procedure is to choose a
parameter $\robsol$ by minimizing $\risk_n(\theta, \mc{P}_n)$.

Let us now discuss a few of the properties of procedures
minimizing the robust empirical risk~\eqref{eqn:robust-risk}. Our first main
technical result, which we show in Section~\ref{sec:variance-expansions}, is
that for bounded loss functions, the robust risk $\risk_n(\theta, \mc{P}_n)$
is a good approximation to the variance-regularized
quantity~\eqref{eqn:empirical-plus-variance}. That is,
\begin{equation}
  \label{eqn:variance-expansion}
  \risk_n(\theta, \mc{P}_n)
  = \E_\emp[\loss(\theta, \statrv)]
  + \sqrt{\frac{2 \tol \var_\emp(\loss(\theta, \statrv))}{n}}
  + \varepsilon_n(\theta),
\end{equation}
where $\varepsilon_n(\theta) \le 0$ and is $O_P(1/n)$ uniformly in $\theta$. We
show specifically that whenever $\loss(\theta, \statrv)$ has
suitably large variance, with high probability
we have $\varepsilon_n = 0$.
From variance expansions of the form~\eqref{eqn:variance-expansion} and
empirical Bernstein inequality~\eqref{eqn:bias-variance-bernstein}, we see
that $\risk_n(\theta, \mc{P}_n)$ is a $O(1/n)$-approximation to the population
risk $\risk(\theta)$, in contrast to the cruder $O(1/\sqrt{n})$-approximation
that the empirical risk $\E_{\emp}[\loss(\theta; X)]$ provides. Based on this
intuition that the robustly regularized risk $\risk_n(\theta; \mc{P}_n)$ is a
tighter approximation to the population risk $\risk(\theta)$, we show a number
of finite-sample convergence guarantees for the estimator
\begin{equation}
  \label{eqn:solve-robust}
  \robsol \in \argmin_{\theta \in \Theta}
  \left\{\sup_P 
  \left\{\E_P[\loss(\theta, \statrv)] 
  : \phidiv{P}{\emp} \le \frac{\tol}{n} \right\}
  \right\}
\end{equation}
that are often tighter than those available for ERM (see
Section~\ref{sec:optimal-selection}). The above problem is a \textit{convex}
optimization problem when the original loss $\loss(\cdot; \statrv)$ is convex
and $\Theta$ is a convex set.

Based on the expansion~\eqref{eqn:variance-expansion}, solutions
$\robsol$ of problem~\eqref{eqn:solve-robust} enjoy automatic
finite sample optimality certificates: for $\tol \ge 0$, with probability at
least $1 - C_1 \exp(-\tol)$ we have
\begin{equation*}
  \risk(\robsol) = E[\loss(\robsol; \statrv)]
  \le \risk_n(\robsol; \mc{P}_n)
  + \frac{C_2\tol}{n}
  = \inf_{\theta \in \Theta} \risk_n(\theta, \mc{P}_n)
  + \frac{C_2 \tol}{n}
\end{equation*}
where $C_1, C_2$ are constants (which we specify) that depend on the loss
$\loss$ and domain $\Theta$. That is, with high probability the robust
solution has risk no worse than the optimal finite sample robust objective up
to an $O(\tol / n)$ error term. To guarantee a desired level
of risk performance with probability $1 - \delta$, we may specify
the robustness penalty $\tol = O(\log \frac{1}{\delta})$.

Secondly, we show that the procedure~\eqref{eqn:solve-robust} allows us to
automatically and near-optimally trade between approximation and estimation
error (bias and variance), so that
\begin{align}
  \label{eqn:intro-fast-convergence}
  \risk(\robsol) = E[\loss(\robsol; \statrv)]
  & \le \inf_{\theta \in \Theta}
    \left\{\E[\loss(\theta; \statrv)]
    + 2 \sqrt{\frac{2 \tol}{n} \var(\loss(\theta; \statrv))}\right\}
    + \frac{C \tol}{n}
\end{align}
with high probability. When there are parameters $\theta$ with small risk
$\risk(\theta)$ and small variance $\var(\loss(\theta, \statrv))$, this
guarantees that the excess risk
$\risk(\robsol) - \inf_{\theta\in \Theta}\risk(\theta)$ is essentially of
order $O(\tol / n)$, where $\tol$ governs our desired confidence level. Our
bounds do not require the Bernstein-type condition
$\var(\loss(\theta; X)) \le \zbound \risk(\theta)$ often required for
ERM. Since it is often the case that $\zbound$ depends on global information
(e.g. size of parameter space $\Theta$), we have
$\var(\loss(\theta; X)) \ll \zbound \risk(\theta)$, in which case the
bound~\eqref{eqn:intro-fast-convergence} offers a tighter guarantee than that
available for the ERM solution $\ermsol$. In particular, we give an explicit
example in Section~\ref{section:beat-erm} where our robustly regularized
procedure~\eqref{eqn:solve-robust} converges at rate $O(\log n/n)$ compared to
$O(1/\sqrt{n})$ of empirical risk minimization.

Bounds that trade between risk and variance are known in a number
of cases in the empirical risk minimization literature~\cite{MammenTs99,
  Tsybakov04, BartlettBoMe05, BoucheronBoLu05, BartlettJoMc06,
  BoucheronLuMa13, Koltchinskii06a}, which is relevant when one wishes to
achieve ``fast rates'' of convergence for statistical learning algorithms
(that is, faster than the $O(1 / \sqrt{n})$ guaranteed by a number of uniform
convergence results~\cite{BartlettMe02, BoucheronBoLu05, BoucheronLuMa13}). In
many cases, however, such tradeoffs require either conditions such as the
Mammen and Tsybakov's noise condition~\cite{MammenTs99, BoucheronBoLu05} or
localization results made possible by curvature conditions that relate the
loss/risk and variance~\cite{BartlettJoMc06, BartlettBoMe05, Mendelson14}.
The robust solutions~\eqref{eqn:solve-robust} enjoy a different
tradeoff between variance and risk than that in this literature, but
essentially without conditions except compactness of $\Theta$.

In proposing any new estimator, it is essential to understand the limits of
the proposed procedure and identify situations in which its performance may
be worse than existing estimators. There are indeed situations in which
minimizing the robust-regularized risk~\eqref{eqn:robust-risk} yields some
inefficiency (for example, in classical statistical estimation problems with
correctly specified model).  To understand limits of the inefficiency
induced by using the distributionally-robustified
estimator~\eqref{eqn:solve-robust}, in Section~\ref{sec:nothing-is-bad} we
study explicit finite sample properties of the robust estimator for general
stochastic optimization problems, and we also provide asymptotic normality
results in classical problems. There are a number of situations, based on
growth conditions on the population risk $\risk$, when convergence rates
faster than $1 / \sqrt{n}$ (or even $1/n$) are attainable
(see~\citet[Chapter 5]{ShapiroDeRu09}). We show that under these conditions,
the robust procedure~\eqref{eqn:solve-robust} still enjoys (near-optimal)
fast rates of convergence, similar to empirical risk minimization (also
known as sample average approximation in the stochastic programming
literature).  Our study of asymptotics makes precise the asymptotic
efficiency loss of the robust procedure over minimizing the standard
(asymptotically optimal) empirical expectation: there is a bias term that
scales as $\sqrt{\tol / n}$ in the limiting distribution of $\robsol$,
though its variance is optimal.

We complement our theoretical results in Section~\ref{sec:experiments}, where
we conclude by providing three experiments comparing empirical risk
minimization strategies to robustly-regularized risk
minimization~\eqref{eqn:solve-robust}.  These results validate our theoretical
predictions, showing that the robust solutions are a practical alternative to
empirical risk minimization. In particular, we observe that the robust
solutions outperform their ERM counterparts on ``harder'' instances with
higher variance.  In classification problems, for example, the robustly
regularized estimators exhibit an interesting tradeoff, where they improve
performance on rare classes (where ERM usually sacrifices performance to
improve the common cases---increasing variance slightly) at minor cost in
performance on common classes.



\subsection*{Related Work}

The theoretical foundations of empirical risk minimization are
solid~\cite{Vapnik98, BartlettMe02, BoucheronBoLu05, BoucheronLuMa13}.  When
the expectation of the excess loss bounds its variance, it is possible to
achieve faster rates than the $O(1/\sqrt{n})$ offered by standard uniform
convergence arguments~\cite{VapnikCh71, VapnikCh74, BartlettJoMc06,
  Koltchinskii06a, BoucheronLuMa13} (see \citet[Section 5]{BoucheronBoLu05}
for an overview in the case of classification, and \citet[Chapter
  5.3]{ShapiroDeRu09} for more general stochastic optimization problems).
\citet{VapnikCh71,VapnikCh74} first provided such results in the context of
$\{0,1\}$-valued losses for classification (see also~\cite{AnthonySh93}),
where the expectation of the loss always upper bounds its variance, so that
if there exists a perfect classifier the convergence rates of empirical risk
minimization procedures are $O(1 / n)$.  Mammen and
Tsybakov~\cite{MammenTs99,Tsybakov04} give low noise conditions for binary
classification substantially generalizing these results, which yield a
spectrum of fast rates. Under related conditions, \citet*{BartlettJoMc06}
show similar fast rates of convergence for convex risk minimization under
appropriate curvature conditions on the loss.  The robust
procedure~\eqref{eqn:solve-robust}, on the other hand, is guaranteed to
provide an at most $O(1/n)$ over-estimate of the population risk and a small
increase of its variance regularized population counterpart.  It may be the
case that the variance-regularized risk $\inf_\theta \{\risk(\theta) +
\sqrt{\var(\loss(\theta, \statrv)) / n}\}$ decreases to $\risk(\theta\opt)$
more slowly than $1/n$. As we note above and detail in
Section~\ref{sec:nothing-is-bad}, however, in stochastic optimization
problems the variance-regularized approach~\eqref{eqn:solve-robust} suffers
limited degradation with respect to empirical risk minimization strategies,
even under convexity and curvature properties that allow faster rates of
convergence than those achievable in classical regimes, as detailed
by~\cite[Chapter 5.3]{ShapiroDeRu09}.

Most related to our work is that of \citet{MaurerPo09}, who propose directly
regularizing empirical risk minimization by variance, providing guarantees
similar to ours and giving a natural foundation off of which many of our
results build. In their setting, however---as they carefully note---it is
unclear how to actually solve the variance-regularized problem, as it is
generally non-convex.  \citet{ShivaswamyJe10,ShivaswamyJe11} build on this
and develop an elegant approach for boosting binary classifiers based on a
variance penalty applied to the exponential loss; as it is a boosting
approach, their approach provides a coordinate-wise strategy for decreasing
the loss, but it is not guaranteed to converge to a global minimizer and
applies to classification-like problems. Our approach, handling general
stochastic optimization problems, removes these obstructions.

The robust procedure~\eqref{eqn:solve-robust} is based on
distributionally robust optimization ideas that many researchers have
developed~\cite{Ben-TalHeWaMeRe13, BertsimasGuKa14, LamZh15}, where
the goal (as in robust optimization more broadly~\cite{Ben-TalGhNe09})
is to protect against all deviations from a nominal data model. In the
optimization literature, there is substantial work on tractability of
the problem~\eqref{eqn:solve-robust}, including that of
\citet{Ben-TalHeWaMeRe13}, who show that the dual
of~\eqref{eqn:robust-risk} often admits a standard form (such as a
second-order cone problem) to which standard polynomial-time interior
point methods can be applied. \citet{NamkoongDu16} develop
stochastic-gradient-like procedures for solving the
problem~\eqref{eqn:solve-robust}, which efficiently provide low
accuracy solutions (which are still sufficient for statistical tasks).
Work on the statistical analysis of such procedures is nascent;
\citet*{BertsimasGuKa14} and \citet{LamZh15} provide confidence
intervals for solution quality under various conditions, and
\citet{DuchiGlNa16} give asymptotics showing that the optimal robust
risk $\risk_n(\robsol; \mc{P}_n)$ is a calibrated upper
confidence bound for
$\inf_{\theta \in \Theta} \E[\loss(\theta; \statrv)]$. They and
\citet{GotohKiLi15} also provide a number of asymptotic results
showing relationships between the robust risk
$\risk_n(\theta; \mc{P}_n)$ and variance regularization, but they do
not leverage these results for guarantees on the solutions
$\robsol$.

\paragraph{Notation}

We collect our notation here. We let $\ball$ denote a unit norm ball
in $\R^d$, $\ball = \{\theta \in \R^d : \norm{\theta} \le 1\}$, where
$d$ and $\norm{\cdot}$ are generally clear from context. Given sets
$A \subset \R^d$ and $B \subset \R^d$, we let
$A + B = \{a + b : a \in A, b \in B\}$ denote Minkowski addition. For
a convex function $f$, the subgradient set $\partial f(x)$ of $f$ at
$x$ is
$\partial f(x) = \{g : f(y) \ge f(x) + g^\top (y - x) ~
\mbox{for~all~} y\}$. For a function $h : \R^d \to \R$, we let $h^*$
denote its Fenchel (convex) conjugate,
$h^*(y) = \sup_x \{y^\top x - h(x)\}$.  For sequences $a_n, b_n$, we
let $a_n \lesssim b_n$ denote that there is a numerical constant
$C < \infty$ such that $a_n \le C b_n$ for all $n$. For a sequence of
random vectors $X_1, X_2, \ldots$, we let $X_n \cd X_\infty$ denote
that $X_n$ converges in distribution to $X_\infty$. For a nonegative
sequence $a_1, a_2, \ldots$, we say $X_n = O_P(a_n)$ if
$\lim_{c \to \infty} \sup_n \P(\norm{X_n} \ge c a_n) = 0$, and we say
$X_n = o_P(a_n)$ if
$\lim_{c \to 0} \limsup_n \P(\norm{X_n} \ge c a_n) = 0$.


\section{Variance Expansion}
\label{sec:variance-expansions}


\begin{figure}[t] 
  \centering
  \includegraphics[width=.7\columnwidth]{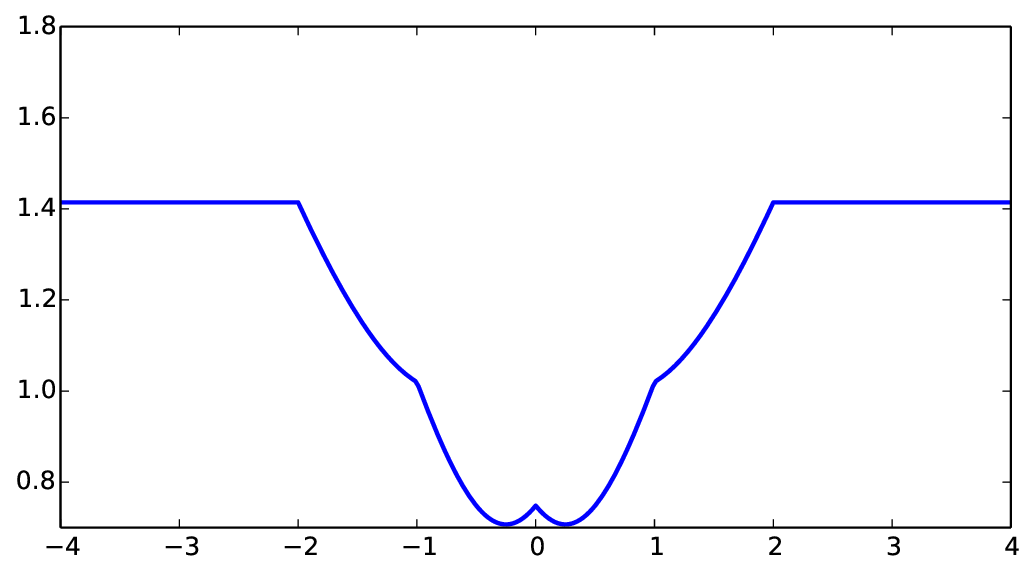} 
  \caption{Plot of $\theta \mapsto \sqrt{\var(\loss(\theta, \statrv))}$
    for
    $\loss(\theta; \statrv) = |\theta - \statrv|$ where
    $\statrv \sim \uniform(\{-2, -1, 0, 1, 2\})$. The function
    is non-convex, with multiple local minima, inflection points,
    and does not grow as $\theta \to \pm\infty$.}
  \label{fig:var-reg-nonconvex} 
  \vspace{0ex}
\end{figure}

We begin our study of the robust regularized empirical risk $\risk_n(\theta,
\mc{P}_n)$ by showing that it is a good approximation to the empirical risk
plus a variance term, that is, studying the variance
expansion~\eqref{eqn:variance-expansion}.  Although the variance of the loss
is in general non-convex (see Figure~\ref{fig:var-reg-nonconvex} for a simple
example), the robust formulation~\eqref{eqn:solve-robust} is a convex
optimization problem for variance regularization whenever the loss
function is convex (the supremum of convex functions is
convex~\cite[Prop. 2.1.2.]{HiriartUrrutyLe93ab}).

\subsection{Variance expansion for a single variable}

To gain intuition for the variance expansion that follows, we begin with
a slightly simpler problem, which is to study the quadratically
constrained linear maximization problem
\begin{equation}
  \label{eqn:simple-problem}
  \maximize_p~ \sum_{i = 1}^n p_i z_i
  ~~ \subjectto ~ p \in \mc{P}_n = \left\{p \in \R^n_+ :
  \half \ltwo{n p - \ones}^2 \le \tol, \<\ones, p\> = 1 \right\},
\end{equation}
where $z \in \R^n$ is a vector. For simplicity, let $s_n^2 = \frac{1}{n}
\ltwo{z}^2 - (\wb{z})^2 = \frac{1}{n} \ltwo{z - \wb{z}}^2$ denote the
empirical ``variance'' of the vector $z$, where $\wb{z} = \frac{1}{n} \<\ones,
z\>$ is the mean value of $z$. Then by introducing the variable $u = p -
\frac{1}{n} \ones$, the objective in problem~\eqref{eqn:simple-problem}
satisfies $\<p, z\> = \wb{z} + \<u, z\> = \wb{z} + \<u, z - \wb{z}\>$
because $\<u, \ones\> = 0$. Thus
problem~\eqref{eqn:simple-problem} is equivalent to solving
\begin{equation*}
  \maximize_{u \in \R^n}
  ~ \wb{z} + \<u, z - \wb{z}\>
  ~~ \subjectto ~
  \ltwo{u}^2 \le \frac{2 \tol}{n^2}, ~
  \<\ones, u\> = 0, ~ u \ge - \frac{1}{n}.
\end{equation*}
Notably, by the Cauchy-Schwarz inequality, we have $\<u, z - \wb{z}\> \le
\sqrt{2 \tol} \ltwo{z - \wb{z}} / n = \sqrt{2 \tol s_n^2 / n}$, and
equality is attained if and only if
\begin{equation*}
  u_i = \frac{\sqrt{2 \tol} (z_i - \wb{z})}{n \ltwo{z - \wb{z}}}
  = \frac{\sqrt{2 \tol} (z_i - \wb{z})}{n \sqrt{n s_n^2}}.
\end{equation*}
It is possible to choose such $u_i$ while
satisfying the constraint $u_i \ge -1/n$ if and only if
\begin{equation}
  \label{eqn:when-choosing-u-is-possible}
  \min_{i \in [n]}
  \frac{\sqrt{2 \tol} (z_i - \wb{z})}{\sqrt{n s_n^2}}
  \ge -1.
\end{equation}
Thus, if inequality~\eqref{eqn:when-choosing-u-is-possible} holds for the
vector $z$---that is, there is enough variance in $z$---we have
\begin{equation*}
  \sup_{p \in \mc{P}_n} \<p, z\>
  = \wb{z} + \sqrt{\frac{2 \tol s_n^2}{n}}.
\end{equation*}
For losses $\loss(\theta, \statrv)$ with enough variance relative to
$\loss(\theta, \statrv_i) - \E_\emp[\loss(\theta, \statrv_i)]$, that is, those
satisfying inequality~\eqref{eqn:when-choosing-u-is-possible}, then, we have
\begin{equation*}
  \risk_n(\theta, \mc{P}_n) = \E_\emp[\loss(\theta, \statrv)]
  + \sqrt{\frac{2 \tol \var_\emp(\loss(\theta, \statrv))}{n}}.
\end{equation*}
A slight elaboration of this argument, coupled with the application of a few
concentration inequalities, yields the next theorem.  The theorem as stated
applies only to bounded random variables, but in subsequent sections we relax
this assumption by applying the
characterization~\eqref{eqn:when-choosing-u-is-possible} of the exact
expansion.  As usual, we assume that $\phi(t) = \half(t - 1)^2$ in our
definition of the $\phi$-divergence.
\begin{theorem}
  \label{theorem:variance-expansion}
  Let $Z$ be a random variable taking values in $[\zbound_0, \zbound_1]$, and
  let $\zbound = \zbound_1 - \zbound_0$.  Let $\sigma^2 = \var(Z)$ and $s_n^2
  = \E_\emp[Z^2] - \E_\emp[Z]^2$ denote the population and sample variance of
  $Z$, respectively. Fix $\tol \ge 0$. Then
  \begin{equation}
    \label{eqn:sure-variance-bounds}
    \hinge{\sqrt{\frac{2 \tol}{n} s_n^2}
      - \frac{2 \zbound \tol}{n}}
    \le 
    \sup_{P} \left\{ \E_P[Z]
     : \phidivs{P}{\emp} \le \frac{\tol}{n} \right\}
    - \E_\emp[Z]
    \le 
    \sqrt{\frac{2 \tol}{n} s_n^2}.
  \end{equation}
  Moreover, for
  $n \ge \max \left\{5, \frac{\zbound^2}{\sigma^2} \max \left\{8 \sigma, 44
    \right\} \right\}$, with probability at least
  $1 - \exp\left(-\frac{n\sigma^2}{11 \zbound^2}\right)$
  \begin{equation}
    \label{eqn:exact-variance-expansion}
    \sup_{P : \phidivs{P}{\emp} \le \frac{\tol}{n}} \E_P[Z]
    = \E_\emp[Z] + \sqrt{\frac{2 \tol}{n} s_n^2}.
  \end{equation}
\end{theorem}
\noindent See Section~\ref{section:proof-of-variance-expansion} for
the proof of Theorem~\ref{theorem:variance-expansion}.

Inequality~\eqref{eqn:sure-variance-bounds} and the exact
expansion~\eqref{eqn:exact-variance-expansion} show that, at least for bounded
loss functions $\loss$, the robustly regularized risk~\eqref{eqn:robust-risk}
is a natural (and convex) surrogate for empirical risk plus standard deviation
of the loss, and the robust formulation approximates exact variance
regularization with a convex penalty. In the sequel, we leverage this
result to provide sharp guarantees for a number of stochastic risk
minimization problems.

\subsection{Uniform variance expansions}

We now turn to a more uniform variant
Theorem~\ref{theorem:variance-expansion}, which depends on familiar notions
of function complexity based on Rademacher averages. For a sample $x_1,
\ldots, x_n$ and i.i.d.\ random signs $\varepsilon_i \in \{-1,1\}$,
independent of the $x_i$, the empirical Rademacher complexity of the class
$\mc{F}$ is
\begin{equation*}
  \radcomp_n(\mc{F}) \defeq \E\left[\sup_{f \in \mc{F}}
    \frac{1}{n} \sum_{i = 1}^n \varepsilon_i f(x_i) \right].
\end{equation*}
The \emph{worst-case} Rademacher
complexity~\cite{SrebroSrTe10} is
\begin{equation*}
  \worstrad(\mc{F}) \defeq \sup_{x_1, \ldots, x_n \in \mc{X}}
  \E\left[\sup_{f \in \mc{F}}
    \left|\frac{1}{n} \sum_{i = 1}^n
    \varepsilon_i f(x_i)\right| \right].
\end{equation*}
For example, when $\fclass$ is a class of functions bounded by $\zbound$
with VC-subgraph dimension $d$,
we have the inequalities $\E[\radcomp_n(\mc{F})] \le \worstrad(\mc{F})
\lesssim \zbound \sqrt{\frac{d}{n}}$. See~\citet[Chapter
  2]{VanDerVaartWe96} and \citet{BartlettMe02} for other bounds.

With this definition, we provide a result showing that the variance
expansion~\eqref{eqn:variance-expansion} holds uniformly for all functions
with \emph{enough} variance.
\begin{theorem}
  \label{theorem:uniform-variance-expansion}
  Let $\fclass$ be a collection of bounded functions $f : \statdomain \to
  [\zbound_0, \zbound_1]$ where $\zbound = \zbound_1 - \zbound_0$, and
  $\zbound \le n$. There exists a universal constant $C$ such that if
  $\stdevthresh^2 > 0$ satisfies
  \begin{equation*}
    \stdevthresh^2 \ge \frac{4\tol\zbound^2}{n}
    + C \left[ \worstrad (\mc{F})^2 \log^3 n
      + \frac{\zbound^2}{n} (t + \log \log n)
      \right].
  \end{equation*}
  Then with probability at least $1-3e^{-t}$
  \begin{equation}
    \label{eqn:uniform-exact-variance-expansion}
    \sup_{P : \phidivs{P}{\emp} \le \frac{\tol}{n}}
    \E_P[f(\statrv)]
    = \E_\emp[f(\statrv)]
    + \sqrt{\frac{2 \tol}{n} \var_\emp(f(\statrv))}
  \end{equation}
  for all $f \in \mc{F}$ such that $\var(f) \ge \stdevthresh^2$.
\end{theorem}
\noindent We prove the theorem in
Section~\ref{section:proof-of-uniform-variance-expansion}.
Theorem~\ref{theorem:uniform-variance-expansion} shows that the variance
expansion of Theorem~\ref{theorem:variance-expansion} holds uniformly for
all functions $f$ with sufficient variance. An asymptotic analogue of the
equality~\eqref{eqn:uniform-exact-variance-expansion} for heavier tailed
random variables is also possible~\cite{DuchiGlNa16}.  In the remainder of
the section, we consider examples and applications to make the theorem
somewhat clearer.

\subsubsection{Linear and margin-based losses}

Consider a standard margin-based classification problem~\cite{BartlettMe02},
where we have data pairs $(x, y) \in \mc{X} \times \{-1, 1\}$, and $\mc{X}
\subset \R^d$. Let $\Theta \subset \R^d$ be a norm ball of radius
$\radius(\Theta)$, $\Theta = \{\theta \in \R^d \mid \norm{\theta} \le
\radius\}$, and let $\dnorm{\cdot}$ be the associated dual norm, assuming
also that $\mc{X} \subset \{x \in \R^d \mid \dnorm{x} \le
\radius(\mc{X})\}$. We may then consider the standard loss minimization
setting, where for some non-increasing and $1$-Lipschitz loss $\loss : \R
\to \R_+$, we have the risk
\begin{equation*}
  \risk(\theta) \defeq \E\left[\loss(Y \<\theta, X\>)\right],
\end{equation*}
so that $\loss(y \<x, \theta\>)$ is the loss suffered by making prediction
$\<\theta, x\>$ when the label is $y$. By taking the function class $\fclass
= \{(x, y) \mapsto \loss(y \<x, \theta\>) - \loss(0) \mid \theta \in
\Theta\}$, in this case, an application of the Ledoux-Talagrand contraction
inequality~\cite{LedouxTa91} implies for any $y_1, x_1, \ldots, y_n, x_n$
that
\begin{align}
  \E\left[\sup_{\theta \in \Theta}
    \left|\sum_{i = 1}^n \varepsilon_i 
    \left[\loss(y_i \<\theta, x_i\>) - \loss(0)\right]\right|\right]
  & \le
  \E\left[\sup_{\theta \in \Theta}
    \left|\sum_{i = 1}^n \varepsilon_i \<\theta, x_i\>\right|\right]
  \le \radius(\Theta)
  \E\left[\dnormbigg{\sum_{i = 1}^n \varepsilon_i x_i}\right].
  \label{eqn:simple-rademacher-bound}
\end{align}

\begin{example}[Euclidean norms]
  \label{example:euclidean-margins}
  In the above context, suppose that norm $\norm{\cdot}$ is the standard
  $\ell_2$ Euclidean norm so that $\Theta$ is contained in an $\ell_2$-ball of
  radius $\radius(\Theta)$, and $\mc{X} \subset \R^d$ in an $\ell_2$ ball of
  radius $\radius(\mc{X})$. Then Jensen's inequality and independence of
  $\varepsilon_i$'s give the bound
  \begin{equation*}
    \E[\norms{\sum_{i = 1}^n \varepsilon_i x_i}]
    \le \sqrt{\E \sum_{j=1}^d \left( \sum_{i = 1}^n \varepsilon_i x_{ij} \right)^2}
    \le \radius(\mc{X})
    \sqrt{n}.
  \end{equation*}
  Then, inequality~\eqref{eqn:simple-rademacher-bound} and
  Theorem~\ref{theorem:variance-expansion} imply that
  \begin{equation*}
    \sup_{P : \phidivs{P}{\emp} \le \frac{\tol}{n}}
    \E_P[\loss(Y \<\theta, X\>)]
    = \E_\emp[\loss(Y \<\theta, X\>)]
    + \sqrt{\frac{2 \tol}{n}
      \var_\emp(\loss(Y \<\theta, X\>))}
  \end{equation*}
  for all $\theta$ satisfying
  \begin{equation*}
    \var(\loss(Y\<\theta, X\>)) \ge \frac{\radius(\mc{X})^2 \radius(\Theta)^2}{n}
    \left[4 \tol + C \log^3 n + C t\right],
  \end{equation*}
  with probability at least $1 - e^{-t}$.
\end{example}

\begin{example}[High-dimensional problems]
  In high dimensional problems, the Euclidean scaling of
  Example~\ref{example:euclidean-margins} may be problematic, so that using
  $\ell_1$-constraints is preferred~\cite{BuhlmannGe11}.  Thus, taking the
  norm $\norm{\cdot}$ in the preceding to be the $\ell_1$ norm, so that
  $\Theta \subset \{\theta \in \R^d \mid \lone{\theta} \le \radius_1(\Theta)\}$
  and $\dnorm{\cdot} = \linf{\cdot}$, then $\E[\linfs{\sum_{i = 1}^n
      \varepsilon_i x_i}] \le \radius(\mc{X}) \sqrt{n \log(2d)}$, where
  $\radius_\infty(\mc{X})$ denotes the $\ell_\infty$-radius of $\mc{X}
  \subset \R^d$.  Thus, if we take the loss class $\mc{F} =
  \{\loss(\<\theta, \cdot\>) - \loss(0) \mid \theta \in \Theta\}$, we obtain
  \begin{equation*}
    \worstrad(\mc{F})
    \lesssim \sup_{x_1, \ldots, x_n \in \mc{X}} \frac{\radius_1(\Theta)}{n}
    \E\left[\linfbigg{\sum_{i = 1}^n
        \varepsilon_i x_i}\right]
    \le \radius_1(\Theta) \radius_\infty(\mc{X})
    \sqrt{\frac{\log (2d)}{n}}.
  \end{equation*}
  Then the exact variance
  expansion~\eqref{eqn:uniform-exact-variance-expansion} holds with
  probability at least $1 - e^{-t}$ uniformly over $\theta$ satisfying
  $\var(\loss(Y \<\theta, X\>)) \ge
  \frac{\radius_1(\Theta)^2 \radius_\infty(\mc{X})^2}{n}
  [4 \tol + C \log d \cdot \log^3 n + Ct]$.
\end{example}

\subsubsection{Covering number guarantees}
\label{sec:covering-numbers-from-rad}

It is also possible to provide guarantees on the exact variance expansion
using standard covering numbers, though careful arguments based on
Rademacher complexity can be tighter.  We begin by recalling the appropriate
notions from approximation theory. Let $\mc{V}$ be a vector space and $V
\subset \mc{V}$ be any collection of vectors in $\mc{V}$. Let $\norm{\cdot}$
be a (semi)norm on $\mc{V}$. We say a collection $v_1, \ldots, v_\covnum
\subset \mc{V}$ is an \emph{$\epsilon$-cover} of $\mc{V}$ if for each $v \in
\mc{V}$, there exists $v_i$ such that $\norm{v - v_i} \le \epsilon$. The
\emph{covering number} of $V$ with respect to $\norm{\cdot}$ is then
\begin{equation*}
  \covnum(V, \epsilon, \norm{\cdot}) \defeq
  \inf\left\{\covnum \in \N : ~ \mbox{there~is~an~}
  \epsilon \mbox{-cover~of~} V
  ~ \mbox{with~respect~to~} \norm{\cdot} \right\}.
\end{equation*}
Now, let $\fclass$ be a collection of functions $f : \statdomain \to \R$,
and define the $L^\infty(\statdomain)$ norm on $f$ by
\begin{equation*}
  \linfstatnorm{f - g}
  \defeq \sup_{\statval \in \statdomain}
  |f(\statval) - g(\statval)|.
\end{equation*}
We also relax our covering number requirements to empirical
$\ell_\infty$-covering numbers as follows. Define $\fclass(\statval) =
\{(f(\statval_1), \ldots, f(\statval_n)) : f \in \fclass\}$ for $\statval
\in \statdomain^n$, and define the empirical $\ell_\infty$-covering numbers
\begin{equation*}
  \covnum_{\infty} (\fclass, \epsilon, n)
  = \sup_{\statval \in \statdomain^n}
  \covnum\left(\fclass(\statval), \epsilon, \linf{\cdot}\right),
\end{equation*}
which bound the number of $\ell_\infty$-balls of radius $\epsilon$ required
to cover $\fclass(\statval)$. Note that we always have
$\covnum_\infty(\fclass, \epsilon, n) \le \covnum(\fclass, \epsilon,
\linfstatnorm{\cdot})$ by definition.  The classical Dudley entropy
integral~\cite{Dudley99,VanDerVaartWe96} shows that, if $P_n$ denotes the
point masses on $x_1, \ldots, x_n$ and $\ltwopn{\cdot}$ the empirical
$L^2$-norm on functions $f : \mc{X} \to [-\zbound, \zbound]$, then
\begin{align}
  \E\left[\frac{1}{n} \sup_{f \in \fclass}
    \left|\sum_{i = 1}^n \varepsilon_i f(x_i)\right|\right]
  & \lesssim \inf_{\delta \ge 0}
  \left\{\delta + \frac{1}{\sqrt{n}}
  \int_\delta^\zbound \sqrt{\log N(\mc{F}, \epsilon, \ltwopn{\cdot})} d\epsilon
  \right\} \nonumber \\
  & \le \inf_{\delta \ge 0}
  \left\{\delta + \frac{1}{\sqrt{n}}
  \int_\delta^\zbound \sqrt{\log N_\infty(\mc{F}, \epsilon, n)} d\epsilon
  \right\}.
  \label{eqn:entropy-integral}
\end{align}

Our main (essentially standard~\cite{VanDerVaartWe96}) motivating example is
that of Lipschitz loss functions for a parametric set $\Theta$, as follows.

\begin{example}
  \label{example:lipschitz-covering}
  Let $\Theta \subset \R^d$ and assume that $\loss : \Theta \times
  \statdomain \to [0, \zbound]$ is $L$-Lipschitz in $\theta$ with respect to
  the $\ell_2$-norm for all $\statval \in \statdomain$, meaning that
  $|\loss(\theta, \statval) - \loss(\theta', \statval)| \le L \ltwo{\theta -
    \theta'}$.  Then taking $\fclass = \{\loss(\theta, \cdot) : \theta \in
  \Theta\}$, any $\epsilon$-covering $\{\theta_1, \ldots, \theta_\covnum\}$
  of $\Theta$ in $\ell_2$-norm guarantees that $\min_i |\loss(\theta,
  \statval) - \loss(\theta_i, \statval)| \le L \epsilon$ for all $\theta,
  \statval$. That is,
  \begin{equation*}
    \covnum(\fclass, \epsilon, \linfstatnorm{\cdot})
    \le \covnum(\Theta, \epsilon / L, \ltwo{\cdot})
    \le \left(1 + \frac{\diam(\Theta) L}{\epsilon}\right)^d,
  \end{equation*}
  where $\diam(\Theta) = \sup_{\theta, \theta' \in \Theta} \ltwo{\theta -
    \theta'}$. Thus $\ell_2$-covering numbers of $\Theta$
  control $L^\infty$-covering numbers of the family $\mc{F}$,
  and
  we have by the entropy integral~\eqref{eqn:entropy-integral} that
  \begin{equation*}
    \worstrad(\mc{F})
    \lesssim
    \sqrt{\frac{d}{n}}
    \int_0^{\diam(\Theta) L} \sqrt{\log \frac{\diam(\Theta) L}{\epsilon}}
    d\epsilon
    \lesssim \diam(\Theta) L \sqrt{\frac{d}{n}}.
  \end{equation*}
  That is, with high probability, for all $\theta$ such that
  $\var(\loss(\theta, \statrv)) \ge \frac{4\zbound^2 \tol}{n} + \frac{Cd
    \diam(\Theta)^2 L^2 \log^3 n}{n}$, we have the exact variance
  expansion~\eqref{eqn:uniform-exact-variance-expansion}.
\end{example}

\section{Optimization by Minimizing the Robust Loss}
\label{sec:optimal-selection}

Based on the precise variance expansions in the preceding section, it is
natural to expect that the robust solution~\eqref{eqn:solve-robust}
automatically trades between approximation and estimation error.  This
intuition is accurate, and we show that the robustly regularized objective
$\risk_n(\theta; \mc{P}_n)$ overestimates the population risk
$\risk(\theta)$ by at most $O(1/n)$. By virtue of optimizing this tighter
approximation---as opposed to the usual $O(1/\sqrt{n})$-approximation given
by the empirical risk $\E_{\emp}[\loss(\theta; X)]$---the robustly
regularized solution~\eqref{eqn:solve-robust} enjoys a number of favorable
finite-sample properties, which are not always comparable to those for
empirical risk minimization (ERM). 

In Section~\ref{section:covering}, we present two versions of our main
result that depend on covering numbers and discuss their consequences,
and we provide an example where the robustly regularized solution $\robsol$
achieves a tighter excess risk bound compared to those that a
straightforward application of localized Rademacher
complexities~\cite{BartlettBoMe05} show that the ERM solution $\ermsol$
achieves. As evidenced by the substantial work on Rademacher- and
Gaussian-complexity and symmetrization, in some instances
covering-number-based arguments do not provide the sharpest
scaling~\cite{BartlettMe02,BartlettBoMe05,SrebroSrTe10}; thus, in
Section~\ref{section:local-rademacher} we present a version of our main
result that depends on localized Rademacher complexities, which can allow
more refined uniform concentration bounds than covering numbers.  We also
provide a concrete (but admittedly somewhat contrived) example where our
robustly regularized procedure~\eqref{eqn:solve-robust} achieves
$\risk(\robsol) - \inf_{\theta\in \Theta}\risk(\theta) \lesssim \frac{\log
  n}{n}$, while empirical risk minimization suffers $ \risk(\ermsol) -
\inf_{\theta\in \Theta} \risk(\theta) \gtrsim \frac{1}{\sqrt{n}}$, in
Section~\ref{section:beat-erm}. The robust ``regularizer'' has invariance
properties other regularization procedures do not, and we mention these
briefly in Section~\ref{section:invariance}.

\subsection{Covering arguments}
\label{section:covering}

Our first guarantee depends on the covering numbers of the function class
$\fclass$ as we describe in Section~\ref{sec:covering-numbers-from-rad}.
While we state our results abstractly, in the loss minimization setting we
typically consider the function class $\fclass \defeq \left\{ \loss(\theta,
\cdot): \theta \in \Theta \right\}$ parameterized by $\theta$.  We have the
following theorem, where as usual, we let $\fclass$ be a collection of
functions $f : \statdomain \to [\zbound_0, \zbound_1]$ with $\zbound =
\zbound_1 - \zbound_0$.
\begin{theorem}
  \label{theorem:selection-by-robustness}
  Let $n \ge 8M^2/t$, $t \ge \log 12$, $\epsilon > 0$, and
  $\tol \ge 9t$. Then with probability at
  least $1 - 2 (3\covnum_{\infty}\left(\fclass, \epsilon, 2n\right)+1)
  e^{-t}$,
  \begin{equation}
    \label{eqn:coverage-F}
    \E[f(\statrv)] \le 
    \sup_{P : \phidivs{P}{\emp} \le \frac{\tol}{n}}
    \E_P[f(\statrv)]
    + \frac{11}{3} \frac{\zbound \tol}{n}
    + \left(2 + 4\sqrt{\frac{2t}{n}}\right) \epsilon
  \end{equation}
  for all $f \in \mc{F}$.
  Defining the empirical minimizer
  \begin{equation*}
    \what{f} \in \argmin_{f \in \fclass}
    \left\{\sup_P \left\{ \E_P[f(\statrv)]
    : \phidivs{P}{\emp} \le \frac{\tol}{n} \right\} \right\}
  \end{equation*}
  we have with the same probability that
  \begin{equation}
    \E[\what{f}(\statrv)]
    \le \inf_{f \in \fclass}
    \left\{\E[f]
    + 2 \sqrt{\frac{2 \tol}{n} \var(f)}\right\}
    + \frac{19 \zbound \tol}{3 n}
    + \left(2 + 4\sqrt{\frac{2t}{n}}\right) \epsilon.
    \label{eqn:fast-convergence-F}
  \end{equation}
\end{theorem}
\noindent
See Section~\ref{sec:proof-selection-by-robustness} for a proof
of the theorem. Because uniform $L^\infty$-covering numbers upper bound
empirical $L^\infty$-covering numbers, it is immediate that
covering $\mc{F}$ in $\linfstatnorm{\cdot}$ provides an identical result.


\subsubsection{Covering bounds: corollaries}

We turn to a number of corollaries that expand on
Theorem~\ref{theorem:selection-by-robustness} to investigate its
consequences. Our first corollary shows that
Theorem~\ref{theorem:selection-by-robustness} applies to standard
Vapnik-Chervonenkis (VC) classes. As VC dimension is preserved through
composition, this result also extends to the
procedure~\eqref{eqn:solve-robust} in typical empirical risk minimization
scenarios. 
\begin{corollary}
  \label{corollary:vc-selection}
  In addition to the conditions of
  Theorem~\ref{theorem:selection-by-robustness}, let $\fclass$ have finite
  VC-dimension $\vcdim(\fclass)$. Then for a numerical constant $c <
  \infty$, the bounds~\eqref{eqn:coverage-F}
  and~\eqref{eqn:fast-convergence-F}
  hold with probability at least
  \begin{equation*}
    1 - \left(c \, \vcdim(\fclass) \left( \frac{16 \zbound ne}{\epsilon}
      \right)^{\vcdim(\fclass)-1} + 2\right) e^{-t}.
  \end{equation*}
\end{corollary}
\begin{proof}
  Let $\lonestatnorm{f} \defeq \int |f(\statval)| dQ(\statval)$ denote the
  $L^1$-norm on $\fclass$ for the probability distribution $Q$.
  Then by Theorem 2.6.7 of
  \citet{VanDerVaartWe96}, we have
  \begin{equation*}
    \sup_Q \covnum(\fclass, \epsilon, \lonestatnorm{\cdot})
    \le c \vcdim(\fclass) \left(\frac{8 \zbound e}{\epsilon}\right)^{\vcdim(\fclass) - 1}
  \end{equation*}
  for a numerical constant $c$.
  Because $\linf{x} \le \lone{x}$, taking $Q$ to be uniform on
  $\statval \in \statdomain^{2n}$ yields
  $\covnum(\fclass(x), \epsilon, \linf{\cdot}) \le \covnum(\fclass,
  \frac{\epsilon}{2n}, \lonestatnorm{\cdot})$.
  The result is immediate.
\end{proof}

Next, we focus more explicitly on the estimator $\robsol$ defined by
minimizing the robust regularized risk~\eqref{eqn:solve-robust}.  Let us
assume that $\Theta \subset \R^d$, and that we have a typical linear modeling
situation, where a loss $h$ is applied to an inner product, that is,
$\loss(\theta, \statval) = h(\theta^\top \statval)$. In this case, by making
the substitution that the class
$\fclass = \{\loss(\theta, \cdot) : \theta \in \Theta\}$ in
Corollary~\ref{corollary:vc-selection}, we have $\vcdim(\fclass) \le d$, and
we obtain the following corollary. In the corollary, recall the
definition~\eqref{eqn:risk} of the population risk
$\risk(\theta) = \E[\loss(\theta, \statrv)]$, and the uncertainty set
$\mc{P}_n = \{P : \phidivs{P}{\emp} \le \frac{\tol}{n}\}$, and that
$\risk_n(\theta, \mc{P}_n) = \sup_{P \in \mc{P}_n} \E_P[\loss(\theta,
\statrv)]$.  By setting $\epsilon = \zbound / n$ in
Corollary~\ref{corollary:vc-selection}, we obtain the following result.
\begin{corollary}
  \label{corollary:theta-minimization}
  Let the conditions of the previous paragraph hold and let
  $\robsol \in \argmin_{\theta \in \Theta} \risk_n(\theta,
  \mc{P}_n)$.
  Assume also that $\loss(\theta, \statval) \in [0, \zbound]$ for all
  $\theta \in \Theta, \statval \in \statdomain$. Then if $n \ge \tol \ge 9 \log 12$,
  \begin{align*}
    \risk(\robsol)
    \le \risk_n(\robsol, \mc{P}_n)
    +  \frac{11 \zbound \tol}{3 n} 
    + \frac{2 \zbound}{n} \left(1 + \sqrt{\frac{\tol}{n}}\right) 
    & \le \inf_{\theta \in \Theta}
    \left\{\risk(\theta)
    + 2 \sqrt{\frac{2 \tol}{n} \var(\loss(\theta; \statrv))}\right\}
    + \frac{11 \zbound \tol}{n}
  \end{align*}
  with probability at least
  $1 - 2 \exp(c_1 d \log n - c_2 \tol)$, where $c_i$ are universal constants
  with $c_2 \ge 1/9$.
\end{corollary}

To give an alternate concrete variant of
Corollary~\ref{corollary:theta-minimization} and
Theorem~\ref{theorem:selection-by-robustness}, let $\Theta \subset \R^d$ and
recall Example~\ref{example:lipschitz-covering}. We assume that for each
$\statval \in \statdomain$, $\inf_{\theta \in \Theta} \loss(\theta,
\statval) = 0$ and that $\loss$ is $L$-Lipschitz in $\theta$.  If $D \defeq
\diam(\Theta) = \sup_{\theta, \theta' \in \Theta} \ltwo{\theta - \theta'} <
\infty$, then $\loss(\theta, \statval) \le L \diam(\Theta)$, and for $\delta
> 0$, we define
\begin{equation}
  \label{eqn:selecting-tol-for-coverage}
  \tol = \log \frac{2}{\delta}
  + d \log(2 n D L).
\end{equation}
Setting $t = \tol$ and $\epsilon = \frac{1}{n}$ in
Theorem~\ref{theorem:selection-by-robustness}
and assuming that $\delta \lesssim 1 / n$, $D \lesssim n^k$ and
$L \lesssim n^k$ for a numerical constant $k$, choosing $\delta = \frac{1}{n}$
we obtain that with probability at least $1 - \delta = 1 - 1/n$,
\begin{align}
  \label{eqn:fast-convergence-lipschitz}
  \E[\loss(\robsol; \statrv)]
  = \risk(\robsol)
  & \le \inf_{\theta \in \Theta}
  \left\{\risk(\theta) + C \sqrt{\frac{d \, \var(\loss(\theta, \statrv))}{n}
    \log n}
  \right\}
  + C \frac{d LD \log n}{n}
\end{align}
where $C$ is a numerical constant.

\subsubsection{Examples and heuristic discussion}
\label{section:examples}

Unpacking Theorem~\ref{theorem:selection-by-robustness}, the first
result~\eqref{eqn:coverage-F} (and its in
Corollary~\ref{corollary:theta-minimization})
provides a high-probability guarantee that the true expectation
$\E[\what{f}]$ cannot be more than $O(1/n)$ worse than its
robustly-regularized empirical counterpart. The second
result~\eqref{eqn:fast-convergence-F} (and
inequality~\eqref{eqn:fast-convergence-lipschitz}) guarantees convergence of
the empirical minimizer to a parameter with risk at most $O(\log n/n)$
larger than the best possible variance-corrected risk.

To illustriate how variance regularization can yield tighter guarantees than
empirical risk minimization by optimizing a $O(1/n)$ upper bound on the
risk, we now compare the second bound~\eqref{eqn:fast-convergence-F}
with an analogous result for
empirical risk minimization (ERM). We first give a heuristic version, making
it more precise in a coming example. For the ERM solution $\ermsol \in
\argmin_{\theta \in \Theta} \E_{\emp}[\loss(\theta; X)]$, one common
assumption is an upper bound of the variance by the risk; for example, when
the losses take values in $[0, \zbound]$, one has $\var(\loss(\theta,
\statrv)) \le \zbound \risk(\theta)$. In such cases, there is typically some
complexity measure $\mathfrak{Comp}_n$ associated with the class of
functions being learned, and it is possible to achieve bounds of the form
\begin{equation}
  \label{eqn:erm-bound}
  \risk(\ermsol) \le \risk(\theta\opt) +
  C \sqrt{\frac{ \mathfrak{Comp}_n \zbound
      \risk(\theta\opt)}{n}} + C \frac{\mathfrak{Comp}_n \zbound }{n}
\end{equation}
where $\theta\opt \in \argmin_{\theta \in \Theta} \risk(\theta)$, a type of
result common for bounded nonnegative
losses~\cite{BoucheronBoLu05,VapnikCh71,Vapnik98}. For example, for classes
of functions of VC-dimension $d$, we typically have $\mathfrak{Comp}_n
\lesssim d \log \frac{n}{d}$. In this caricature, when
$\var(\loss(\theta\opt, \statrv)) \ll \zbound \risk(\theta\opt)$ and $\tol
\gtrsim \mathfrak{Comp}_n$, the optimality
guarantee~\eqref{eqn:fast-convergence-F} for variance regularization can be
tighter than its ERM counterpart~\eqref{eqn:erm-bound}. This bound is
certainly not always sharp, but yields minimax optimal rates in some cases.

\begin{example}[Well-specified least-absolute-deviation regression]
  \label{example:l1-regression}
  We consider the least-absolute-deviation (LAD) regression problem,
  comparing the rates of convergence that localized Rademacher complexities
  guarantee against those that the robust program provides. Let $Z = (X, Y)
  \in \R^d \times \R$, where $X \in \{x \in \R^d \mid \ltwo{x} \le L\}$, and
  let $D \defeq \diam (\Theta)$ be the $\ell_2$-diameter of $\Theta$.
  The LAD loss is
  \begin{equation*}
    \loss(\theta; (x, y)) \defeq |y - \<\theta, x\>|.
  \end{equation*}
  For some $\theta\opt \in \Theta$, assume that
  \begin{equation*}
    Y = \<\theta\opt, X\> + \epsilon
  \end{equation*}
  where $\epsilon \in [-B, B]$ is independent of $X$. We have the global
  bound
  \begin{equation*}
    \loss(\theta; (X, Y)) \le DL + B \eqdef \zbound.
  \end{equation*}
  Suppose for simplicity that $\epsilon$ is uniform on $[-B, B]$; then
  $\theta\opt = \argmin_{\theta \in \Theta} \risk(\theta)$ and
  $\risk(\theta\opt) = \E[\loss(\theta\opt; Z)] = \half B$. In this case,
  \begin{align*}
    \var\left(\loss(\theta\opt; Z) \right)
    = \frac{B^2}{12}
    \le
    \half (DL + B) B =
    M \E[\loss(\theta\opt; Z)]
    = M \risk(\theta\opt).
  \end{align*}

  Using that the loss is $1$-Lipschitz, the $L^\infty$ covering numbers for
  the set of functions $\fclass \defeq \{f_\theta(x, y)= |\<\theta, x\> - y|
  \mid \theta \in \Theta\}$ satisfy $\log \covnum(\fclass, \epsilon,
  \linfstatnorm{\cdot}) \lesssim d \log\frac{DL}{\epsilon}$, and so applying
  the bound~\eqref{eqn:fast-convergence-lipschitz} for the robustly
  regularized solution~$\robsol$ with $\epsilon = DL / n$, we obtain
  \begin{equation*}
    \risk(\robsol)
    \le \risk(\theta\opt)
    + C \sqrt{\frac{d \log n}{n} B^2}
    + C \frac{d (LD + B)\log n}{n}
  \end{equation*}
  with probability at least $1-1/n$. On the other hand, even an
  ``optimistic'' (but naive) ERM bound, achieved by taking
  $\mathfrak{Comp}_n \lesssim 1$ in the bound~\eqref{eqn:erm-bound}, yields
  \begin{equation*}
    \risk(\ermsol)
    \le \risk(\theta\opt)
    + C \sqrt{\frac{\log n}{n} (BDL + B^2)}
    + C \frac{(LD + B)\log n}{n}
  \end{equation*}
  with probability at least $1-1/n$. We see that leading term for the
  robustly regularized solution $\robsol$ only depends on the noise-level
  $B^2$ while the corresponding term for the ERM solution $\ermsol$ depends
  on global information like the size of the parameter space $D$, and a
  uniform bound over covariates $L$. For typical VC and other
  $d$-dimensional classes, the bound $\mathfrak{Comp}_n$ scales linearly in
  $d$ (cf.~\cite[Corollary 3.7]{BartlettBoMe05}, in which case the
  bound~\eqref{eqn:erm-bound} scales as $\risk(\theta\opt) + C \sqrt{d (BDL
    + B^2) \log n / n} + O(\log n / n)$, which is worse.
\end{example}

\begin{example}[A hard median estimation problem]
  To give a bit more insight into the behavior of the robust estimator,
  consider the simple 1-dimensional median problem, where
  $\loss(\theta; x) = |\theta - x|$, and assume that
  $x \in \{-B, B\}$ with $\P(X = B) = \frac{1 + \delta}{2}$ for
  some $\delta > 0$, so that
  $\theta\opt = \argmin \risk(\theta) = B$
  and $\risk(\theta\opt) = (1 - \delta) B$.
  In this case, taking $\theta_0 = 0$ yields
  $\var(\loss(\theta; X)) = 0$ and $\risk(\theta_0) - \risk(\theta\opt)
  = \delta B$.
  For $\delta$ small (on the order of $1 / \sqrt{n}$), with constant
  probability the empirical risk minimizer is
  $\ermsol = -B$, yielding risk $\risk(\ermsol) - \risk(\theta\opt)
  = 2 \delta B$. On the other hand, with high
  probability $\robsol \ge 0$ (because $\var(\loss(\theta_0; X)) = 0$
  as $\loss(0; X) \equiv B$), and so
  $\risk(\robsol) - \risk(\theta\opt) \le \delta B$.
  This gap is of course small, but it shows that the robust solution
  is more conservative: it chooses $\robsol$ so that large
  losses (of scale $2 B$) are less frequent.
\end{example}

When the population problem is ``easy'', it is often possible to achieve
faster rates of convergence than the usual $O\left(1/\sqrt{n}\right)$
rate. The simplest scenario where this occurs is if the problem is realizable
$\risk(\theta\opt) = 0$, in which case $\ermsol$ has excess risk of the order
$O(\log n/n)$; see the bound~\eqref{eqn:erm-bound}.  The robustly regularized
solution $\robsol$ enjoys the same faster rates of convergence under the more
general condition that $\var(\loss(\theta\opt; X))$ is small. As a concrete
instance of this, let $\loss(\theta; X) \in [0, \zbound]$ and assume that
$\loss(\theta; X)$ satisfies the conditions of the first part of
Example~\ref{example:lipschitz-covering}, and let the problem be realizable
$\risk(\theta\opt) = 0$. Since
$\var(\loss(\theta; X)) \le \zbound \risk(\theta)$, we have from the
bounds~\eqref{eqn:fast-convergence-lipschitz} and~\eqref{eqn:erm-bound} that
\begin{equation*}
  \risk(\ermsol) \le \frac{Cd DL \log n}{n}
  ~~\mbox{and}~~
  \risk(\robsol) \le \frac{Cd DL \log n}{n}.
\end{equation*}
For example, $\var(\loss(\theta; X)) = 0$ allows for the existence of some
$\theta_0 \in \Theta$ such that $\loss(\theta_0; X) < \loss(\theta\opt; X)$
with positive probability.

\subsection{Localized Rademacher Complexity}
\label{section:local-rademacher}

A somewhat more sophisticated approach to concentration inequalities and
generalization bounds is based on localization ideas, motivated by the fact
that near the optimum of an empirical risk, the complexity of the function
class may be smaller than over the entire (global)
class~\cite{VanDerVaartWe96, BartlettBoMe05}. With this in mind, we
now present a refined version of
Theorem~\ref{theorem:selection-by-robustness} that depends on localized
Rademacher averages.

The starting point for this approach is a notion of
localized Rademacher complexity (we give a slightly less general notion than
\citet{BartlettBoMe05}, as it is sufficient for our derivations). For a
function class $\mc{F}$ of functions $f : \mc{X} \to \R$, the localized
Rademacher complexity at level $r$ is
\begin{equation*}
  \E\left[\radcomp_n\left(\left\{c f \mid f \in \fclass,
    c \in [0, 1], \E[c^2 f^2 \le r] \right\}\right)\right].
\end{equation*}
In addition, we require a few analytic notions, beginning with
\emph{sub-root} functions, where we recall~\cite{BartlettBoMe05} that a
function $\psi : \R_+ \to \R_+$ is \emph{sub-root} if it is nonnegative,
nondecreasing, and $r \mapsto \psi(r) / \sqrt{r}$ is nonincreasing for all
$r > 0$.  Any non-constant sub-root function $\psi$ is continuous and has a
unique positive fixed point $r\opt = \psi(r\opt)$, where $r \ge \psi(r)$ for
all $r \ge r\opt$. Lastly, we consider upper bounds $\psi_n : \R_+ \to \R_+$
on the localized Rademacher complexity satisfying
\begin{equation}
  \label{eqn:sub-root}
  \psi_n(r) \ge \E[ \radcomp_n(\{ cf: f \in \fclass, c \in [0, 1], \E[c^2f^2]
  \le r\})],
\end{equation}
where $\psi_n$ is sub-root. (The localized Rademacher complexity itself is
sub-root.)  Roots of $\psi_n$ play a fundamental role in providing uniform
convergence guarantees, and \citet{BartlettBoMe05} and
\citet{Koltchinskii06a} provide careful analyses of localized Rademacher
complexities, with typical results as follows.  For a class of functions $f$
with range bounded by $1$, for any root $r_n\opt$ of $\psi_n$, with
probability at least $1 - e^{-t}$ we have
\begin{equation*}
  \E[f]
  \le \E_\emp[f] + \frac{1}{\eta} \E_\emp[f] + C (1 + \eta)
  \left(r_n\opt + \frac{1}{n}\right) + \frac{t}{n}
  ~~ \mbox{for~all~} f \in \mc{F} ~ \mbox{and}~ \eta \ge 0.
\end{equation*}
As an example, when $\fclass$ is a bounded VC-class,
we have $r_n \opt \asymp \frac{\vcdim(\mc{F}) \log
  (n / \vcdim(\mc{F}))}{n}$~\cite[Corollary 3.7]{BartlettBoMe05}.

With this motivation, we have the following theorem.
\begin{theorem}
  \label{theorem:selection-by-robustness-localized}
  For $\zbound \ge 1$, let $\fclass$ be a collection of functions $f:
  \statdomain \to [0, \zbound]$, let $\psi_n$ be a sub-root function
  bounding the localized complexity~\eqref{eqn:sub-root}, and let $r_n\opt
  \ge \psi_n(r_n\opt)$. Let $t > 0$ be arbitrary and assume that $\tol$
  satisfies
  \begin{equation}
    \label{eqn:rho-bound}
    \frac{\tol}{n} \ge
    8 \left(\frac{45 \zbound}{n}
    \xprime + 18 r_n\opt \right).
  \end{equation}
  Then with probability at least $1 - e^{-t}$,
  \begin{equation}
    \label{eqn:coverage-F-local}
    \E[f]
    \le
    \left(1 + 2 \sqrt{\frac{2\tol}{n}} \right)
    \sup_{P : \phidivs{P}{\emp} \le \frac{\tol}{n}} \E_P[f]
    + \left(13 + 4 \sqrt{\frac{2\tol}{n}} \right) \frac{\zbound \tol}{n}
    ~~ \mbox{for~all~} f \in \mc{F}.
  \end{equation}
  Additionally, if $\what{f}$ minimizes
  $\sup_{P : \phidivs{P}{\emp} \le \tol / n} \E_P[f]$, then
  with probability at least $1 - 3 e^{-t}$,
  \begin{equation}
    \label{eqn:second-fast-convergence-F-local}
    \E[\what{f}]
    \le 
    \left(1 + 2 \sqrt{\frac{2 \tol}{n}} \right)
    \inf_{f \in \fclass} \left(\E[f] + \sqrt{\frac{91 \tol}{ 45n} \var(f)}
    \right)
    + \left(14 + 6 \sqrt{\frac{2\tol}{n}} \right)
    \frac{\zbound (3\tol + t)}{n}.
  \end{equation}
\end{theorem}
\noindent
We provide the proof of
Theorem~\ref{theorem:selection-by-robustness-localized} in
Appendix~\ref{section:proof-of-selection-by-robustness-localized}.  It
builds off of and parallels many of the techniques developed by
\citet*{BartlettBoMe05}, but we require a bit of care to develop the precise
variance bounds we provide.

Let us consider the additional $\sqrt{\frac{\tol}{n}}$ factors in
Theorem~\ref{theorem:selection-by-robustness-localized} (as compared to
Theorem~\ref{theorem:selection-by-robustness}).  In general, these terms are
negligible to the extent that the variance of $f$ dominates the first moment
of the function $f$---heuristically, in situations in which we expect
penalizing the variance to improve performance.
Let us make this more precise in a regime where $n$ is large.
Letting $f \in \fclass$, we see
that we have the inequality
\begin{equation*}
  (1 + \sqrt{\tol / n})\left(\E[f] + \sqrt{\frac{\tol}{n}} \var(f)\right)
  \le \E[f] + C \sqrt{\frac{\tol}{n} \var(f)}
\end{equation*}
(for a constant $C > 1 + \sqrt{\tol / n}$) if and only if
$(C - 1 - \sqrt{\tol / n})^2
\var(f) \ge \E[f]^2$.
Equivalently, as $n$ gets large, this occurs
roughly when $\E[f^2] \ge \frac{C^2 - 2C + 2}{C^2 - 2C + 1} \E[f]^2$,
which holds for large enough $C$ whenever $\var(f) > 0$.

In some scenarios, we can obtain substantially tighter bounds by using
localized Rademacher averages instead of the covering number arguments
considered in Section~\ref{section:covering}. (Recall also the discussion
following Theorem~\ref{theorem:uniform-variance-expansion}.)  To illustriate
this point, we consider the case where $\mc{F}$ is a bounded subset of a
reproducing kernel Hilbert space generated by some sufficiently nice kernel
$K$; even for the Gaussian kernel $K(x, z) = \exp(-\half \norm{x - z}^2)$,
log covering numbers for such function spaces grow at least exponentially in
the dimension~\cite{Zhou03, Kuhn11}.

\begin{example}[Reproducing kernels and least-absolute-deviation regression]
  We now give an example using a non-parametric class of functionals in
  which covering number arguments do not apply, as the covering numbers of
  the associated classes are too large.
  Let $\mc{H}$ be a
  reproducing kernel Hilbert space (RKHS) with norm $\hnorm{\cdot}$ and
  associated kernel (representer of evaluation) $K : \mc{X} \times \mc{X}
  \to \R$.  Letting $P$ be a distribution on $\mc{X}$, Mercer's
  theorem~\cite[e.g.][]{CristianiniSh04} implies that the integral operator
  $T_K : L^2(\mc{X}, P) \to L^2(\mc{X}, P)$ defined by $T_K(f)(x) = \int
  K(x, z) dP(z)$ is compact, and $K(x, x') = \sum_{j = 1}^\infty \lambda_j
  \phi_j(x) \phi_j(z)$ where $\lambda_j$ are the eigenvalues of $T$ in
  decreasing order and $\phi_j$ form an orthonormal decomposition of
  $L^2(\mc{X}, P)$.

  Consider now the least absolute deviation (LAD) loss function $\loss(h; x,
  y) = |h(x) - y|$, defined for $h \in \mc{H}$, and let $\ball_{\mc{H}}$ be
  the unit $\hnorm{\cdot}$-ball of $\mc{H}$. Assume additionally that
  the model is well-specified, and that
  $y = h\opt(x) + \xi$ for some random variable $\xi$ with
  $\E[\xi \mid X] = 0$, $\E[\xi^2] \le \sigma^2$, and $h\opt \in \ball_{\mc{H}}$.
  Let the function class
  \begin{equation*}
    \{\loss \circ \mc{H}\}_{\le r}
    \defeq
    \left\{(x, y) \mapsto c \loss(h(x), y)
    \mid c \in [0, 1], c^2 \E[\loss(h(X), Y)^2] \le r
    \right\}.
  \end{equation*}
  Based on inequality~\eqref{eqn:sub-root}, we consider
  the localized complexity
  \begin{equation*}
    \radcomp_n(\{\loss \circ \mc{H}\}_{\le r})
    = \E\left[
      \frac{1}{n} \sup_{h \in \ball_{\mc{H}}, c \in [0, 1]} \sum
      \varepsilon_i c \loss(h(x_i), y_i)
      \mid \E[\loss(h(X), Y)^2] \le r / c^2 \right].
  \end{equation*}
  We claim that
  \begin{equation}
    \label{eqn:rademacher-eigenvalues}
    \radcomp_n(\{\loss \circ \mc{H}\}_{\le r})
    \lesssim
    \sqrt{r / n}
    + \left(\frac{1}{n} \sum_{j = 1}^\infty
    \min\{\lambda_j, r\}\right)^\half.
  \end{equation}
  As this claim is not central to our development---but does show
  a slightly different localization result based on Gaussian
  comparison inequalities than
  available, for example, in \citet{Mendelson03}---we
  provide its proof in Appendix~\ref{sec:proof-rademacher-eigenvalues}.

  Let us use inequality~\eqref{eqn:rademacher-eigenvalues}. To apply
  Theorem~\ref{theorem:selection-by-robustness}, we must find a bound on the
  fixed point of the localized complexity.  To give this bound, we require
  some knowledge on the eigenvalues $\lambda_j$, for which there exists a
  body of work. For example~\cite{Mendelson03}, the Gaussian kernel $K(x,
  x') = \exp(-\half \ltwo{x - x'}^2)$ generates a class of smooth functions
  for which the eigenvalues $\lambda_j$ decay exponentially, as $\lambda_j
  \lesssim e^{-j^2}$. Kernel operators underlying Sobolev spaces with
  different smoothness orders~\cite{BirmanSo67,Gu02} typically have
  eigenvalues scaling as $\lambda_j \lesssim j^{-2 \alpha}$ for some $\alpha
  > \half$. As a concrete example, the first-order Sobolev (min) kernel
  $K(x, x') = 1 + \min\{x, x'\}$ generates an RKHS of Lipschitz functions
  with $\alpha = 1$.  In the former case of $\lambda_j \lesssim e^{-j^2}$,
  $r_n\opt = \frac{\sqrt{\log n}}{n}$
  \begin{equation*}
    \left(\frac{1}{n} \sum_{j = 1}^\infty \min\left\{e^{-j^2},
    \frac{\log n}{n} \right\}\right)^\half
    \approx \left(\frac{1}{n} \sum_{j = 1}^{\sqrt{\log n}}
    \frac{\sqrt{\log n}}{n}
    + \frac{1}{n} \int_{\sqrt{\log n}}^\infty e^{-t^2} dt \right)^\half
    \lesssim \frac{\sqrt{\log n}}{n} = r_n\opt.
  \end{equation*}
  In the latter case of polynomially decaying eigenvalues
  $\lambda_j \lesssim j^{-2 \alpha}$, we have
  $j^{-2\alpha} = r$ when $r^{-\frac{1}{2\alpha}} = j$, so
  \begin{equation*}
    \sum_{j = 1}^\infty \min\{j^{-2\alpha}, r\}
    \approx
    r^{\frac{2 \alpha - 1}{2 \alpha}}
    + \int_{r^{-1 / 2\alpha}}^\infty t^{-2\alpha} dt
    \asymp
    r^{\frac{2 \alpha - 1}{2\alpha}}.
  \end{equation*}
  Solving for $n r = r^\frac{2 \alpha - 1}{2\alpha}$, we find the
  fixed point $(r_n\opt)^\frac{2 \alpha - 1}{4 \alpha} = r_n\opt \sqrt{n}$
  yields $r_n\opt = n^{-\frac{2 \alpha}{2 \alpha + 1}}$.

  Ignoring constants, the above analysis shows that in the case that
  the kernel eigenvalues scale as $\lambda_j \lesssim e^{-j^2}$,
  as soon as $\tol \gtrsim \sqrt{\log n}$ we have
  \begin{equation*}
    \E[\loss(h(X), Y)]
    \le (1 + 2\sqrt{2 \tol / n})
    \left(\E_\emp[\loss(h(X), Y)] + \sqrt{\frac{2 \tol}{n}
      \var_\emp(\loss(h(X), Y))} \right)
    + \frac{C \tol}{n}
    ~~ \mbox{for~all~} h \in \ball_{\mc{H}}
  \end{equation*}
  with high probability. In the case of polynomial eigenvalues,
  if $\what{h}$ minimizes the robust empirical loss
  $\sup_{P : \phidivs{P}{\emp} \le \tol/n} \E_P[\loss(h(X), Y)]$
  and $\tol \asymp n^{1 - \frac{2 \alpha}{2 \alpha + 1}}$, then
  \begin{equation*}
    \E\left[\loss(\what{h}(X), Y)\right]
    \le
    \left(1 + C n^{-\frac{\alpha}{2 \alpha + 1}}\right)
    \inf_{h \in \ball_{\mc{H}}}
    \left(\E[\loss(h(X), Y)]
    + C n^{-\frac{\alpha}{2 \alpha + 1}}
    \sqrt{\var(\loss(h(X), Y))}\right)
    + C n^{-\frac{2 \alpha}{2 \alpha + 1}}.
  \end{equation*}
  This rate of convergence holds without any assumptions on the smoothness
  of the distribution of the noise $\xi$.
\end{example}

\subsection{Beating empirical risk minimization}
\label{section:beat-erm}

We now provide a concrete example where the robustly regularized
estimator~$\robsol$ exhibits a substantial performance gap over
empirical risk minimization. In the sequel, we bound the performance
degradation to show that the formulation~\eqref{eqn:solve-robust} in general
loses little over empirical risk minimization.  For intuition in this section,
consider the (admittedly contrived) setting in which we replace the loss
$\loss(\theta, \statrv)$ with
$\loss(\theta, \statrv) - \loss(\theta\opt, \statrv)$, where
$\theta\opt \in \argmin_{\theta \in \Theta} \risk(\theta)$. Then in this case,
by taking $\theta = \theta\opt$ in
Corollary~\ref{corollary:theta-minimization}, we have
$\risk(\robsol) \le \risk(\theta\opt) + O(1/n)$ with high probability.
More broadly, we expect the robustly regularized approach to offer performance
benefits in situations in which the empirical risk minimizer is highly
sensitive to noise, say, because the losses are piecewise linear, and slight
under- or over-estimates of slope may significantly degrade solution quality.

With this in mind, we construct a concrete $1$-dimensional
example---estimating the median of a discrete distribution supported on
$\statdomain = \{-1, 0, 1\}$---in which the robustly regularized estimator has
convergence rate $\log n / n$, while empirical risk minimization is at best
$1 / \sqrt{n}$.  Define the loss
$\loss(\theta; \statval) = |\theta - \statval| - |\statval|$, and for
$\delta \in (0, 1)$ let the distribution $P$ be defined by
\begin{equation}
  \label{eqn:hard-erm-distribution}
  P(\statrv = 1) = \frac{1 - \delta}{2},
  ~~ P(\statrv = -1) = \frac{1 - \delta}{2},
  ~~ P(\statrv = 0) = \delta.
\end{equation}
Then for $\theta \in \R$, the risk of the loss is
\begin{equation*}
  \risk(\theta) = \delta |\theta| + \frac{1 - \delta}{2} |\theta - 1|
  + \frac{1 - \delta}{2} |\theta + 1| - (1 - \delta).
\end{equation*}
By symmetry, it is clear that $\theta\opt \defeq \argmin_\theta \risk(\theta)
= 0$, which satisfies $\risk(\theta\opt) = 0$. (Note also that $\loss(\theta,
\statval) = \loss(\theta, \statval) - \loss(\theta\opt, \statval)$.) Without
loss of generality, we assume that $\Theta = [-1, 1]$ in this problem.

Now, consider a sample $\statrv_1, \ldots, \statrv_n$ drawn i.i.d.\ from the
distribution $P$, let $\emp$ denote its empirical distribution, and define
the empirical risk minimizer
\begin{equation*}
  \ermsol \defeq \argmin_{\theta \in \R} 
  \E_\emp[\loss(\theta, \statrv)]
  = \argmin_{\theta \in [-1, 1]} \E_\emp[|\theta - \statrv|].
\end{equation*}
If too many of the
observations satisfy $\statrv_i = 1$ or too many satisfy $\statrv_i = -1$,
then $\ermsol$ will be either $1$ or $-1$; for small $\delta$, such events
become reasonably probable, as the following lemma makes precise.
In the lemma, $\Phi(x) = \frac{1}{\sqrt{2\pi}} \int_{-\infty}^x e^{-\half t^2} dt$
denotes the standard Gaussian CDF.
(See Section~\ref{sec:proof-erm-sucks} for a proof.) 
\begin{lemma}
  \label{lemma:erm-sucks}
  Let the loss $\loss(\theta; \statval) = |\theta - \statval| - |\statval|$,
  $\delta \in [0, 1]$, and $\statrv$ follow the
  distribution~\eqref{eqn:hard-erm-distribution}.  Then
  $\risk(\ermsol) - \risk(\theta\opt) \ge \delta$
  with probability at
  least
  \begin{equation*}
    2 \Phi\left(-\sqrt{\frac{n \delta^2}{1 - \delta^2}}\right)
    - (1 - \delta^2)^\frac{n}{2} \sqrt{\frac{8}{\pi n}}.
  \end{equation*}
\end{lemma}

On the other hand, we certainly have $\loss(\theta\opt; \statval) = 0$ for
all $\statval \in \statdomain$, so that $\var(\loss(\theta\opt; \statrv)) =
0$.  Now, consider the bound in
Theorem~\ref{theorem:selection-by-robustness}.  We see that $\log
\covnum(\{\loss(\theta, \cdot) : \theta \in \Theta\} , \epsilon,
\linfstatnorm{\cdot}) \le 2 \log \frac{1}{\epsilon}$, and taking $\epsilon =
\frac{1}{n}$, we have that if $\robsol \in \argmin_{\theta \in \Theta}
\risk_n(\theta, \mc{P}_n)$, then
\begin{equation*}
  \risk(\robsol)
  \le \risk(\theta\opt)
  + \frac{15 \tol}{n}
  ~~ \mbox{with~probability}~ \ge 
  1 - 4 \exp\left(2 \log n - \tol \right).
\end{equation*}
In particular, taking $\tol = 3 \log n$, we see that
\begin{equation*}
  \risk(\robsol) \le \risk(\theta\opt)
  + \frac{45 \log n}{n}
  ~~ \mbox{with probability at least} ~
  1 - \frac{4}{n}.
\end{equation*}
The risk for the empirical risk minimizer, as Lemma~\ref{lemma:erm-sucks}
shows, may be substantially higher; taking $\delta = 1 / \sqrt{n}$
we see that with probability at least
$2 \Phi(-\sqrt{\frac{n}{n - 1}}) - 2 \sqrt{2} / \sqrt{\pi e n} \ge 2
\Phi(-\sqrt{\frac{n}{n-1}}) - n^{-\half}$,
\begin{equation*}
  \risk(\ermsol) \ge \risk(\theta\opt) + n^{-\half}.
\end{equation*}
(For $n \ge 20$, the probability of this event is $\ge .088$.)
For this (specially constructed) example, there is a gap of nearly
$n^\half$ in order of convergence.

\subsection{Invariance properties}
\label{section:invariance}

The robust regularization~\eqref{eqn:robust-risk} technique enjoys a number of
invariance properties. Standard regularization techniques (such as $\ell_1$-
and $\ell_2$-regularization), which generally regularize a parameter toward a
particular point in the parameter space, do not. While we leave deeper
discussion of these issues to future work, we make two observations, which
apply when $\Theta = \R^d$ is unconstrained. Throughout, we let
$\robsol \in \argmin_\theta \risk_n(\theta, \mc{P}_n)$ denote the
robustly regularized empirical solution.

First, consider a location estimation problem in which we wish to estimate the
minimizer of some the expectation of a loss of the form
$\loss(\theta, \statrv) = h(\theta - \statrv)$, where $h : \R^d \to \R$ is
convex and symmetric about zero. Then the robust solution is by inspection
shift invariant, as $\loss(\theta + c, \statrv + c) = \loss(\theta, \statrv)$
for any vector $c \in \R^d$. Concretely, in the example of the previous
section, $\ell_1$- or $\ell_2$-regularization achieve better convergence
guarantees than ERM does, but if we shift all data
$\statval \mapsto \statval + c$, then non-invariant regularization techniques
lose efficiency (while the robust regularization technique does not).  Second,
we may consider a generalized linear modeling problem, in which data comes in
pairs $(\statval, y) \in \statdomain \times \mc{Y}$ and
$\loss(\theta, (\statval, y)) = h(y, \theta^\top \statval)$ for a function
$h : \mc{Y} \times \R \to \R$ that is convex in its second argument.  Then
$\robsol$ is invariant to invertible linear transformations, in the sense that
for any invertible $A \in \R^{d \times d}$,
\begin{align*}
  \argmin_{\theta} \Big\{\sup_{P : \phidivs{P}{\emp} \le \frac{\tol}{n}}
  \E_P[\loss(\theta, (\statrv, Y))]\Big\}
  & = \argmin_{\theta} \Big\{\sup_{P : \phidivs{P}{\emp} \le \frac{\tol}{n}}
  \E_P[\loss(A^{-1} \theta, (A \statrv, Y))]\Big\}
  = \robsol.
\end{align*}
Our results in this section do not precisely apply as we require unbounded
$\theta$, however, the next section shows that localization approaches
can address this.



\section{Robust regularization cannot be too bad}
\label{sec:nothing-is-bad}

The previous two sections provide guarantees on the performance of the
robust regularized estimator~\eqref{eqn:solve-robust}, it does
not---cannot---dominate classical approaches based on empirical risk
minimization (also known as sample average approximation in the stochastic
optimization literature), though it can improve on them in some cases.  For
example, with a correctly specified linear regression model with gaussian
noise, least-squares---empirical risk minimization with the loss
$\loss(\theta, (x, y)) = \half (\theta^\top x - y)^2$---is essentially
optimal. Our goal in this section is thus to provide more understanding of
potential poor behavior of the procedure~\eqref{eqn:solve-robust} with
respect to ERM, considering two scenarios. The first is in stochastic
(convex) optimization problems, where we investigate the finite-sample
convergence rates of the robust solution to the population optimal risk. We
show that the robust solution $\robsol$ enjoys fast rates of convergence in
cases in which the risk has substantial curvature---precisely as with
empirical risk minimization.  The second is to consider the asymptotics of
the robust solution $\robsol$, where we show that in classical statistical
scenarios the robust solution is nearly efficient, though there is an
asymptotic bias of order $1 / \sqrt{n}$ that scales with the confidence
$\tol$.

\subsection{Fast Rates}
\label{section:fast-rates-main}

In cases in which the risk $\risk$ has curvature, empirical risk minimization
often enjoys faster rates of convergence~\cite{BoucheronBoLu05,
  ShapiroDeRu09}. The robust solution $\robsol$ similarly attains faster rates
of convergence in such cases, even with approximate minimizers of
$\risk_n(\theta, \mc{P}_n)$.  For the risk $\risk$ and $\epsilon \ge 0$, let
\begin{equation*}
  \solset^\epsilon \defeq \left\{\theta \in \Theta : \risk(\theta) \le
    \inf_{\theta\opt \in \Theta} \risk(\theta\opt) + \epsilon\right\}
\end{equation*}
denote the $\epsilon$-sub-optimal (solution) set, and similarly let
\begin{equation*}
  \empsolset^\epsilon \defeq \left\{\theta \in \Theta : \risk_n(\theta,
  \mc{P}_n) \le \inf_{\theta' \in \Theta} \risk_n(\theta', \mc{P}_n) +
  \epsilon \right\}.
\end{equation*}
For a vector $\theta \in \Theta$, let $\pi_{\solset}(\theta) =
\argmin_{\theta\opt \in \solset} \ltwos{\theta\opt - \theta}$ denote the
Euclidean projection of $\theta$ onto the set $\solset$; this projection
operator is very useful for showing faster rates of convergence in
stochastic optimization (see~\citet{ShapiroDeRu09}, whose techniques we
closely follow).  In the statement of the result, for $A \subset \Theta$, we
let $\radcomp_n(A)$ denote the Rademacher complexity of the localized
process $\{x \mapsto \loss(\theta; x) - \loss(\pi_{\solset}(\theta); x) :
\theta \in A\}$.
We then have the following result, whose proof we provide in
Section~\ref{section:proof-of-fast-rates}.

\begin{theorem}
  \label{theorem:fast-rates}
  Let $\Theta$ be convex and let $\loss(\cdot; \statval)$ be convex and
  $L$-Lipshitz in its first argument for all $\statval \in \statdomain$. For
  constants $\lambda > 0$, $\growthpow > 1$, and $\hoodrad > 0$, assume the
  risk $\risk$ satisfies
  \begin{equation}
    \label{eqn:growth}
    \risk(\theta) - \inf_{\theta \in \Theta} \risk(\theta) \ge
    \lambda \dist(\theta, \solset)^\growthpow
    ~~ \mbox{for~all}~
    \theta ~ \mbox{such~that}~
    \dist(\theta, \solset) \le \hoodrad.
  \end{equation}
  Let $t > 0$. If $0 \le \epsilon \le \half \lambda r^\growthpow$
  satisfies
  \begin{equation}
    \label{eqn:fast-rate-rademacher-bound}
    \epsilon \ge
    \left(2 \frac{8^\growthpow L^\growthpow}{\lambda}
    \right)^\frac{1}{\growthpow - 1}
    \left(\frac{\tol}{n}\right)^\frac{\growthpow}{2(\growthpow - 1)}
    ~~ \mbox{and} ~~
    \frac{\epsilon}{2} \ge
     2 \E[\radcomp_n(\solset^{2\epsilon})]
     + L \left(\frac{2\epsilon}{\lambda}\right)^{\frac{1}{\growthpow}}
     \sqrt{\frac{2t}{n}},
   \end{equation}
  then
  $\P( \empsolset^{\epsilon} \subset \solset^{2\epsilon} ) \ge 1-e^{-t}$,
\end{theorem}

We provide a brief discussion of this result as well as a corollary that
gives more explicit rates of convergence. First, we note that (by an
inspection of the proof) the $L$-Lipschitz assumption need only hold in the
neighborhood $\solset^{2\epsilon}$ for the result to hold.
We also have the following
\begin{corollary}
  In addition to the conditions of Theorem~\ref{theorem:fast-rates},
  assume that $\solset = \{\theta\opt\}$ is a single point and
  $\Theta \subset \R^d$. Then
  for any $\epsilon \le \half \lambda r^\growthpow$,
  we have $\P(\empsolset^\epsilon \subset \solset^{2\epsilon})
  \ge 1 - e^{-t}$ for
  \begin{equation*}
    \epsilon
    \gtrsim
    \left(\frac{L^\growthpow}{\lambda}\right)^\frac{1}{\growthpow - 1}
    \left(\frac{d}{n} \log \frac{n}{d}
    + \frac{t}{n} + \frac{\tol}{n} \right)^\frac{\growthpow}{2(\growthpow - 1)}.
  \end{equation*}
\end{corollary}

So long as $\tol \lesssim d \log \frac{n}{d}$, this rate of convergence is
as good as that enjoyed by standard empirical risk minimization
approaches~\cite[Ch.~5]{ShapiroDeRu09} under these types of growth
conditions.  The case that $\growthpow = 2$ corresponds (roughly) to strong
convexity, and in this case we get the approximate rate of convergence of
$\frac{L^2}{\lambda} \frac{d \log \frac{n}{d}}{n}$, the familiar
rate of convergence under these conditions. Of course, if there
is too much variance penalization (i.e.\ $\tol$ is too large), then
the rates of convergence may be slower.

\begin{proof}
  That $\solset$ is a singleton implies
  that $\solset^{2 \epsilon} \subset \{\theta \mid \norms{\theta - \theta\opt}
  \le (2 \epsilon / \lambda)^\frac{1}{\growthpow}\}$. Moreover,
  in this case we also have
  that
  \begin{equation*}
    \left|\E_\emp[\loss(\theta; \statrv)
      - \loss(\theta\opt; \statrv)]
    - \E_\emp[\loss(\theta'; \statrv)
      - \loss(\theta\opt; \statrv)]\right|
    \le L \norm{\theta - \theta'},
  \end{equation*}
  so that an $\epsilon / L$-cover of $\{\theta \mid \norm{\theta - \theta\opt}
  \le (2 \epsilon / \lambda)^\frac{1}{\growthpow}\}$ is an $\epsilon$-cover of
  the function class $\fclass = \{f(x) = \loss(\theta; x) - \loss(\theta\opt;
  x) \mid \theta \in \solset^{2\epsilon}\}$ in $\ltwopn{\cdot}$ norm.  Thus,
  the standard Dudley entropy integral~\cite{Dudley99,VanDerVaartWe96}
  yields
  \begin{align*}
    \E[\radcomp_n(\solset^{2\epsilon})]
    & \lesssim \frac{1}{\sqrt{n}}
    \int_0^\infty \sqrt{\log \covnum(\fclass, \delta, \ltwopn{\cdot})}
    d \delta \\
    & \lesssim \frac{1}{\sqrt{n}}
    \int_0^{L (2 \epsilon / \lambda)^\frac{1}{\growthpow}}
    \sqrt{d \log \frac{L}{\delta}} d\delta
    \le
    L \sqrt{\frac{d}{n}}
    \left(\frac{2 \epsilon}{\lambda}\right)^\frac{1}{\growthpow}
    \sqrt{1 + \frac{1}{\growthpow} \log\frac{\lambda}{2L^\growthpow\epsilon}}
  \end{align*}
  where we have used that $\int_0^\varepsilon \sqrt{\log \frac{L}{\delta}}
  d\delta \le \varepsilon \sqrt{1 + \log \frac{L}{\varepsilon}}$.
  Solving for
  $\epsilon$ in the localization
  inequality~\eqref{eqn:fast-rate-rademacher-bound} then
  yields the corollary, showing that the specified choice of 
  $\epsilon$ is sufficient for
  all the conditions~\eqref{eqn:fast-rate-rademacher-bound}
  to hold.
\end{proof}


\vspace{-20pt}

\subsection{Asymptotics}
\label{sec:asymptotics}

It is important to understand the precise limiting behavior of the robust
estimator in addition to its finite sample properties---this allows us to
more precisely characterize when there may be \emph{degradation} relative to
classical risk minimization strategies. With that in mind, in this section
we provide asymptotic results for the robust
solution~\eqref{eqn:solve-robust} to better understand the consequences of
penalizing the variance of the loss itself. In particular, we would like to
understand efficiency losses relative to (say) maximum likelihood in
situations in which maximum likelihood is efficient. Before stating the
results, we make a few standard assumptions on the risk $\risk(\theta)$, the
loss $\loss$, and the moments of $\loss$ and its derivatives.  Concretely,
we assume that
\begin{equation*}
  \theta\opt \defeq \argmin_\theta \risk(\theta)
  ~~ \mbox{and} ~~
  \nabla^2 \risk(\theta\opt) \succ 0,
\end{equation*}
that is, the risk functional has strictly positive definite Hessian at
$\theta\opt$, which is thus unique.  Additionally, we have the following
smoothness assumptions on the loss function, which are satisfied by common
loss functions, including the negative log-likelihood for any
exponential family or generalized linear model~\cite{LehmannCa98}.  In the
assumption, we let $\ball$ denote the $\ell_2$-ball of radius $1$ in $\R^d$.
\begin{assumption}
  \label{assumption:smoothness-assumptions}
  For some $\epsilon > 0$, there exists a function
  $\lipobj : \statdomain \to \R_+$ satisfying
  \begin{equation*}
    |\loss(\theta, \statval) - \loss(\theta', \statval)| \le
    \lipobj(\statval) \ltwo{\theta - \theta'}
    ~~ \mbox{for~} \theta, \theta' \in \theta\opt + \epsilon \ball
  \end{equation*}
  and $\E[\lipobj(\statrv)^2] \le \lipobj(P) < \infty$.
  Additionally, there is a function $\liphess$ such that the function
  $\theta \mapsto \loss(\theta, \statval)$ has $\liphess(\statval)$-Lipschitz
  continuous Hessian (with respect to the Frobenius norm) on
  $\theta\opt + \epsilon \ball$, where $\E[\liphess(\statrv)^2] < \infty$.
\end{assumption}

Then, recalling the robust estimator~\eqref{eqn:solve-robust} as the minimizer
of $\risk_n(\theta, \mc{P}_n)$, we have the following theorem, which
we prove in Section~\ref{sec:proof-asymptotic-convergence}.
\begin{theorem}
  \label{theorem:asymptotic-convergence}
  Let Assumption~\ref{assumption:smoothness-assumptions} hold, and
  let the sequence $\robsol$ be defined by
  $\robsol \in \argmin_\theta \risk_n(\theta, \mc{P}_n)$. Define
  \ifdefined\useaosstyle
  \begin{align*}
    b(\theta\opt) & \defeq
    \frac{\cov(\nabla_\theta \loss(\theta\opt, \statrv),
      \loss(\theta\opt, \statrv))}{\sqrt{\var(\loss(\theta\opt, \statrv))}} \\
    \Sigma(\theta\opt) & = \left(\nabla^2 \risk(\theta\opt)\right)^{-1}
    \cov(\nabla\loss(\theta\opt, \statrv)) 
    \left(\nabla^2 \risk(\theta\opt)\right)^{-1}.
  \end{align*}
  \else
  \begin{equation*}
    b(\theta\opt) \defeq
    \frac{\cov(\nabla_\theta \loss(\theta\opt, \statrv),
      \loss(\theta\opt, \statrv))}{\sqrt{\var(\loss(\theta\opt, \statrv))}}
    ~~ \mbox{and} ~~
    \Sigma(\theta\opt) = \left(\nabla^2 \risk(\theta\opt)\right)^{-1}
    \cov(\nabla\loss(\theta\opt, \statrv)) 
    \left(\nabla^2 \risk(\theta\opt)\right)^{-1}.
  \end{equation*}
  \fi
  Then
  $\robsol \cas \theta\opt$
  and
  \begin{equation*}
    \sqrt{n} (\robsol - \theta\opt)
    \cd \normal\left(-\sqrt{2 \tol} \, b(\theta\opt),
      \Sigma(\theta\opt) \right)
  \end{equation*}
\end{theorem}

The asymptotic variance $\Sigma(\theta\opt)$ in
Theorem~\ref{theorem:asymptotic-convergence} is generally unimprovable, as
made apparent by Le Cam's local asymptotic normality theory and the H\'ajek-Le
Cam local minimax theorems~\cite{VanDerVaartWe96}. Thus,
Theorem~\ref{theorem:asymptotic-convergence} shows that the robust regularized
estimator~\eqref{eqn:solve-robust} has some efficiency loss, but it is only in
the bias term.  We explore this a bit more in the context of the risk of
$\robsol$.  Letting $W \sim \normal(0, \Sigma(\theta\opt))$, as an
immediate corollary to this theorem, the delta-method implies that
\begin{equation}
  \label{eqn:delta-method-application}
  n \left[\risk(\robsol) - \risk(\theta\opt)\right]
  \cd \half \norm{\sqrt{2 \tol} \, b(\theta\opt) + W}_{\nabla^2 \risk(\theta\opt)}^2,
\end{equation}
where we recall that $\norm{x}_A^2 = x^\top A x$. 
This follows from a Taylor expansion, because
$\nabla \risk(\theta\opt) = 0$ and so
$\risk(\theta) - \risk(\theta\opt) = \half (\theta - \theta\opt)^\top \nabla^2
\risk(\theta\opt) (\theta - \theta\opt) + o(\norms{\theta - \theta\opt}^2)$, or
\begin{align*}
  n (\risk(\robsol) - \risk(\theta\opt))
  & = n \left(
    \half (\robsol - \theta\opt)^\top \nabla^2 \risk(\theta\opt)
    (\robsol - \theta\opt)
    + o(\norms{\robsol - \theta\opt}^2)
    \right) \\
  & = \half \left(\sqrt{n}(\robsol - \theta\opt) \right)^\top
    \nabla^2 \risk(\theta\opt) \left(\sqrt{n}(\robsol  - \theta\opt) \right)
    + o_P(1) \\
  & \cd \half (\sqrt{2\tol} \, b(\theta\opt) + W)^\top \nabla^2 \risk(\theta\opt)
    (\sqrt{2 \tol} \, b(\theta\opt) + W)
\end{align*}
by Theorem~\ref{theorem:asymptotic-convergence}.

The limiting random
variable in expression~\eqref{eqn:delta-method-application} has expectation
\ifdefined\useaosstyle
\begin{align*}
  & \half \E[\norms{\sqrt{2\tol} b(\theta\opt) + W}_{\nabla^2 \risk(\theta\opt)}^2] \\
  & = \tol
    b(\theta\opt)^\top \nabla^2 \risk(\theta\opt) b(\theta\opt) + \half \tr(\nabla^2
    \risk(\theta\opt)^{-1} \cov(\loss(\theta\opt, \statrv)),
\end{align*}
\else
\begin{equation*}
  \half \E[\norms{\sqrt{2\tol} b(\theta\opt) + W}_{\nabla^2 \risk(\theta\opt)}^2] = \tol
  b(\theta\opt)^\top \nabla^2 \risk(\theta\opt) b(\theta\opt) + \half \tr(\nabla^2
  \risk(\theta\opt)^{-1} \cov(\loss(\theta\opt, \statrv)),
\end{equation*}
\fi
while the classical empirical risk minimization procedure (standard
$M$-estimation)~\cite{LehmannCa98, VanDerVaartWe96} has limiting mean-squared
error
$\half \tr(\nabla^2 \risk(\theta\opt)^{-1} \cov(\loss(\theta\opt, \statrv)))$.
Thus there is an additional
$\tol \norm{b(\theta\opt)}_{\nabla^2 \risk(\theta\opt)}^2$ penalty in the
asymptotic risk (at a rate of $1/n$) for the robustly-regularized estimator.
An inspection of the proof of Theorem~\ref{theorem:asymptotic-convergence}
reveals that
$b(\theta\opt) = \nabla_\theta \sqrt{\var(\loss(\theta\opt, \statrv))}$;
if the variance of the loss is stable near $\theta\opt$, so that
moving to a parameter $\theta = \theta\opt + \Delta$ for some small $\Delta$
has little effect on the variance, then the standard loss terms dominate, and
robust regularization has asymptotically little effect. On the other hand,
highly unstable loss functions for which
$\nabla_\theta \sqrt{\var(\loss(\theta\opt, \statrv))}$ is large yield
substantial bias.

We conclude our study of the asymptotics with a (to us) somewhat surprising
example. Consider the classical linear regression setting in which
$y = x^\top \theta\opt + \varepsilon$, where
$\varepsilon \sim \normal(0, \sigma^2)$. Using the standard squared error
loss $\loss(\theta, (x, y)) = \half (\theta^\top x - y)^2$, we obtain that
\begin{equation*}
  \nabla \loss(\theta\opt, (x, y))
  = (x^\top\theta\opt - y) x
  = (x^\top \theta\opt - x^\top \theta\opt - \varepsilon) x
  = -\varepsilon x,
\end{equation*}
while $\loss(\theta\opt, (x, y)) = \half \varepsilon^2$. The covariance
$\cov(\varepsilon X, \varepsilon^2) = \E[\varepsilon X (\varepsilon^2 -
\sigma^2)] = 0$
by symmetry of the error distribution, and so---in the special classical case
of correctly specified linear regression---the bias term $b(\theta\opt) = 0$
for linear regression in Theorem~\ref{theorem:asymptotic-convergence}.  That
is, the robustly regularized estimator~\eqref{eqn:solve-robust} is
asymptotically efficient.


\section{Experiments}
\label{sec:experiments}

\blfootnote{Code is available at \url{https://github.com/hsnamkoong/robustopt}.}
We present three experiments in this section. The first is a small simulation
example, which serves as a proof of concept allowing careful comparison of
standard empirical risk minimization (ERM) strategies to our
variance-regularized approach. The latter two are classification problems on
real datasets; for both of these we compare performance of robust
solution~\eqref{eqn:solve-robust} to its ERM counterpart.

\subsection{Minimizing the robust objective}

As a first step, we give a brief description of our (essentially standard)
method for solving the robust risk problem. Our work in this paper focuses
mainly on the properties of the robust objective~\eqref{eqn:robust-risk} and
its minimizers~\eqref{eqn:solve-robust}, so we only briefly describe the
algorithm we use; we leave developing faster and more accurate specialized
methods to further work. To solve the robust problem, we use a gradient
descent-based procedure, and we focus on the case in which the empirical
sampled losses $\{\loss(\theta, \statrv_i)\}_{i = 1}^n$ have non-zero variance
for all parameters $\theta \in \Theta$, which is the case for all of our
experiments.

Recall the definition of the subdifferential
$\partial f(\theta) = \{g \in \R^d : f(\theta') \ge f(\theta) + \<g, \theta' -
\theta\> ~ \mbox{for all}~\theta'\}$,
which is simply the gradient for differentiable functions $f$.  A standard
result in convex analysis~\cite[Theorem VI.4.4.2]{HiriartUrrutyLe93ab} is that
if the vector $p^* \in \R^n_+$ achieving the supremum in the
definition~\eqref{eqn:robust-risk} of the robust risk is unique, then
\begin{equation*}
  \partial_\theta \risk_n(\theta, \mc{P}_n)
  = \partial_\theta \sup_{P \in \mc{P}_n} \E_P[\loss(\theta; \statrv)]
  = \sum_{i = 1}^n p_i^* \partial_\theta \loss(\theta; \statrv_i),
\end{equation*}
where the final summation is the standard Minkowski sum of sets. As
this maximizing vector $p$ is indeed unique whenever
$\var_\emp(\loss(\theta; \statrv)) \neq 0$, we see that for all our
problems, so long as $\loss$ is differentiable, so too is
$\risk_n(\theta, \mc{P}_n)$ and
\ifdefined\useaosstyle
\begin{equation}
  \label{eqn:derivative-robust-risk}
  \nabla_\theta \risk_n(\theta, \mc{P}_n) = \sum_{i = 1}^n p_i^* \nabla_\theta
  \loss(\theta; \statrv_i)
\end{equation}
where
$p^* = \argmax_{p \in \mc{P}_n} \bigg\{\sum_{i = 1}^n p_i \loss(\theta;
\statrv_i) \bigg\}$.
\else
\begin{equation}
  \label{eqn:derivative-robust-risk}
  \nabla_\theta \risk_n(\theta, \mc{P}_n) = \sum_{i = 1}^n p_i^* \nabla_\theta
  \loss(\theta; \statrv_i)
  ~~ \mbox{where} ~~
  p^* = \argmax_{p \in \mc{P}_n}
  \bigg\{\sum_{i = 1}^n p_i \loss(\theta; \statrv_i) \bigg\}.
\end{equation}
\fi In order to perform gradient descent on the risk
$\risk_n(\theta, \mc{P}_n)$, then, by
equation~\eqref{eqn:derivative-robust-risk} we require only the computation of
the worst-case distribution $p^*$. By taking the dual of the
maximization~\eqref{eqn:derivative-robust-risk}, this is an efficiently
solvable convex problem; for completeness, we provide a procedure for this
computation in Section~\ref{appendix:efficient-alg} that requires
time 
$O(n \log n + \log \frac{1}{\epsilon} \log n)$ to compute an
$\epsilon$-accurate solution to the
maximization~\eqref{eqn:derivative-robust-risk}. As all our examples have
smooth objectives, we perform gradient descent on the robust risk
$\risk_n(\cdot, \mc{P}_n)$, with stepsizes chosen by a backtracking (Armijo)
line search~\cite[Chapter 9.2]{BoydVa04}.

\subsection{Simulation experiment}

For our simulation experiment, we use a quadratic loss with linear
perturbation. For $v, \statval \in \R^d$, define the loss
$\loss(\theta; \statval) = \half \ltwo{\theta - v}^2 + \statval^\top(\theta -
v)$.
We set $d = 50$ and take $\statrv \sim \uniform(\{-B, B\}^d)$, varying $B$ in
the experiment. For concreteness, we let the domain
$\Theta = \{\theta \in \R^d : \ltwo{\theta} \le \radius\}$ and set
$v = \frac{\radius}{2\sqrt{d}} \ones$, so that $v \in \interior \Theta$; we
take $\radius = 10$.  Notably, standard regularization strategies, such as
$\ell_1$ or $\ell_2$-regularization, pull $\theta$ toward 0, while the
variance of $\loss(\theta; \statrv)$ is minimized by $\theta = v$ (thus
naturally advantaging the variance-based regularization we consider, as
$\risk(v) = \inf_\theta \risk(\theta) = 0$). Moreover, as $X$ is pure noise,
this is an example where we expect variance regularization to be particularly
useful. We choose $\delta = .05$ and set $\tol$
as in
Eq.~\eqref{eqn:selecting-tol-for-coverage} (using that $\loss$ is
$(3 \radius + \sqrt{d} B)$-Lipschitz) to obtain robust coverage with
probability at least $1 - \delta$.  In our experiments, we obtained
$100\%$ coverage in the sense of~\eqref{eqn:coverage-F}, as the high
probability bound is conservative.

\begin{figure}[ht]
  \begin{center}
    \includegraphics[width=.7\columnwidth]{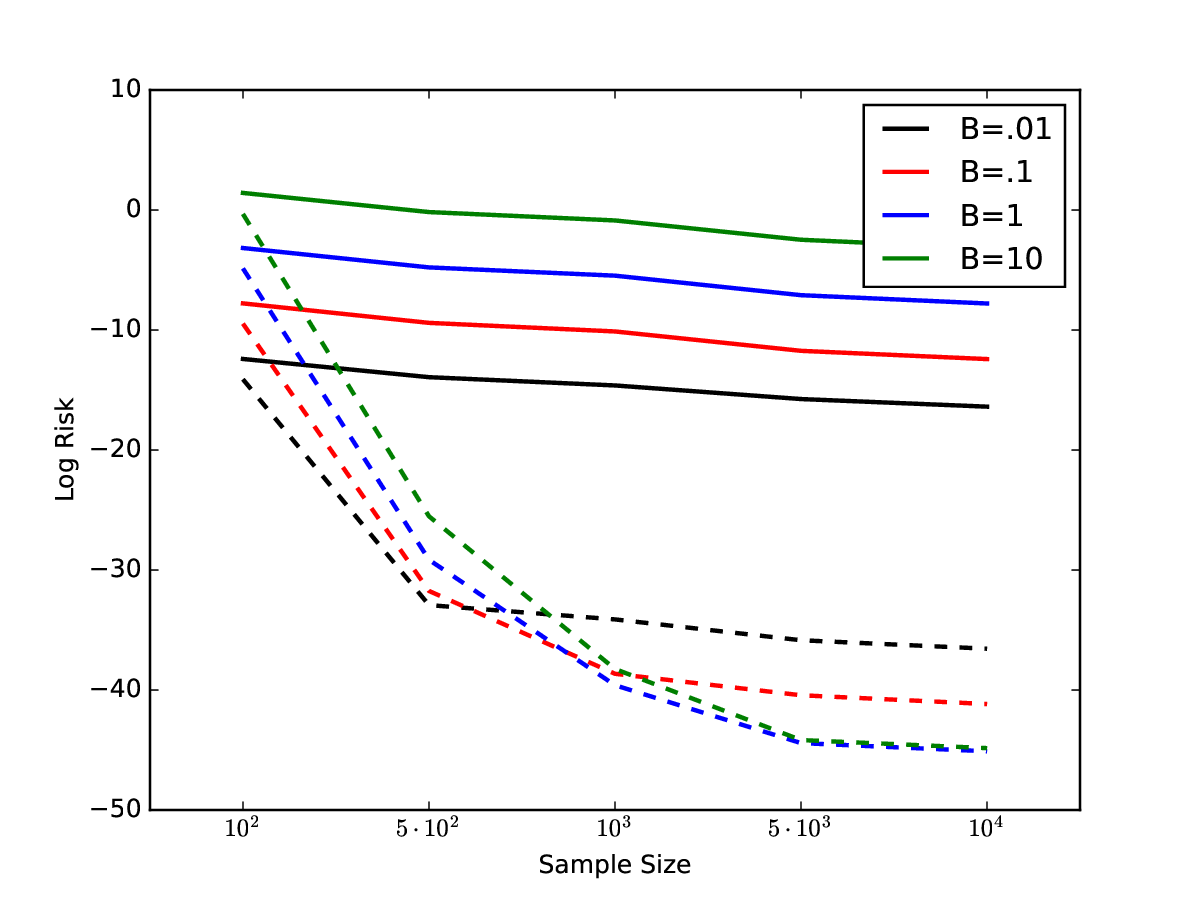} 
    \caption{  \label{fig:sim} 
      Simulation experiment. $\log \E[\risk(\ermsol)]$ is the
      solid lines, in decreasing order from $B = 10$ (top) to $B = .01$ (bottom).
      $\log \E[\risk(\robsol)]$ is the dashed line, in the same vertical ordering
      at sample size $n = 10^2$.}
  \end{center}
\end{figure}

Figure~\ref{fig:sim} summarizes the results.  The robust solution
$\robsol = \argmin_{\theta \in \Theta} \risk_n(\theta, \mc{P}_n)$ always
outperforms the empirical risk minimizer
$\ermsol = \argmin_{\theta \in \Theta} \E_\emp[\loss(\theta, \statrv)]$ in
terms of the true risk
$\E[\loss(\theta, \statrv)] = \half \ltwo{\theta - v}^2$. Each experiment
consists of 1,200 independent replications for each sample size $n$ and value
$B$.  In Tables~\ref{table:sim-mean} and~\ref{table:sim-var}, we display the
risks of $\ermsol$ and $\robsol$ and variances, respectively, computed for the
1,200 independent trials. The gap between the risk of $\ermsol$
and $\robsol$ is siginificant at level $p < .01$ for all sample sizes and values of
$B$ we considered according to a one-sided T-test. Notice also
in Table~\ref{table:sim-var} that the variance of the robust solutions
is substantially smaller than that of the empirical risk minimizer---often
several orders of magnitude smaller for large sample sizes $n$.
This simulation shows that---in a simple setting favorable to
it---our procedure outperforms standard alternatives.

\begin{table}[ht]
  \centering
  \caption{Simulation experiment: Mean risks over 1,200 simulations}
  \label{table:sim-mean}
  \begin{filecontents*}{./Experiments/UniformSimulation/uniform.csv}
sample,.01_erm,.01_rob,.1_erm,.1_rob,1_erm,1_rob,10_erm,10_rob
100,4.06E-06,7.42E-07,4.17E-04,7.65E-05,4.20E-02,7.64E-03,4.15E+00,7.12E-01
500,8.91E-07,5.01E-15,8.22E-05,1.63E-14,8.36E-03,2.19E-13,8.41E-01,8.21E-12
1000,4.47E-07,1.52E-15,4.02E-05,1.64E-17,4.20E-03,6.32E-18,4.19E-01,2.45E-17
5000,1.44E-07,2.68E-16,8.00E-06,2.74E-18,8.27E-04,5.09E-20,8.38E-02,6.55E-20
10000,7.64E-08,1.32E-16,4.02E-06,1.32E-18,4.13E-04,2.57E-20,4.18E-02,3.34E-20
\end{filecontents*}
\pgfplotstabletypeset[
  col sep=comma,
  string type,
  every head row/.style={%
    before row={\hline 
      & \multicolumn{2}{c}{$B = .01$} 
      & \multicolumn{2}{c}{$B = .1$} 
      & \multicolumn{2}{c}{$B = 1$} 
      & \multicolumn{2}{c}{$B = 10$}
      \\
    },
    after row=\hline
  },
  every last row/.style={after row=\hline},
  columns/sample/.style={column name=$n$, column type=l},
  columns/.01_erm/.style={column name=$\risk(\ermsol)$, column type=l},
  columns/.01_rob/.style={column name=$\risk(\robsol)$, column type=c},
  columns/.1_erm/.style={column name=$\risk(\ermsol)$, column type=l},
  columns/.1_rob/.style={column name=$\risk(\robsol)$, column type=c},
  columns/1_erm/.style={column name=$\risk(\ermsol)$, column type=l},
  columns/1_rob/.style={column name=$\risk(\robsol)$, column type=c},
  columns/10_erm/.style={column name=$\risk(\ermsol)$, column type=l},
  columns/10_rob/.style={column name=$\risk(\robsol)$, column type=c}
  ]{./Experiments/UniformSimulation/uniform.csv}
\end{table}
\begin{table}[ht]
  \centering
  \caption{\label{table:sim-var} Simulation experiment: Variances
    of $\risk(\what{\theta})$ over 1,200 simulations}
  \begin{filecontents*}{./Experiments/UniformSimulation/uniform_variance.csv}
sample,.01_erm,.01_rob,.1_erm,.1_rob,1_erm,1_rob,10_erm,10_rob
100,7.06E-13,9.76E-14,6.58E-09,1.03E-09,7.09E-05,1.08E-05,7.37E-01,9.20E-02
500,5.98E-14,7.15E-28,3.04E-10,3.52E-26,2.80E-06,2.26E-24,2.92E-02,3.26E-21
1000,2.63E-14,1.07E-31,7.53E-11,1.99E-35,7.14E-07,3.44E-33,7.03E-03,4.78E-32
5000,7.34E-15,2.94E-33,2.70E-12,3.28E-37,2.95E-08,2.50E-39,2.74E-04,5.24E-38
10000,1.60E-15,6.54E-34,6.74E-13,7.59E-38,7.04E-09,3.34E-39,6.52E-05,2.25E-38
\end{filecontents*}
\pgfplotstabletypeset[
  col sep=comma,
  string type,
  every head row/.style={%
    before row={\hline 
      & \multicolumn{2}{c}{$B = .01$} 
      & \multicolumn{2}{c}{$B = .1$} 
      & \multicolumn{2}{c}{$B = 1$} 
      & \multicolumn{2}{c}{$B = 10$}
      \\
    },
    after row=\hline
  },
  every last row/.style={after row=\hline},
  columns/sample/.style={column name=$n$, column type=l},
  columns/.01_erm/.style={column name=ERM, column type=l},
  columns/.01_rob/.style={column name=Robust, column type=c},
  columns/.1_erm/.style={column name=ERM, column type=l},
  columns/.1_rob/.style={column name=Robust, column type=c},
  columns/1_erm/.style={column name=ERM, column type=l},
  columns/1_rob/.style={column name=Robust, column type=c},
  columns/10_erm/.style={column name=ERM, column type=l},
  columns/10_rob/.style={column name=Robust, column type=c}
  ]{./Experiments/UniformSimulation/uniform_variance.csv}
\end{table}

\subsection{Protease cleavage experiments}

For our second experiment, we compare our robust regularization procedure to
other regularizers using the HIV-1 protease cleavage dataset from the UCI
ML-repository~\cite{Lichman13}.  In this binary classification task, one is
given a string of amino acids (a protein) and a featurized representation of
the string of dimension $d = 50960$, and the goal is to predict whether the
HIV-1 virus will cleave the amino acid sequence in its central position.  We
have a sample of $n = 6590$ observations of this process, where the class
labels are somewhat skewed: there are $1360$ examples with label $Y = +1$
(HIV-1 cleaves) and $5230$ examples with $Y=-1$ (does not cleave).

We use the logistic loss
$\loss(\theta; (x, y)) = \log (1 + \exp(-y\theta^\top x))$. We compare the
performance of different constraint sets $\Theta$ by taking
\begin{equation*}
  \Theta
  = \left\{ \theta \in \R^d: a_1 \lone{\theta} + a_2 \ltwo{\theta} \le r
\right\},
\end{equation*}
which is equivalent to elastic net regularization~\cite{ZouHa05}, while
varying $a_1$, $a_2$, and $r$. We experiment with $\ell_1$-constraints
($a_1 = 1, a_2 = 0$) with
\ifdefined\useaosstyle
$r \in \{ 50, 100, 500,\\ 1000, 5000\}$,
\else
$r \in \{ 50, 100, 500, 1000, 5000\}$,
\fi
$\ell_2$-constraints ($a_1 = 0, a_2 = 1$) with $r \in \{ 5, 10, 50, 100, 500\}$,
elastic net ($a_1 = 1, a_2 = 10$) with
$r \in \{ 100, 200, 1000, 2000, 10000\}$, our robust regularizer with
$\tol \in \{ 100, 1000, 10000, 50000, 100000 \}$ and our robust regularizer
coupled with the $\ell_1$-constraint ($a_1 = 1, a_2 = 0$) with $r = 100$.
Though we use a convex surrogate (logistic loss), we measure performance of
the classifiers using the $0$-$1$ (misclassification) loss
$\indics{\sign(\theta^T \statval) y \le 0}$.  For validation, we perform 50
experiments, where in each experiment we randomly select $9/10$ of the data to
train the model, evaluating its performance on the held out fraction
(test).

We plot results summarizing these experiments in Figure~\ref{fig:hiv-reg}.
The horizontal axis in each figure indexes our choice of regularization value
(so ``\texttt{Regularizer = 1}'' for the $\ell_1$-constrained problem
corresponds to $\radius = 50$). The figures show that the robustly regularized
risk provides a different type of protection against overfitting than standard
regularization or constraint techniques do: while other regularizers
underperform in heavily constrained settings, the robustly regularized
estimator $\robsol$ achieves low classification error for all values of $\tol$
(Figure~\ref{fig:hiv-reg}(b)).  Notably, even when coupled with a fairly
stringent $\ell_1$-constraint ($r = 100$), robust regularization has
perofrmance better than $\ell_1$ except for large values $r$, especially on
the rare label $Y = +1$ (Figure~\ref{fig:hiv-reg} (d) and (f)).

We investigate the effects of the robust regularizer with a slightly different
perspective in Figure~\ref{fig:hiv}, where we use
$\Theta = \{\theta : \lone{\theta} \le r\}$ with $r = 100$ for the constraint
set for each experiment. The horizontal axis indicates the tolerance $\tol$ we
use in construction of the robust estimator $\robsol$, where \texttt{ERM}
means $\tol = 0$. In Fig.~\ref{fig:hiv}(a), we plot the logistic risk
$\risk(\what{\theta}) = \E[\loss(\what{\theta}, (X, Y))]$ for the train and
test distribution. We also plot the upper confidence bound
$\risk_n(\theta, \mc{P}_n)$ in this plot, which certainly over-estimates the
test risk---we hope to tighten this overestimate in future work.  In
Figure~\ref{fig:hiv}(b), we plot the misclassification error on train and test
for different values of $\tol$, along with $2$-standard-error intervals for
the $50$ runs. Figures~\ref{fig:hiv}(c) and (d) show the error rates
restricted to examples from the uncommon (c) and common (d) classes.  In
Table~\ref{table:hiv} we give explicit error rates and logistic risk values
for the different procedures. Due to the small size of the test dataset
($n_{\rm test} = 659$), the deviation across folds is somewhat large.

\begin{figure}[H]
  \begin{center}
    \begin{tabular}{cc}
      \hspace{-.4cm}
      \includegraphics[width=.5\columnwidth]{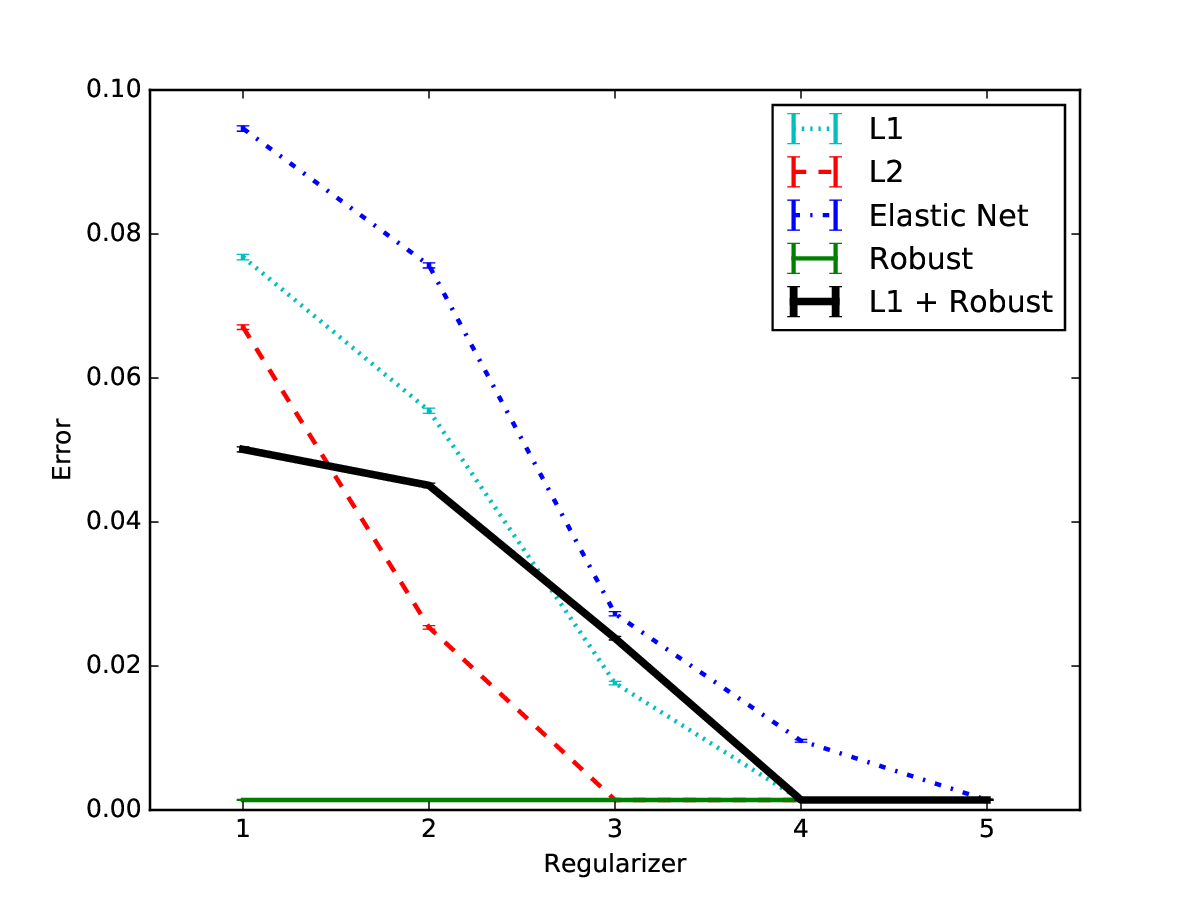} 
      & 
      \includegraphics[width=.5\columnwidth]{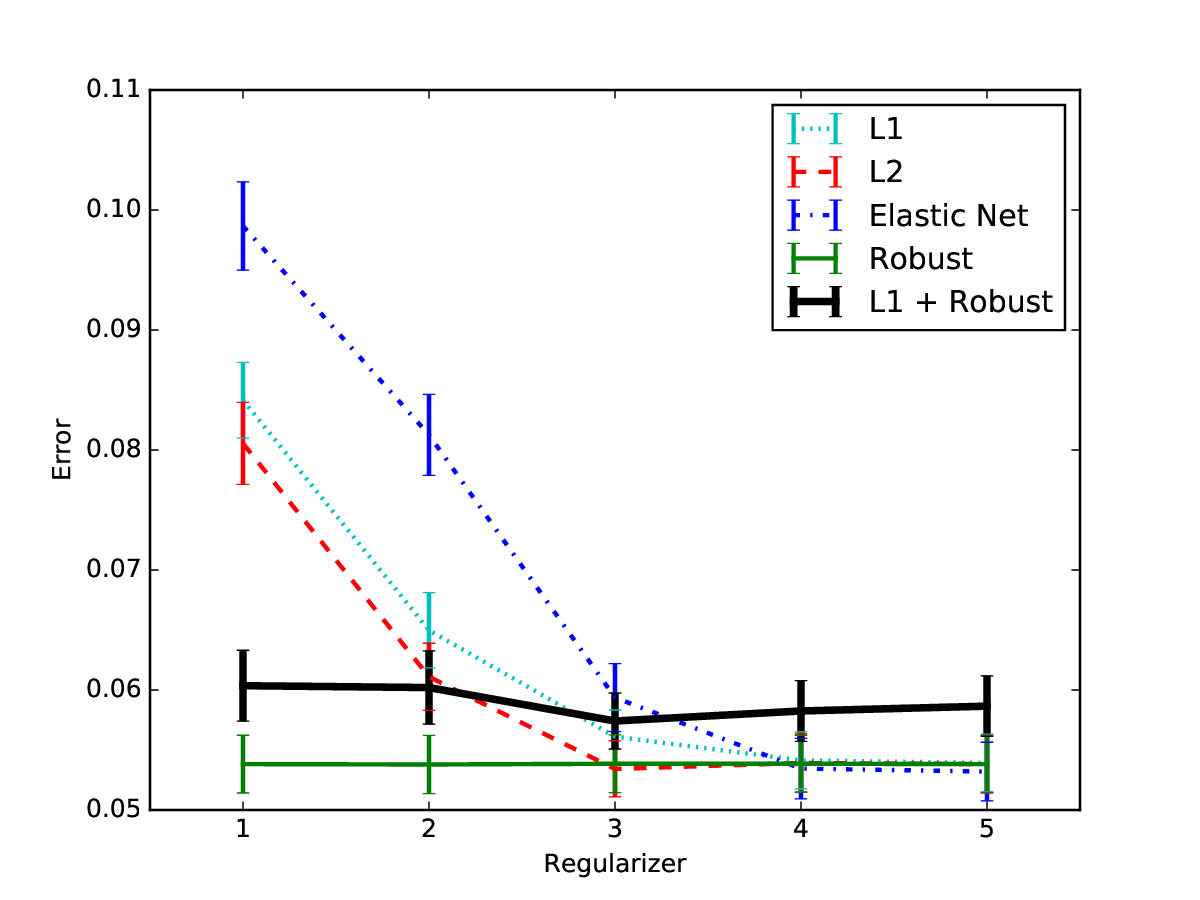}  \\
      (a) Train error 
      & (b) Test error \\
      \hspace{-.4cm}
      \includegraphics[width=.5\columnwidth]{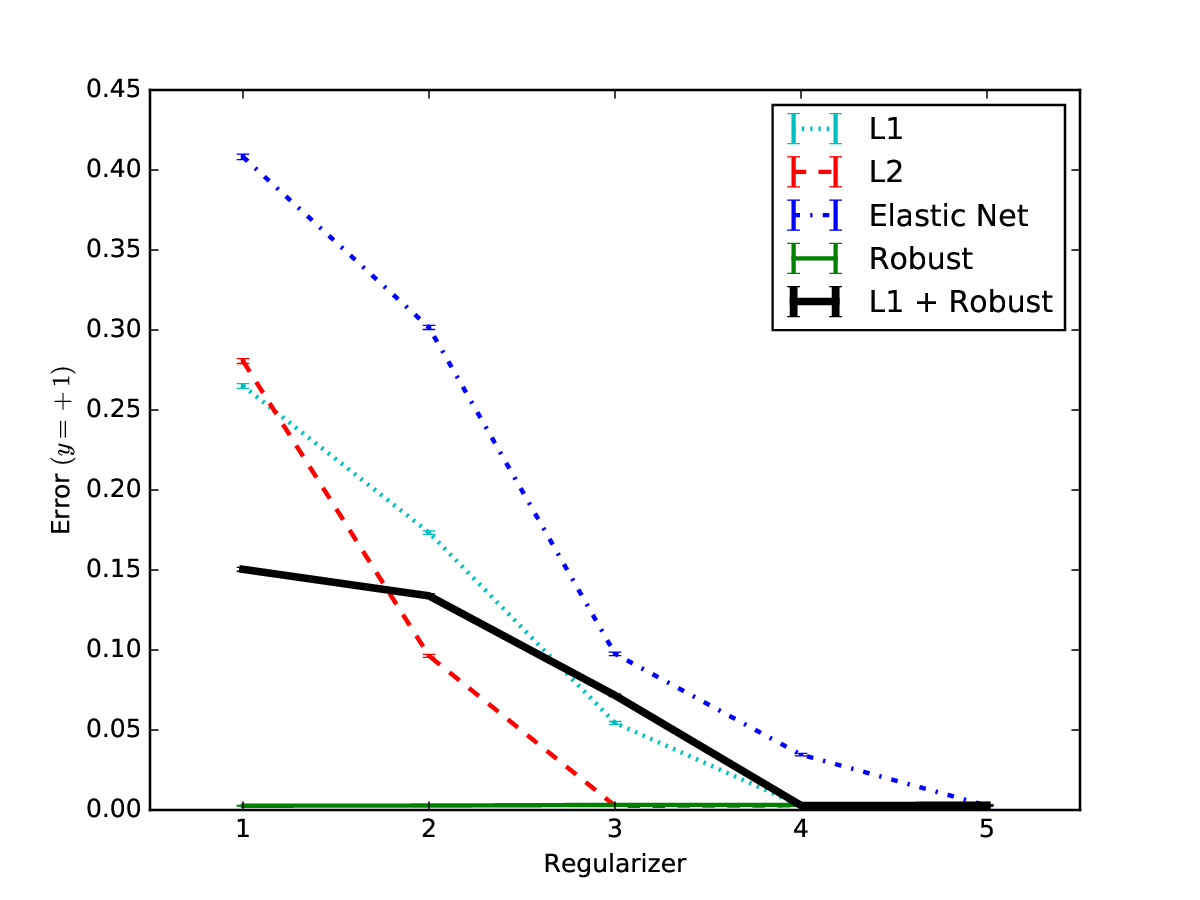} 
      & 
        \includegraphics[width=.5\columnwidth]{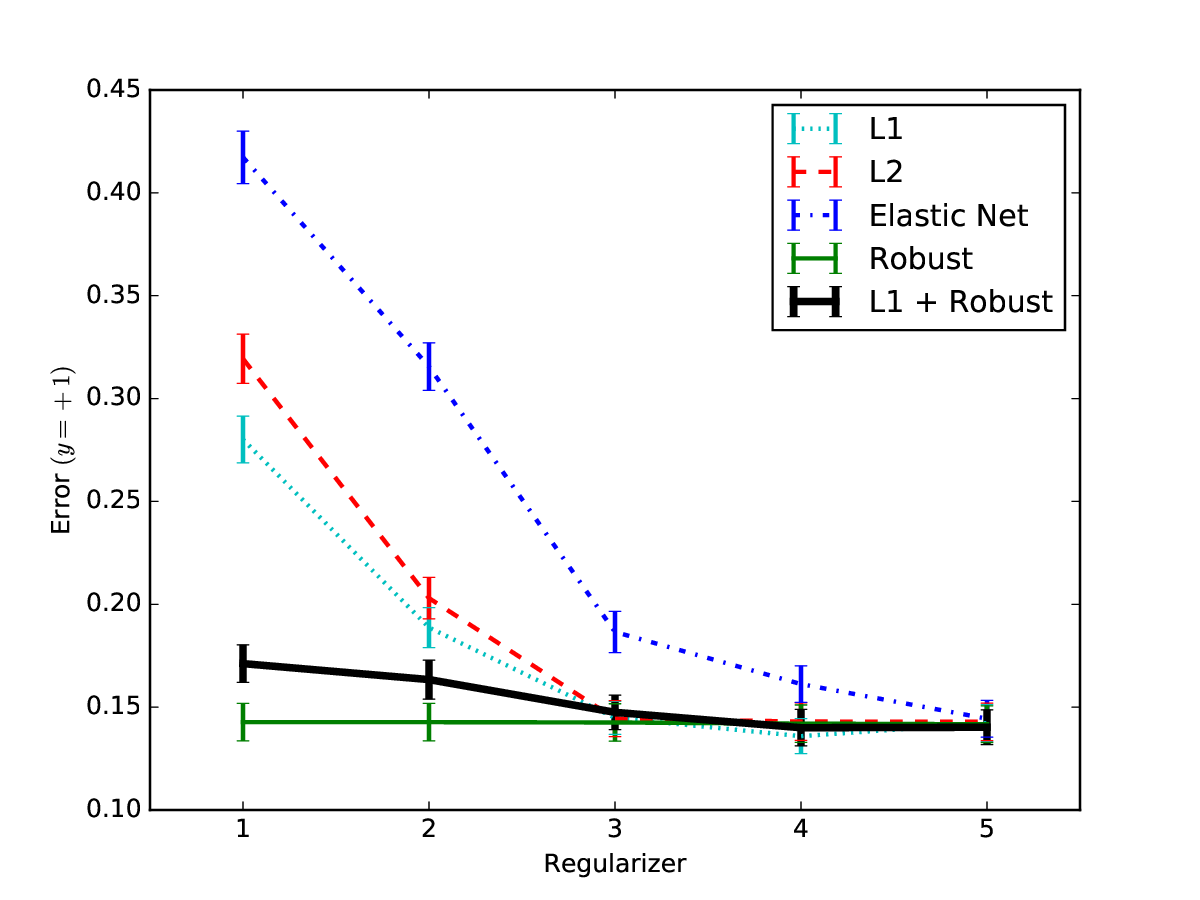}  \\
      (c) Train error on rare class ($Y_i = +1$)
      & (d) Test error on rare class ($Y_i = +1$) \\
      \hspace{-.4cm}
      \includegraphics[width=.5\columnwidth]{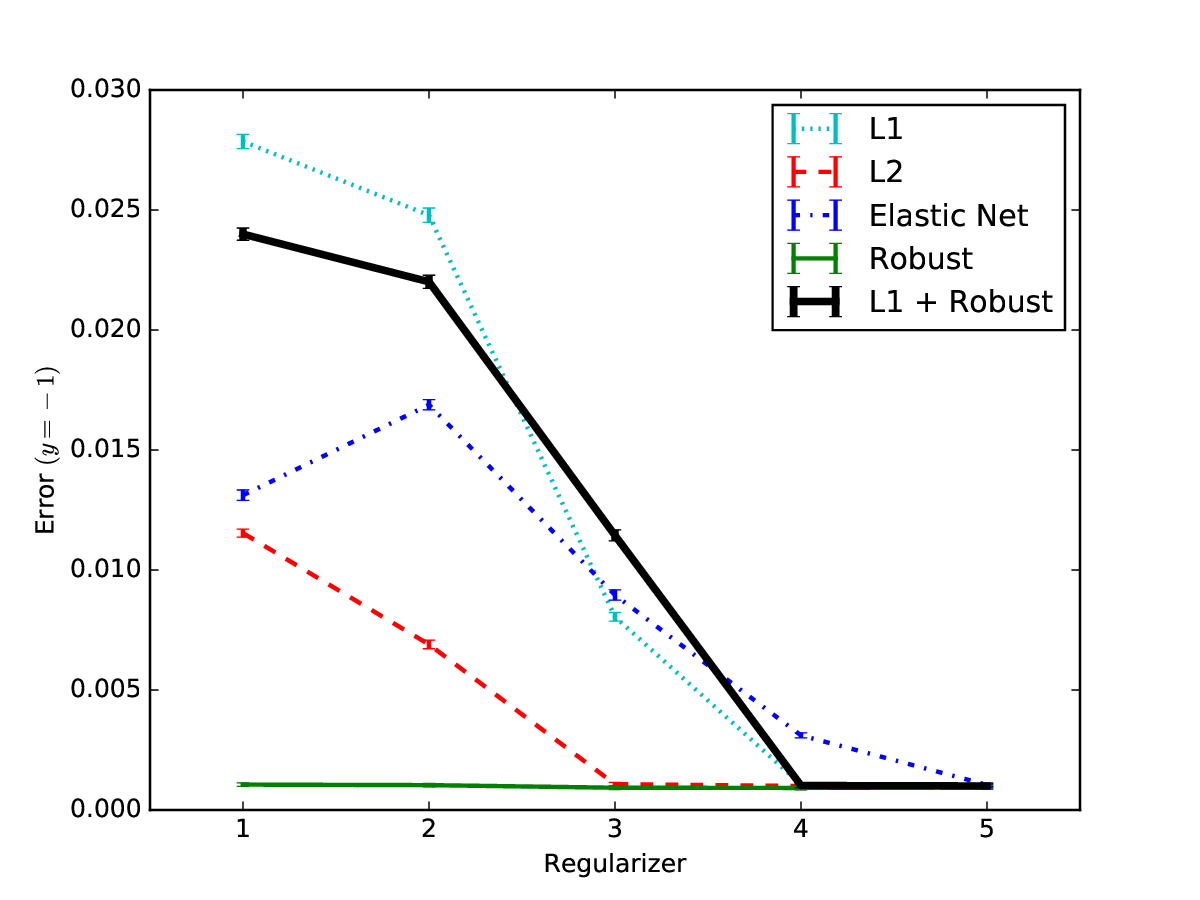} 
      & 
        \includegraphics[width=.5\columnwidth]{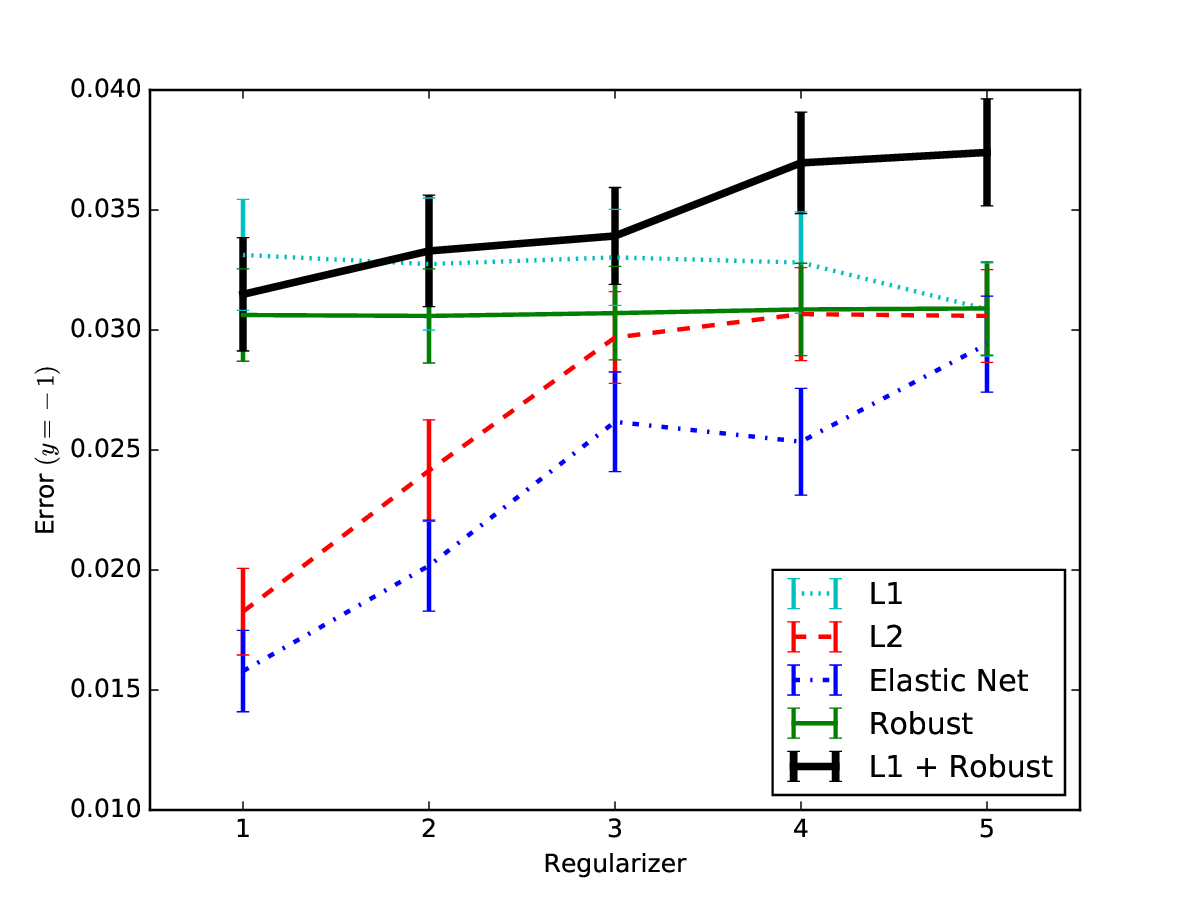}  \\
      (e) Train error on common class ($Y_i = -1$)
      & (f) Test error on common class ($Y_i = -1$)
    \end{tabular}
    \caption{ \label{fig:hiv-reg} HIV-1 Protease Cleavage plots ($2$-standard
      error confidence bars). Comparison of misclassification error rates
      among different regularizers.}
  \end{center}
\end{figure}

\begin{figure}[ht]
  \begin{center}
    \begin{tabular}{cc}
      \hspace{-.4cm}
      \includegraphics[width=.5\columnwidth]{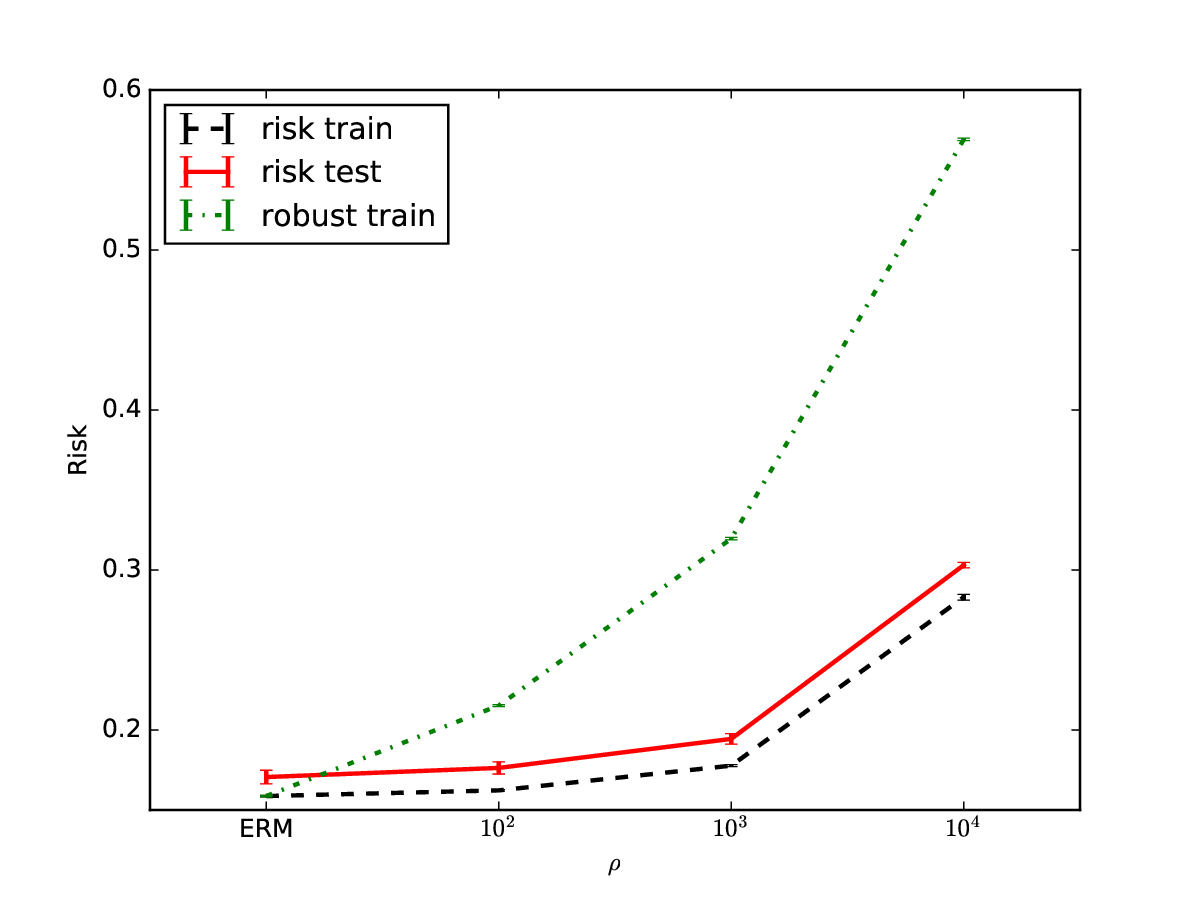} 
      & 
      \includegraphics[width=.5\columnwidth]{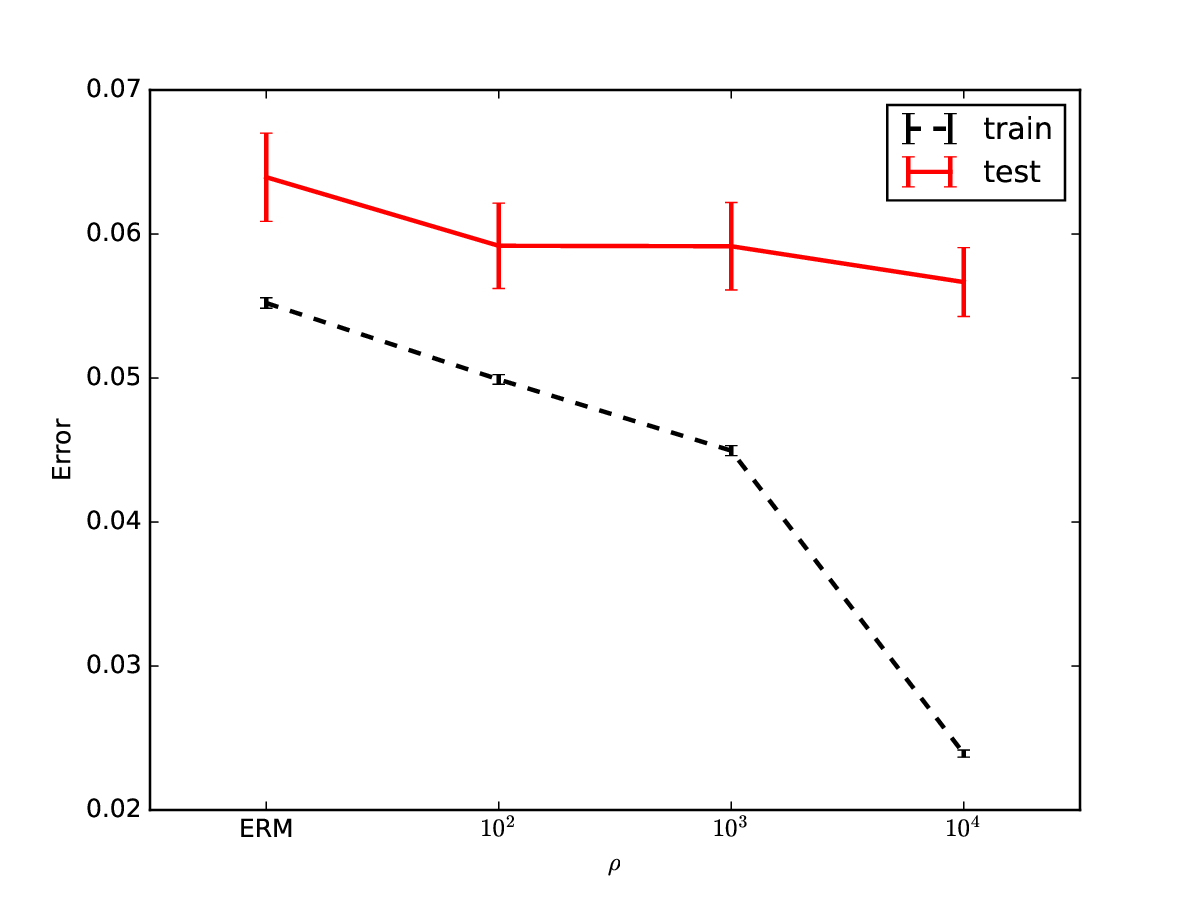}  \\
      (a) Logistic risk and confidence bound
      & (b) Misclassification error rate \\
      \hspace{-.4cm}
      \includegraphics[width=.5\columnwidth]{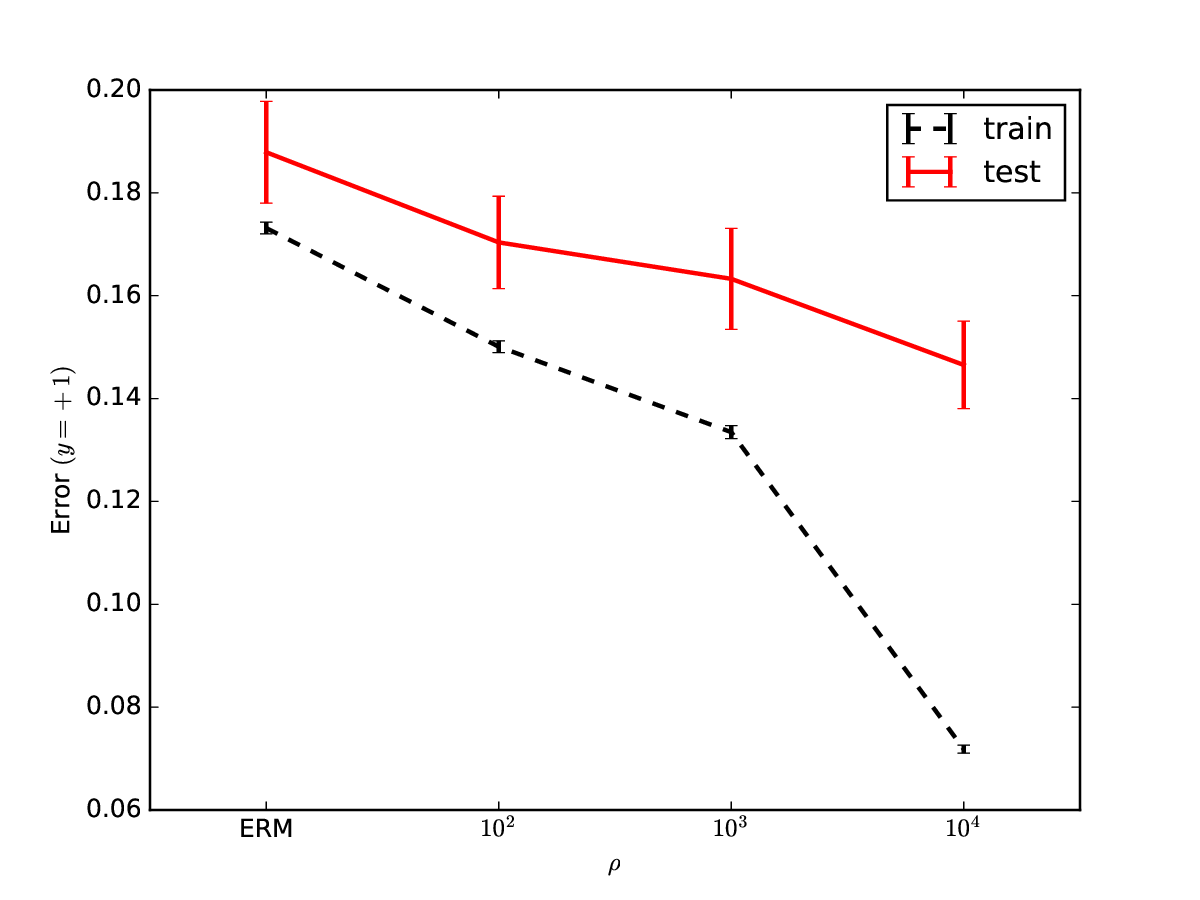} 
      & 
        \includegraphics[width=.5\columnwidth]{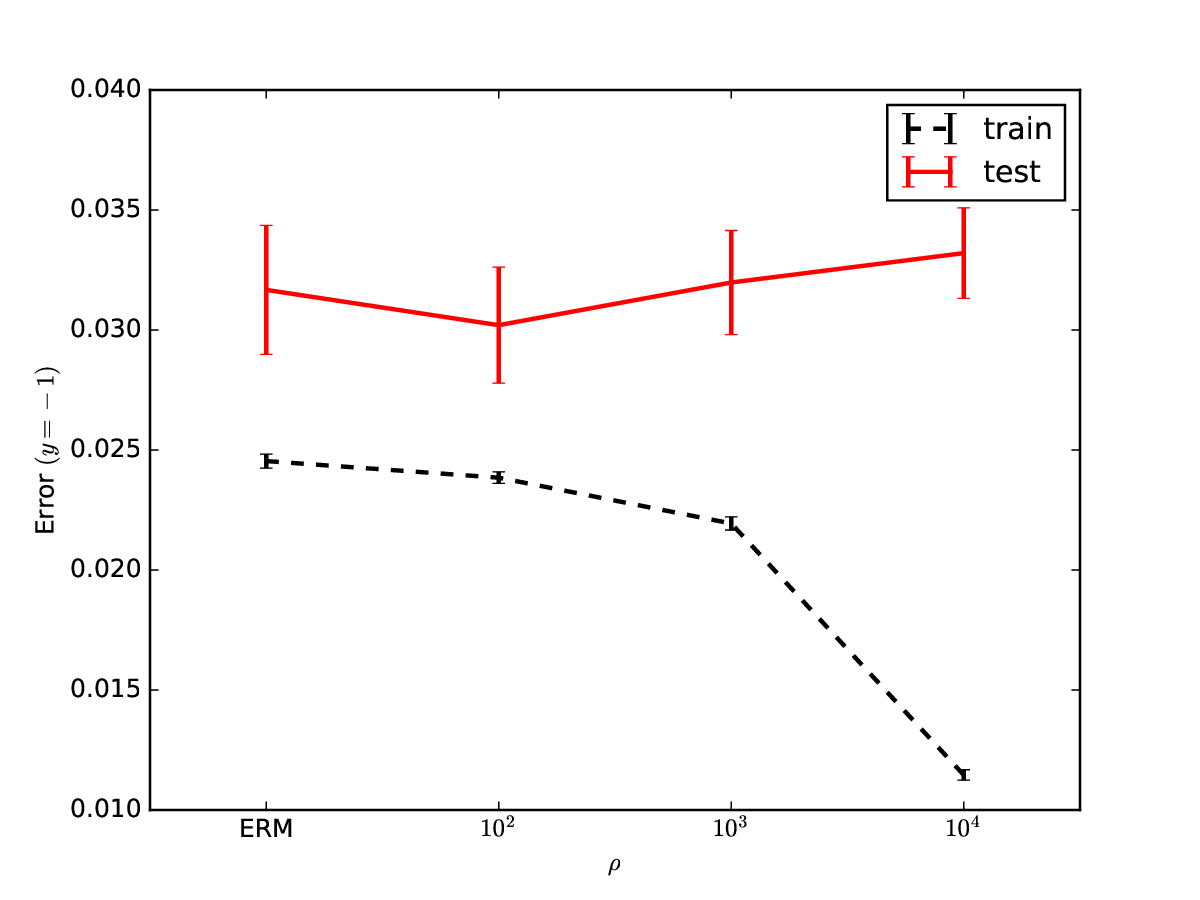}  \\
      (c) Error on rare class ($Y_i = +1$)
      & (d) Error on common class ($Y_i = -1$)
    \end{tabular}
    \caption{ \label{fig:hiv} HIV-1 Protease Cleavage plots ($2$-standard
      error confidence bars). Plot (a) shows the logistic risk
      $\risk(\theta) = \E[\log(1 + e^{-Y \theta^\top X})]$ and confidence
      bounds computed from the robust risk~\eqref{eqn:robust-risk}.  Plots
      (b)--(d) show misclassification error rates plotted against robustness
      parameter $\tol$.}
  \end{center}
\end{figure}

In this experiment, we see (roughly) that the ERM solutions achieve good
performance on the common class $(Y = -1)$ but sacrifice performance on the
uncommon class. As we increase $\tol$, performance of the robust solution
$\robsol$ on the rarer label $Y = +1$ improves (Fig.~\ref{fig:hiv}(c)), while
the misclassification rate on the common class degrades a small
(insignificant) amount (Fig.~\ref{fig:hiv}(d)); see also
Table~\ref{table:hiv}. This behavior is roughly what we might expect for the
robust estimator: the poor performance of the ERM estimator $\ermsol$ on the
rare class induces (relatively) more variance, which the robust solution
reduces by via improved classification performance on the rare $(Y = +1)$
class.  This occurs at little expense over the more common label $Y = -1$ so
that overall performance improves by a small amount. We remark---but are
unable to explain---that this improvement on classification error for the rare
labels comes despite increases in logistic risk; while the average logistic
loss increases, misclassification errors decrease.

\begin{table}[ht]
  \centering
  \caption{HIV-1 Cleavage Error}
  \label{table:hiv}
  \begin{filecontents*}{./Experiments/HIV/hiv_subsample.csv}
label,obj_train,obj_test,e_train,e_test,p_train,p_test,n_train,n_test
erm,0.1587,0.1706,5.52,6.39,17.32,18.79,2.45,3.17
100,0.1623,0.1763,4.99,5.92,15.01,17.04,2.38,3.02
1000,0.1777,0.1944,4.5,5.92,13.35,16.33,2.19,3.2
10000,0.283,0.3031,2.39,5.67,7.18,14.65,1.15,3.32
\end{filecontents*}
\pgfplotstabletypeset[
  col sep=comma,
  string type,
  every head row/.style={%
    before row={\hline
      & \multicolumn{2}{c}{risk}
      & \multicolumn{2}{c}{error (\%)}
      & \multicolumn{2}{c}{error ($Y = +1$)}
      & \multicolumn{2}{c}{error ($Y = -1$)}
      \\
    },
    after row=\hline
  },
  every last row/.style={after row=\hline},
  columns/label/.style={column name=$\rho$, column type=l},
  columns/obj_train/.style={column name=train, column type=l},
  columns/obj_test/.style={column name=test, column type=c},
  columns/e_train/.style={column name=train, column type=l},
  columns/e_test/.style={column name=test, column type=c},
  columns/p_train/.style={column name=train, column type=l},
  columns/p_test/.style={column name=test, column type=c},
  columns/n_train/.style={column name=train, column type=l},
  columns/n_test/.style={column name=test, column type=c}
  ]{./Experiments/HIV/hiv_subsample.csv}
\end{table}

\subsection{Document classification in the Reuters corpus}

For our final experiment, we consider a multi-label classification problem
with a reasonably large dataset. The Reuters RCV1 Corpus~\cite{LewisYaRoLi04}
has 804,414 examples with $d = 47,\!236$ features, where feature $j$ is an
indicator variable for whether word $j$ appears in a given document. The goal
is to classify documents as a subset of the 4 categories Corporate, Economics,
Goverment, and Markets, and each document in the data is labeled with a subset
of those. As each document can belong to multiple categories, we fit binary
classifiers on each of the four categories.  There are different numbers of
documents labeled as each category, with the Economics category having the
fewest number of positive examples.  Table~\ref{table:reuters-num-example}
gives the number of times a document is labeled as each of the four categories
(so each document has about 1.18 associated classes). In this experiment, we
expect the robust solution to outperform ERM on the rarer category
(Economics), as the robustification~\eqref{eqn:solve-robust} naturally
upweights rarer (harder) instances, which disproportionally affect
variance---as in the experiment on HIV-1 cleavage.

\begin{table}[ht]
  \centering
  \caption{Reuters Number of Examples}
  \label{table:reuters-num-example}
  \begin{tabular}{cccc}
    \hline
    Corporate & Economics & Government & Markets \\
    381,327 &  119,920 &  239,267 &  204,820 \\
    \hline
  \end{tabular}
\end{table}

For each category $k \in \{1, 2, 3, 4\}$, we use the logistic loss
$\loss(\theta_k; (x, y)) = \log (1 + \exp(-y\theta_k^\top x))$.
For each binary
classifier,
we use the $\ell_1$ constraint set $\Theta = \left\{ \theta \in \R^{d}:
\lone{\theta} \le 1000\right\}$.  To evaluate performance on this multi-label
problem, we use precision (ratio of the number of correct positive labels to
the number classified as positive) and recall (ratio of the number of correct
positive labels to the number of actual positive labels):
\begin{align*}
  \mbox{precision} & = \frac{1}{n} \sum_{i = 1}^n \frac{\sum_{k=1}^4
  \indics{\theta_k^\top x_i \ge 0, y_i = 1}}{\sum_{k=1}^4
  \indics{\theta_k^\top x_i > 0}}, \\
  \mbox{recall} & = \frac{1}{n}
  \sum_{i = 1}^n \frac{\sum_{k=1}^4 \indics{\theta_k^\top x_i \ge 0, y_i = 1}
  }{\sum_{k=1}^4 \indic{y_i =1}}.
\end{align*}
We partition the data into ten equally-sized sub-samples and
perform ten validation experiments, where in each experiment we use
one of the ten subsets for fitting the logistic models and the remaining
nine partitions as a test set to evaluate performance.

\begin{figure}[ht]
    \begin{subfigure}[1]{0.47\linewidth}
    \centering
    \includegraphics[width=1\linewidth]{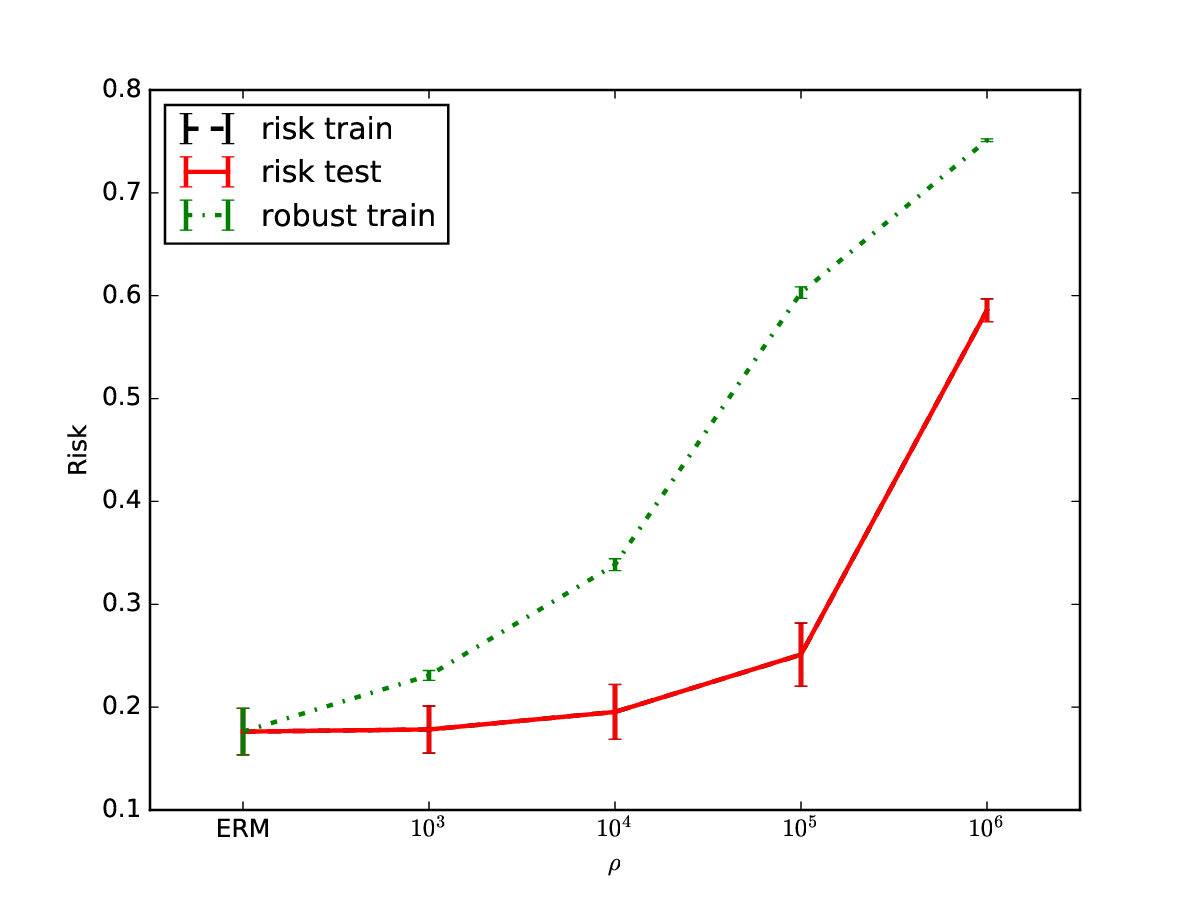} 
    \caption{Logistic risk and confidence bound} 
    \label{fig:reuters-risk} 
    \vspace{0ex}
  \end{subfigure}
  \begin{subfigure}[2]{0.47\linewidth}
    \centering
    \includegraphics[width=1\linewidth]{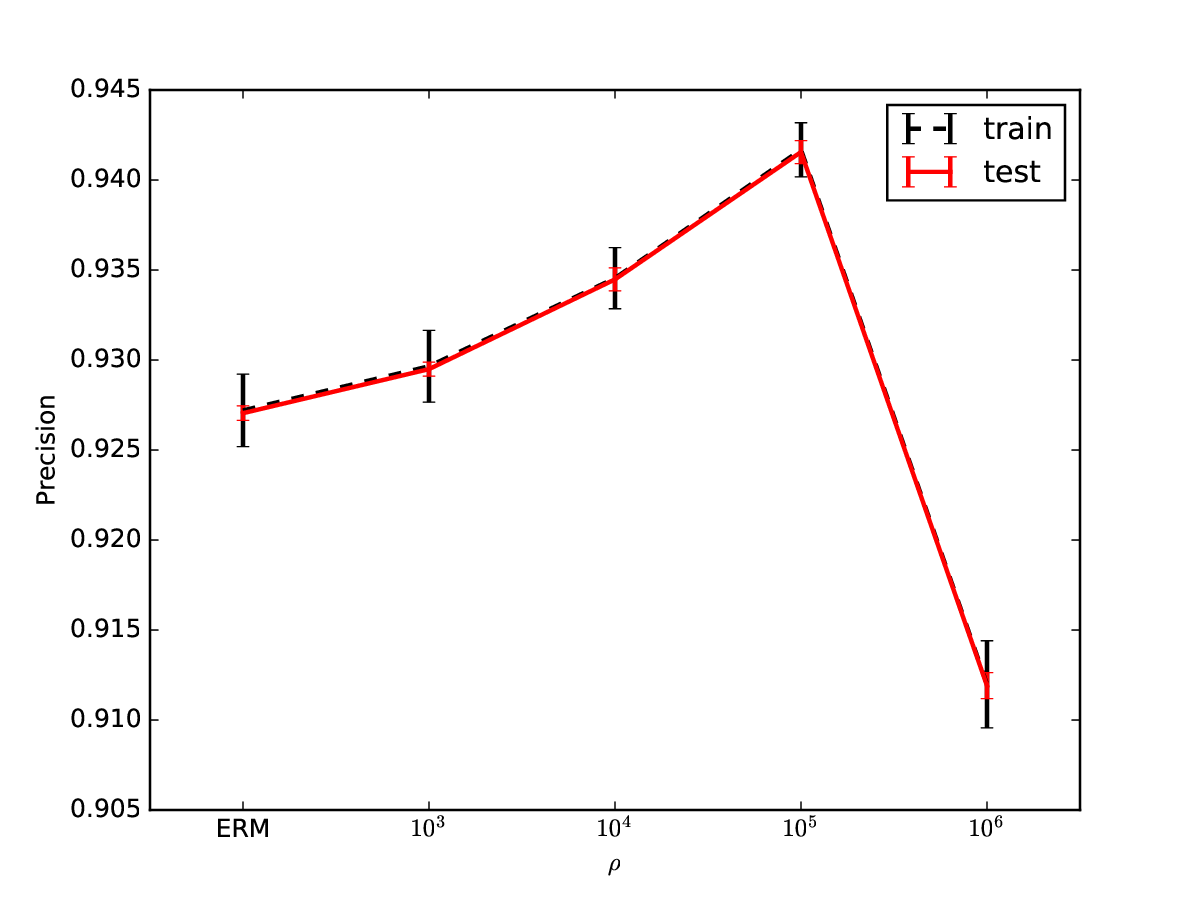} 
    \caption{Precision} 
    \label{fig:reuters-precision} 
    \vspace{0ex}
  \end{subfigure}
  
  \begin{subfigure}[3]{0.47\linewidth}
    \centering
    \includegraphics[width=1\linewidth]{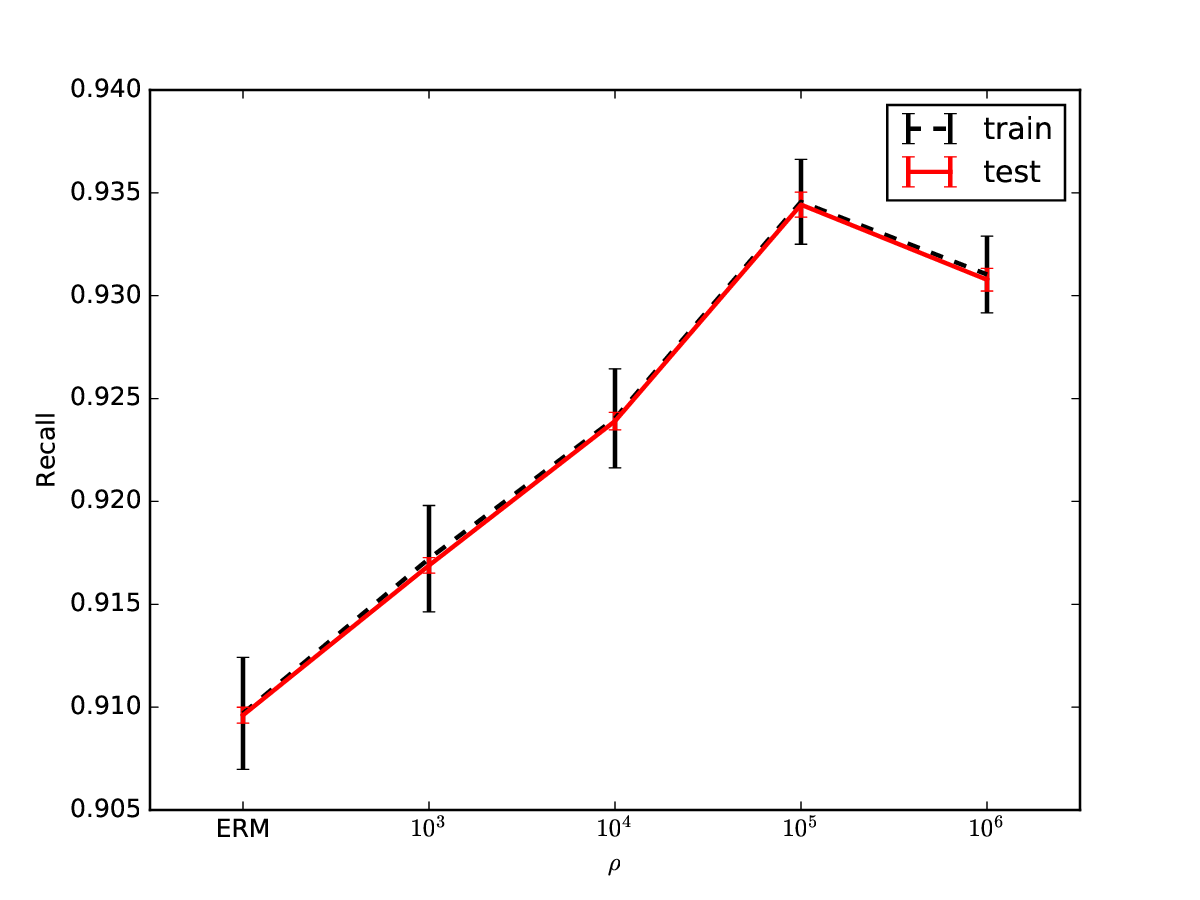} 
    \caption{Recall} 
    \label{fig:reuters-recall} 
  \end{subfigure}
  \begin{subfigure}[4]{0.47\linewidth}
    \centering
    \includegraphics[width=1\linewidth]{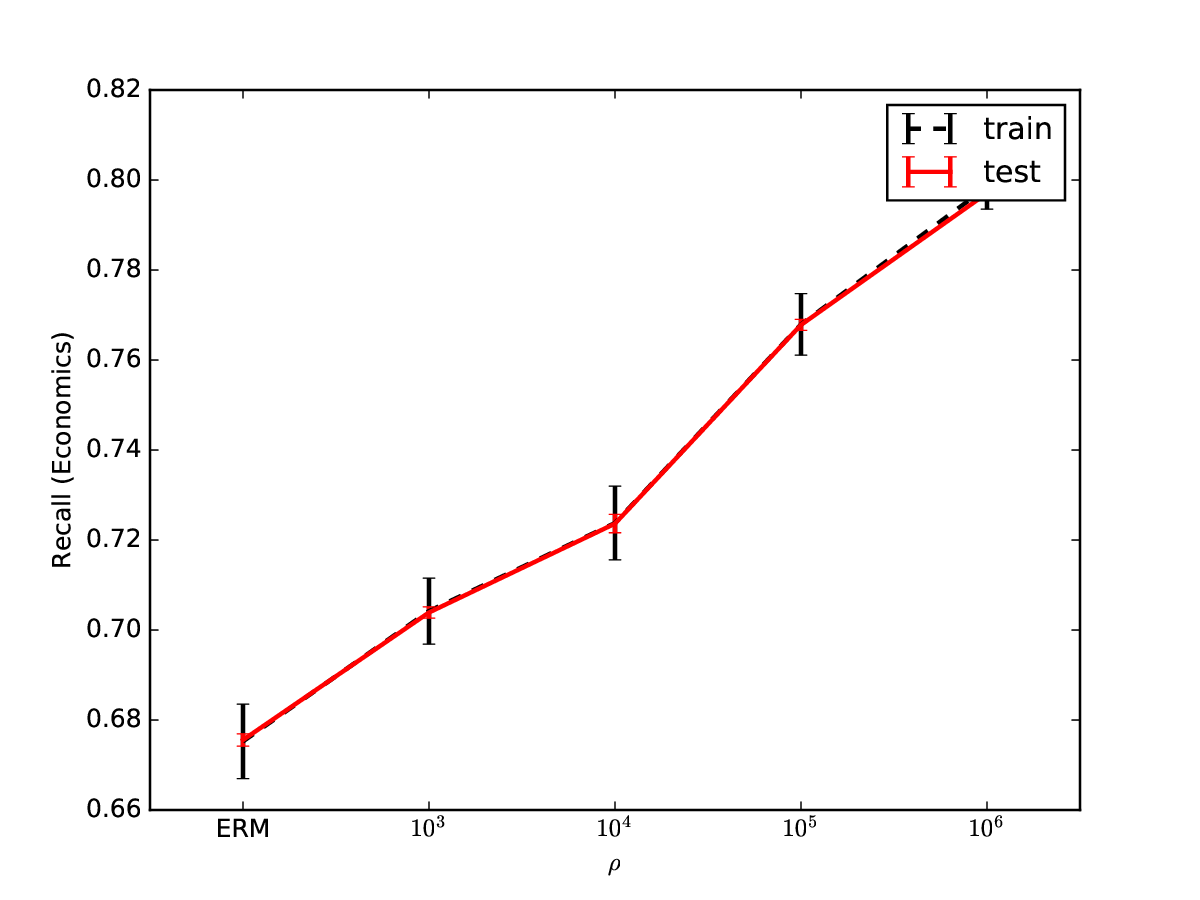} 
    \caption{Recall on rare category (Economics)} 
    \label{fig:reuters-recall-econ} 
  \end{subfigure}  
  \caption{Reuters Corpus ($2$-standard error deviations)}
  \label{fig:reuters} 
\end{figure}


In Figure~\ref{fig:reuters}, we summarize the results of our experiment
averaged over the $10$ runs, with $2$-standard error bars (computed across the
folds). To facilitate comparison across the document categories, we give exact
values of these averages in Tables~\ref{table:reuters-precision}
and~\ref{table:reuters-recall}. Both $\robsol$ and $\ermsol$ have reasonably
high precision across all categories, with increasing $\tol$ giving a mild
improvement in precision (from $.93 \pm .005$ to $.94 \pm .005$); see also
Figure~\ref{fig:reuters}(a).  On the other hand, we observe in
Figure~\ref{fig:reuters}(d) that ERM has low recall ($.69$ on test) for the
Economics category, which contains about 15\% of documents. As we increase
$\tol$ from $0$ (ERM) to $10^5$, we see a smooth and substantial improvement
in recall for this rarer category (without significant degradation in
precision). This improvement in recall amounts to reducing variance in
predictions on the rare class.  We also note that while the robust solutions
outperform ERM in classification performance for $\tol \le 10^5$, for very
large $\tol = 10^6 \ge 10n$, the regularizing effects of robustness
degrade the solution $\robsol$.  This precision and recall improvement comes
in spite of the increase in the average binary logistic risk for each of the
$4$ classes, which we show in Figure~\ref{fig:reuters-risk}, which plots the
average binary logistic loss (on train and test sets) averaged over the $4$
categories as well as the upper confidence bound $\risk_n(\theta, \mc{P}_n)$
as we vary $\tol$. The robust regularization effects reducing variance
appear to improve the performance of the binary logistic loss as a surrogate
for true misclassification error.
\begin{table}[ht]
  \centering
  \caption{Reuters Corpus Precision ($\%$)}
  \label{table:reuters-precision}
  \begin{filecontents*}{./Experiments/Reuters/reuters_precision.csv}
label,train,test,c_train,c_test,e_train,e_test,g_train,g_test,m_train,m_test
erm,92.72,92.7,93.55,93.55,89.02,89,94.1,94.12,92.88,92.94
1E3,92.97,92.95,93.31,93.33,87.84,87.81,93.73,93.76,92.56,92.62
1E4,93.45,93.45,93.58,93.61,87.6,87.58,93.77,93.8,92.71,92.75
1E5,94.17,94.16,94.18,94.19,86.55,86.56,94.07,94.09,93.16,93.24
1E6,91.2,91.19,92,92.02,74.81,74.8,91.19,91.25,89.98,90.18
\end{filecontents*}
\pgfplotstabletypeset[
  col sep=comma,
  string type,
  every head row/.style={%
    before row={\hline
      & \multicolumn{2}{c}{Precision}
      & \multicolumn{2}{c}{Corporate}
      & \multicolumn{2}{c}{Economics}
      & \multicolumn{2}{c}{Government}
      & \multicolumn{2}{c}{Markets}
      \\
    },
    after row=\hline
  },
  every last row/.style={after row=\hline},
  columns/label/.style={column name=$\rho$, column type=l},
  columns/train/.style={column name=train, column type=l},
  columns/test/.style={column name=test, column type=c},
  columns/c_train/.style={column name=train, column type=l},
  columns/c_test/.style={column name=test, column type=c},
  columns/e_train/.style={column name=train, column type=l},
  columns/e_test/.style={column name=test, column type=c},
  columns/g_train/.style={column name=train, column type=l},
  columns/g_test/.style={column name=test, column type=c},
  columns/m_train/.style={column name=train, column type=l},
  columns/m_test/.style={column name=test, column type=c}
  ]{./Experiments/Reuters/reuters_precision.csv}
\end{table}

\begin{table}[ht]
  \centering
  \caption{Reuters Corpus Recall ($\%$)}
  \label{table:reuters-recall}
  \begin{filecontents*}{./Experiments/Reuters/reuters_recall.csv}
label,train,test,c_train,c_test,e_train,e_test,g_train,g_test,m_train,m_test
erm,90.97,90.96,90.20,90.25,67.53,67.56,90.49,90.49,88.77,88.78
1E3,91.72,91.69,90.83,90.86,70.42,70.39,91.26,91.23,89.62,89.58
1E4,92.40,92.39,91.47,91.54,72.38,72.36,91.76,91.76,90.48,90.45
1E5,93.46,93.44,92.65,92.71,76.79,76.78,92.26,92.21,91.46,91.47
1E6,93.10,93.08,92.00,92.04,79.84,79.71,91.89,91.90,92.00,91.97
\end{filecontents*}
\pgfplotstabletypeset[
  col sep=comma,
  string type,
  every head row/.style={%
    before row={\hline
      & \multicolumn{2}{c}{Recall}
      & \multicolumn{2}{c}{Corporate}
      & \multicolumn{2}{c}{Economics}
      & \multicolumn{2}{c}{Government}
      & \multicolumn{2}{c}{Markets}
      \\
    },
    after row=\hline
  },
  every last row/.style={after row=\hline},
  columns/label/.style={column name=$\rho$, column type=l},
  columns/train/.style={column name=train, column type=l},
  columns/test/.style={column name=test, column type=c},
  columns/c_train/.style={column name=train, column type=l},
  columns/c_test/.style={column name=test, column type=c},
  columns/e_train/.style={column name=train, column type=l},
  columns/e_test/.style={column name=test, column type=c},
  columns/g_train/.style={column name=train, column type=l},
  columns/g_test/.style={column name=test, column type=c},
  columns/m_train/.style={column name=train, column type=l},
  columns/m_test/.style={column name=test, column type=c}
  ]{./Experiments/Reuters/reuters_recall.csv}
\end{table}

\subsubsection*{Summary}

We have seen through multiple examples that robustification---our convex
surrogate for variance regularization---is an effective tool in a number of
applications. As we heuristically expect, variance-based regularization
(robust regularization) yields predictors with better performance on ``hard''
instances, or subsets of the problem that induce higher variance, such as
classes with relatively few training examples in classification problems. The
robust regularization $\tol$ gives a principled knob for tuning performance to
trade between variance (uniform or across-the-board performance)
and---sometimes---absolute performance.



\section{Discussion}


In this paper, we have developed theoretical results for robust
regularization~\eqref{eqn:solve-robust} that apply to general stochastic
optimization and learning problems problems. The examples we describe in
Section~\ref{sec:optimal-selection} illustrate our expectation that the
robust solution $\robsol$ should have good performance in cases in which
$\var(\loss(\theta\opt; \statrv))$ is small (recall also
Theorems~\ref{theorem:selection-by-robustness}
and~\ref{theorem:selection-by-robustness-localized}).  Identifying the
separation between the performance empirical risk minimization and related
estimators and that of the robustly-regularized estimators---as well as
variance-regularized estimates---we consider more generally remains a
challenge. We hope that this paper inspires work in this direction in
machine learning and statistics, and more broadly, torward considering
distributionally robust problems. Part of this is likely to come from making
rigorous our empirical observations (Section~\ref{sec:experiments})
that robust regularization improves performance on ``hard'' instances
without sacrificing performance on easier cases.

Our understanding of so-called ``fast rates'' for stochastic optimization
problems, while considering robustness, is also limited.  For empirical risk
minimization, fast rates of convergence hold under conditions in which the
the gap $\risk(\theta) - \risk(\theta\opt)$ controls the variance of the
excess loss $\loss(\theta, \statrv) - \loss(\theta\opt,
\statrv)$~\cite[cf.][]{MammenTs99,BartlettBoMe05,BoucheronBoLu05,
  BartlettJoMc06}, which usually requires some type of uniform convexity
assumption. These bounds typically follow from localization
guarantees~\cite[Section 5]{BartlettBoMe05} on the function class
\begin{equation*}
  \left\{ \statval \mapsto
  \loss(\theta, \statval) - \loss(\theta\opt, \statval) \mid
  \theta \in \Theta
  \right\}.
\end{equation*}
While in Section~\ref{section:fast-rates-main}, we show that the robust
estimate $\robsol$ enjoys faster rates of convergence under growth
conditions analogous to uniform convexity of the risk, as
$\var(\loss(\theta; \statrv) - \loss(\theta\opt; \statrv)) \neq
\var(\loss(\theta; \statrv))$, it is not clear how to directly connect these
guarantees to results of the form in
Theorems~\ref{theorem:selection-by-robustness}
and~\ref{theorem:selection-by-robustness-localized}. We leave investigation
of these topics to future work.


The last point of our discussion is to revisit
Theorem~\ref{theorem:selection-by-robustness-localized}, which
provides a guarantee for robustly regularized estimators based
on localized Rademacher complexities. An investigation of our
proof shows that our derivation proceeds by considering
the complexity of self-normalized classes of functions
of the form
\begin{equation*}
  \mc{G}_r = \left\{ \sqrt{\frac{r}{\E[f^2] \vee r}}
  f \mid f \in \fclass \right\}.
\end{equation*}
In contrast, the analogous result of~\citet[Thereom 3.3]{BartlettBoMe05}
for empirical risk minimization
considers the complexity of classes of functions of the form
\begin{equation*}
  \mc{G}_r = \left\{ \frac{r}{\E[f^2] \vee r} f
  \mid f \in \fclass \right\}.
\end{equation*}
The latter class normalizes functions $f$ by $\sqrt{\E[f^2]}$---a type of
self-normalization that arises in the computation of pivotal (asymptotically
independent of the underlying distribution) statistics.  While this choice
\emph{prima facie} is just a step in our proof, the robust objective
$\risk_n(\theta, \mc{P}_n)$ defined in Eq.~\eqref{eqn:robust-risk} is an
empirical likelihood upper confidence bound on the optimal population
risk~\cite[see also][]{DuchiGlNa16}. One of the important characteristics of
empirical likelihood confidence bounds is that they are self-normalizing and
yield pivotal statistics~\cite{Owen01}. Investigating such
self-normalization in complexity guarantees seems likely to yield fruitful
insights.

\ifdefined\usejmlrstyle

\acks{
  We thank Feng Ruan for pointing out a much simpler proof of
  Theorem~\ref{theorem:variance-expansion} than in our original
  paper. JCD and HN were partially supported
  by the SAIL-Toyota Center for AI Research and HN was partially supported
  Samsung Fellowship. JCD was also partially supported by the National Science
  Foundation award NSF-CAREER-1553086 and the Sloan Fellowship.
}

\else

\subsection*{Acknowledgments}
We thank Feng Ruan for pointing out a much simpler proof of
Theorem~\ref{theorem:variance-expansion} than in our original
paper. JCD and HN were partially supported
by the SAIL-Toyota Center for AI Research and HN was partially supported
Samsung Fellowship. JCD was also partially supported by the National Science
Foundation award NSF-CAREER-1553086 and the Sloan Fellowship.
\fi

\newpage
\appendix



\section{Proof of Theorem~\ref{theorem:variance-expansion}}
\label{section:proof-of-variance-expansion}

The theorem is immediate if $s_n = 0$ or $\sigma^2 = 0$, as in this case
$\sup_{P : \phidivs{P}{\emp} \le \tol/n} \E_P[Z] = \E_\emp[Z] = \E[Z]$.
In what follows, we will thus assume that
$\sigma^2, s_n^2 > 0$.
We recall the maximization problem~\eqref{eqn:simple-problem},
which is
\begin{equation*}
  \maximize_p~ \sum_{i = 1}^n p_i z_i
  ~~ \subjectto ~ p \in \mc{P}_n = \left\{p \in \R^n_+ :
  \half \ltwo{n p - \ones}^2 \le \tol, \<\ones, p\> = 1 \right\},
\end{equation*}
and the solution criterion~\eqref{eqn:when-choosing-u-is-possible}, which
guarantees that the maximizing value of problem~\eqref{eqn:simple-problem}
is $\wb{z} + \sqrt{2 \tol s_n^2 / n}$ whenever
\begin{equation*}
  \sqrt{2 \tol} \frac{z_i - \wb{z}}{\sqrt{n s_n^2}} \ge -1.
\end{equation*}
Letting $z = Z$, then under the conditions of the theorem, we have $|z_i -
\wb{z}| \le \zbound$, and to satisfy
inequality~\eqref{eqn:when-choosing-u-is-possible} it is certainly
sufficient that
\begin{equation}
  \label{eqn:bounded-z-sufficient-n}
  2 \tol \frac{\zbound^2}{n s_n^2} \le 1,
  ~~ \mbox{or} ~~
  n \ge \frac{2 \tol \zbound^2}{s_n^2},
  ~~ \mbox{or} ~~
  s_n^2 \ge \frac{2 \tol \zbound^2}{n}.
\end{equation}
Conversely, suppose that $s_n^2 < \frac{2 \tol \zbound^2}{n}$. Then
we have $\frac{2 \tol s_n^2}{n} < \frac{4 \tol^2 \zbound^2}{n^2}$, which
in turn implies that
\begin{equation*}
  \sup_{p \in \mc{P}_n} \<p, z\>
  \ge \frac{1}{n} \<\ones, z\>
  + \hinge{\sqrt{\frac{2 \tol s_n^2}{n}} - \frac{2 \zbound \tol}{n}}.
\end{equation*}
Combining this inequality with the
condition~\eqref{eqn:bounded-z-sufficient-n} for the exact expansion to hold
yields the two-sided variance bounds~\eqref{eqn:sure-variance-bounds}.

We now turn to showing the high-probability exact
expansion~\eqref{eqn:exact-variance-expansion}, which occurs whenever the
sample variance is large enough by
expression~\eqref{eqn:bounded-z-sufficient-n}.  To that end, we show that
$s_n^2$ is bounded from below with high probability. Define the event
\begin{equation*}
  \event_n \defeq \left\{ s_n^2 \ge \frac{3}{64} \sigma^2 \right\},
\end{equation*}
and let $n \ge \frac{4\zbound^2}{\sigma^2} \max \left\{ 2 \sigma, 11
\right\}$. Then, on event $\event_n$ we have
  $n \ge \frac{44 \tol \zbound^2}{\sigma^2}
  \ge \frac{2 \tol \zbound^2}{s_n^2}$,
so that the sufficient condition~\eqref{eqn:bounded-z-sufficient-n}
holds and
expression~\eqref{eqn:exact-variance-expansion} follows.  We now argue
that the event $\event_n$ has high probability via the following
lemma, which is an application of concentration
inequalities for convex functions coupled with careful estimates
of the expectation of standard deviations.
\begin{lemma}
  \label{lemma:empirical-stdev-close}
  Let $Z_i$ be i.i.d.\ random variables taking values in
  $[\zbound_0, \zbound_1]$ with $\zbound = \zbound_1 - \zbound_0$, and
  let
  $s_n^2 = \frac{1}{n} \sum_{i=1}^n Z_i^2 - \left( \frac{1}{n}
    \sum_{i=1}^n Z_i \right)^2$. Let
  $c_n = 1 + \frac{7}{4 n} + \frac{3}{n^2}$. For all $t \ge 0$, we have
  \begin{equation*}
    \P\left( s_n \ge \sqrt{\E s_n^2} + t\right)
    \vee
    \P\left(s_n \le \sqrt{\E s_n^2}
      - \frac{c_n  \zbound^2}{n} - t\right) 
    \le \exp\left(-\frac{n t^2}{2 \zbound^2}\right).
  \end{equation*}
\end{lemma}
\noindent
The proof of the lemma is involved, but the lemma yields a quick proof of the
theorem that we now provide. (See
Section~\ref{sec:proof-empirical-stdev-close} for a proof.)

Let $\sigma^2 = \var(Z)$ and note that $\E[s_n^2] = (1 - \frac{1}{n})
\sigma^2$.  Set $t = \frac{\sqrt{3}}{4} \sigma$ and note that $c_n <
\frac{3}{2}$ if $n \ge 5$. Then, since $n \ge \frac{8 \zbound^2}{\sigma}$ by
hypothesis, we have
that
\begin{align*}
  \sigma \sqrt{1-n^{-1}} - \frac{c_n \zbound^2}{n} - t 
  & \ge
  \frac{2 \sqrt{5}}{5} \sigma
  - \frac{3 \zbound^2}{2 n} - t
  > \frac{7}{10} \sigma - t
  > \frac{\sqrt{3}}{8} \sigma.
\end{align*}
Lemma~\ref{lemma:empirical-stdev-close} implies that
$s_n \ge \sigma \sqrt{1 - n^{-1}} - \frac{c_n \zbound^2}{n} - t$
with probability at least $1 - e^{-n t^2 / 2 \zbound^2}
= 1 - e^{-3 n \sigma^2 / 32 \zbound^2}$, which gives the theorem.

\subsection{Proof of Lemma~\ref{lemma:empirical-stdev-close}}
\label{sec:proof-empirical-stdev-close}

We use three technical lemmas in the proof of this lemma.
\begin{lemma}[Samson~\cite{Samson00}, Corollary 3]
  \label{lemma:convex-concentration}
  Let $f : \R^n \to \R$ be convex and $L$-Lipschitz continuous with
  respect to the $\ell_2$-norm over $[a, b]^n$, and let
  $Z_1, \ldots, Z_n$ be independent random variables on $[a, b]$. Then
  for all $t \ge 0$,
  \begin{equation*}
    \P(f(Z_{1:n}) \ge \E[f(Z_{1:n})] + t) \vee
    \P(f(Z_{1:n}) \le \E[f(Z_{1:n})] - t)
    \le \exp\left(-\frac{t^2}{2L^2 (b - a)^2}\right).
  \end{equation*}
\end{lemma}

The function $\R^n \ni z \mapsto \ltwo{(I - (1/n) \ones\ones^\top) z}$ is
$1$-Lipschitz with respect to the Euclidean norm, so
Lemma~\ref{lemma:convex-concentration} implies
\begin{equation*}
  \P\left(\sqrt{\var_\emp(Z)} \ge \E[\sqrt{\var_\emp(Z)}] + t\right)
  \vee
  \P\left(\sqrt{\var_\emp(Z)} \le \E[\sqrt{\var_\emp(Z)}] - t\right)
  \le \exp\left(-\frac{n t^2}{2 M^2}\right).
\end{equation*}
As $\E[\var_\emp(Z)^\half] \le \E[\var_\emp(Z)]^\half = \sqrt{(1 - 1/n)
  \var(Z)}$, this yields the first part of the first inequality of the
lemma. We must, however, also lower bound $\E[\var_\emp(Z)^\half]$.
\begin{lemma}
  \label{lemma:sqrt-moments}
  Let $Y_i$ be random variables with finite $4$th moment and assume that
  $\cov(Y_i^2, Y_j^2) \le \sigma^4$ for all pairs $i, j$. Then we have the
  following inequalities:
  \begin{subequations}
    \begin{align}
      \label{eqn:root-n-sqrt-lower}
      \E\left[\bigg(\frac{1}{n} \sum_{i=1}^n Y_i^2\bigg)^\half\right]
      & \ge \left( \frac{1}{n} \sum_{i=1}^n \E[Y_i^2]\right)^\half
        - \frac{1}{\sqrt{n}} \sqrt{\frac{\frac{1}{n} \sum_{i=1}^n
        \var(Y_i^2) + n \sigma^4}{\frac{1}{n} \sum_{i=1}^n\E[Y_i^2]}} \\
      \E\left[\bigg(\frac{1}{n} \sum_{i=1}^n Y_i^2\bigg)^\half\right]
      & \ge \left( \frac{1}{n} \sum_{i=1}^n \E[Y_i^2]\right)^\half
      - \frac{1}{n} \frac{\frac{1}{n} \sum_{i=1}^n
        \var(Y_i^2) + n \sigma^4}{\frac{1}{n} \sum_{i=1}^n\E[Y_i^2]}.
      \label{eqn:one-n-var-lower}
    \end{align}
  \end{subequations}
\end{lemma}
\begin{lemma}
  \label{lemma:sqrt-moments-redux}
  Let $Z_1, \ldots, Z_n$ be i.i.d\ variables with
  finite fourth moment. Let
  $Y_i = Z_i - \frac{1}{n} \sum_{j = 1}^n Z_j$. Then
  \begin{equation*}
    \E\left[\bigg(\frac{1}{n} \sum_{i=1}^n Y_i^2\bigg)^\half\right]
    \ge \left( \frac{1}{n} \sum_{i=1}^n \E[Y_i^2]\right)^\half
    - \frac{1}{n} \left[\frac{\E[(Z - \E[Z])^4]}{\var(Z)}
      + \frac{7 + 12 / n}{n} \var(Z) \right]
  \end{equation*}
  If $\max_j Z_j - \min_j Z_j \le C$ with probability 1, then
  for the constant $c_n = 1 + \frac{7}{4 n} + \frac{3}{n^2}$, we have
  \begin{equation*}
    \E\left[\bigg(\frac{1}{n} \sum_{i=1}^n Y_i^2\bigg)^\half\right]
    \ge \left( \frac{1}{n} \sum_{i=1}^n \E[Y_i^2]\right)^\half
    - c_n \frac{C^2}{n}.
  \end{equation*}
\end{lemma}
\noindent
We provide the proof of Lemmas~\ref{lemma:sqrt-moments}
and~\ref{lemma:sqrt-moments-redux} in Sections~\ref{sec:proof-sqrt-moments}
and~\ref{sec:proof-sqrt-moments-redux}, respectively.

Based on these lemmas, Lemma~\ref{lemma:empirical-stdev-close}
is immediate once we set
$Y_i = Z_i - \frac{1}{n} \sum_{j=1}^n Z_j$, so that
$s_n^2 = \frac{1}{n} \sum_{i=1}^n Y_i^2$, and apply
Lemma~\ref{lemma:sqrt-moments-redux} with $C = \zbound$.


\subsection{Proof of Lemma~\ref{lemma:sqrt-moments}}
\label{sec:proof-sqrt-moments}

We first prove the claim~\eqref{eqn:root-n-sqrt-lower}.
To see this, we use that
\begin{equation*}
  \inf_{\lambda \ge 0}
  \left\{\frac{a^2}{2 \lambda} + \frac{\lambda}{2}
  \right\} = \sqrt{a^2} = |a|,
\end{equation*}
and taking derivatives yields that for all $\lambda' \ge 0$,
\begin{equation*}
  \frac{a^2}{2 \lambda} + \frac{\lambda}{2}
  \ge \frac{a^2}{2 \lambda'} + \frac{\lambda'}{2}
  - \left(\frac{a^2}{2 {\lambda'}^2} - \half\right)(\lambda - \lambda').
\end{equation*}
By setting
$\lambda_n = \sqrt{\frac{1}{n} \sum_{i=1}^n Y_i^2}$, we thus have
for any $\lambda \ge 0$ that
\begin{align*}
  \E\left[\bigg(\frac{1}{n} \sum_{i=1}^n Y_i^2\bigg)^\half \right]
  & = \E\left[\frac{\sum_{i=1}^n Y_i^2}{2 n \lambda_n}
    + \frac{\lambda_n}{2} \right] \\
  & \ge \E\left[\frac{\sum_{i=1}^n Y_i^2}{2 n \lambda}
    + \frac{\lambda}{2} \right]
    + \E\left[\left(\half -\frac{\sum_{i=1}^n Y_i^2}{2 n \lambda^2}\right)
    (\lambda_n - \lambda)\right].
\end{align*}
Now we take $\lambda = \sqrt{\frac{1}{n} \sum_{i=1}^n \E[Y_i^2]}$, and
we apply the Cauchy-Schwarz inequality to obtain
\begin{align}
  \label{eqn:annoyance-to-deal-with}
  \lefteqn{\E\left[\bigg(\frac{1}{n} \sum_{i=1}^n Y_i^2\bigg)^\half \right]} \\
  & \ge \left( \frac{1}{n} \sum_{i=1}^n \E[Y_i^2]\right)^\half 
  - \frac{1}{2\lambda^2}
  \E\left[\left( \frac{1}{n} \sum_{i=1}^n (Y_i^2 - \E[Y_i^2])
    \right)^2\right]^\half
  \E\left[\left(\bigg(\frac{1}{n} \sum_{i=1}^n Y_i^2\bigg)^\half
    - \left(\frac{1}{n} \sum_{i=1}^n \E[Y_i^2]\right)^\half
    \right)^2\right]^\half.
  \nonumber
\end{align}
We control each of these quantities in turn.

First, our assumption that
$\cov(Y_i^2, Y_j^2) \le \sigma^4$ implies that
\begin{align*}
  \lefteqn{\E\left[\left( \frac{1}{n} \sum_{i=1}^n (Y_i^2 - \E[Y_i^2])
      \right)^2\right]} \\
  & ~~~ = \frac{1}{n^2} \sum_{i = 1}^n \var(Y_i^2)
  + \frac{1}{n^2} \sum_{i \neq j}
  \E\left[(Y_i^2 - \E[Y_i^2])(Y_j^2 - \E[Y_j^2])\right]
  \le \frac{1}{n}
  \left(\frac{1}{n} \sum_{i = 1}^n \var(Y_i^2) + n \sigma^4\right).
\end{align*}
The triangle inequality 
implies that $\E[((n^{-1} \sum_{i = 1}^n Y_i^2)^\half - 
(n^{-1} \sum_{i = 1}^n \E[Y_i^2])^\half)^2]^\half
\le 2 \lambda = 2 \sqrt{\frac{1}{n} \sum_{i = 1}^n \E[Y_i^2]}$,
so that substituting in inequality~\eqref{eqn:annoyance-to-deal-with}
we have
\begin{align*}
  \E\left[\bigg(\frac{1}{n} \sum_{i=1}^n Y_i^2\bigg)^\half \right]
  \ge \left( \frac{1}{n} \sum_{i=1}^n \E[Y_i^2]\right)^\half 
  - \frac{1}{\lambda\sqrt{n}}
  \left(\frac{1}{n} \sum_{i = 1}^n \var(Y_i^2) + n\sigma^4\right)^\half.
\end{align*}
This is the bound~\eqref{eqn:root-n-sqrt-lower}.

Now we give the sharper result. We have
\begin{align*}
  \E\left[\left(\bigg(\frac{1}{n} \sum_{i=1}^n Y_i^2\bigg)^\half
    - \left(\frac{1}{n} \sum_{i=1}^n \E[Y_i^2]\right)^\half
  \right)^2\right]
  & = \frac{2}{n} \sum_{i=1}^n \E[Y_i^2]
    - 2 \left(\frac{1}{n} \sum_{i=1}^n \E[Y_i^2]\right)^\half
    \E\bigg[\bigg(\frac{1}{n} \sum_{i=1}^n Y_i^2
    \bigg)^\half\bigg]
  \\
  & \le \frac{2}{\sqrt{n}}
    \left(\frac{1}{n} \sum_{i=1}^n \var(Y_i^2)
    + n \sigma^4\right)^\half,
\end{align*}
where we have used the bound~\eqref{eqn:root-n-sqrt-lower}. Returning to
inequality~\eqref{eqn:annoyance-to-deal-with}, we obtain the second
claim~\eqref{eqn:one-n-var-lower} of the lemma.

\subsection{Proof of Lemma~\ref{lemma:sqrt-moments-redux}}
\label{sec:proof-sqrt-moments-redux}

As $Y_i = Z_i - \frac{1}{n} \sum_{j = 1}^n Z_j$, it is no loss of generality
to assume that $\E[Z] = 0$, as $Y_i$ is shift-invariant. The lemma follows
immediately from Lemma~\ref{lemma:sqrt-moments} once we prove the claims that
\begin{subequations}
  \begin{equation}
    \label{eqn:bound-variance-Y}
    \var(Y_1^2) \le \frac{n - 1}{n} \E[Z^4] + \frac{6}{n} \var(Z)^2
  \end{equation}
  and
  \begin{equation}
    \E[(Y_1^2 - \E[Y_1^2])(Y_1^2 - \E[Y_1^2])]
    \le \frac{1 + 12/n}{n^2} \var(Z)^2.
    \label{eqn:correlation-variance-estimates}
  \end{equation}
\end{subequations}

Note that $\E[Y_i^2] = \E[(Z_i - \frac{1}{n} \sum_{j = 1}^n Z_j)^2]
= \frac{n - 1}{n} \var(Z)$.
We begin with the first inequality~\eqref{eqn:bound-variance-Y}.
We have
\begin{equation*}
  \var(Y_1^2)
  = \E[Y_1^4] - \E[Y_1^2]^2
  = \E[(Z_1 - \wb{Z}_n)^4] - \frac{(n - 1)^2}{n^2} \var(Z)^2,
\end{equation*}
where we use the shorthand $\wb{Z}_n = \frac{1}{n} \sum_{j = 1}^n Z_j$.
Expanding the first quantity, the fourth moment of $Y_1$, we have
\begin{equation*}
  \E[(Z_1 - \wb{Z}_n)^4]
  = \E[Z^4] - 4 \E[Z_1^3 \wb{Z}_n]
  + 6 \E[Z_1^2 \wb{Z}_n^2]
  - 4 \E[Z_1 \wb{Z}_n^3]
  + \E[\wb{Z}_n^4].
\end{equation*}
Using that $\E[Z_1 \wb{Z}_n^3] = \E[Z^4] / n^3
+ \frac{3(n - 1)}{n^3} \var(Z)^2$,
$\E[Z_1^2 \wb{Z}_n^2] = \frac{1}{n^2} \E[Z^4] + \frac{n - 1}{n^2}
\var(Z)^2$, and
$\E[\wb{Z}_n^4]
\le \frac{1}{n^3} \E[Z^4] + \frac{3}{n^2} \var(Z)^2$, we obtain
\begin{align*}
  \E[(Z_1 - \wb{Z}_n)^4]
  & \le \left(1 -\frac{4}{n} + \frac{6}{n^2}
  - \frac{3}{n^3} \right)\E[Z^4]
  + \left(6 \frac{n - 1}{n^2}
  - 4 \frac{3n - 3}{n^3}
  + \frac{3}{n^2} \right)\var(Z)^2 \\
  & \le \frac{n - 1}{n} \E[Z^4] + \frac{6}{n} \var(Z)^2,
\end{align*}
which is the result~\eqref{eqn:bound-variance-Y}.

For the second claim~\eqref{eqn:correlation-variance-estimates},
we expand
\begin{equation*}
  Y_i^2 - \E[Y_i^2]
  = Z_i^2 + \wb{Z}_n^2
  - 2 Z_i \wb{Z}_n - \frac{n - 1}{n} \var(Z).
\end{equation*}
The right-hand-side of Eq.~\eqref{eqn:correlation-variance-estimates} is
thus
\begin{align}
  & \E\left[\left(Z_1^2 - \var(Z)
      + n^{-1}\var(Z)
      + \wb{Z}_n^2 - 2 Z_1 \wb{Z}_n\right)
      \left(Z_2^2 - \var(Z)
      + n^{-1}\var(Z)
      + \wb{Z}_n^2
      - 2 Z_2 \wb{Z}_n\right)\right] \nonumber\\
  & = 2 \E\left[(Z_1^2 - \var(Z))
    \left(n^{-1} \var(Z) + \wb{Z}_n^2 - 2 Z_2 \wb{Z}_n\right)\right]
  \nonumber \\
  & \qquad ~ 
  + \E\left[\left(n^{-1} \var(Z) + \wb{Z}_n^2
    - 2 Z_1 \wb{Z}_n\right)
    \left(n^{-1} \var(Z) + \wb{Z}_n^2
    - 2 Z_2 \wb{Z}_n\right)\right].
  \label{eqn:weird-Z-correlations}
\end{align}
The first term above satisfies
\begin{align*}
  \E\left[(Z_1^2 - \var(Z))
    \left(n^{-1} \var(Z) + \wb{Z}_n^2 - 2 Z_2 \wb{Z}_n\right)\right]
  & = \E\left[(Z_1^2 - \var(Z)) \wb{Z}_n^2 \right]
  = \frac{\var(Z^2)}{n^2}.
\end{align*}
The latter term~\eqref{eqn:weird-Z-correlations} is equal to
\begin{align*}
  \lefteqn{\frac{1}{n^2} \var(Z)^2
    + \frac{2}{n} \var(Z) \E[\wb{Z}_n^2]
    - \frac{4}{n} \var(Z) \E[Z_1 \wb{Z}_n]
    + \E[\wb{Z}_n^4]
    - 4 \E[Z_1 \wb{Z}_n^3]
    + 4 \E[Z_1 Z_2 \wb{Z}_n^2]} \\
  & = -\frac{1}{n^2}
  \var(Z)^2 + \E[\wb{Z}_n^4]
  - 4 \E[Z_1 \wb{Z}_n^3]
  + 4 \E[Z_1 Z_2 \wb{Z}_n^2]
  = \var(\wb{Z}_n^2)
  + 4 \left(\E[Z_1 Z_2 \wb{Z}_n^2] - \E[Z_1 \wb{Z}_n^3]\right).
\end{align*}
As earlier, we have that $\E[Z_1 \wb{Z}_n^3] = \E[Z^4] / n^3
+ \frac{3(n - 1)}{n^3} \var(Z)^2$,
$\E[Z_1 Z_2 \wb{Z}_n^2]
= \frac{2}{n^2} \E[Z_1^2 Z_2^2]
= \frac{2}{n^2} \var(Z)^2$,
and $\E[\wb{Z}_n^4]
\le \frac{1}{n^3} \E[Z^4] + \frac{3}{n^2} \var(Z)^2$,
so that we can bound the last term~\eqref{eqn:weird-Z-correlations}
by
\begin{equation*}
  -\frac{3}{n^3} \E[Z^4] + \frac{3}{n^2} \var(Z)^2
  + 4 \frac{3 - n}{n^3} \var(Z)^2
  \le \frac{12 - n}{n^3} \var(Z)^2.
\end{equation*}
This yields the claim~\eqref{eqn:correlation-variance-estimates}.

The final inequality, when $Z_i$ are bounded, follows immediately upon
noticing that $\E[(Z - \E[Z])^4] \le C^2 \var(Z)$ when $Z$ takes
values in an interval of width at most $C$, and that moreover,
in this circumstance, $\var(Z) \le \frac{C^2}{4}$.

\section{Proof of Theorem~\ref{theorem:uniform-variance-expansion}}
\label{section:proof-of-uniform-variance-expansion}

Our starting point is to recall from
inequality~\eqref{eqn:bounded-z-sufficient-n} in the proof of
Theorem~\ref{theorem:variance-expansion} that for each $f \in \fclass$, the
empirical variance equality~\eqref{eqn:uniform-exact-variance-expansion} holds
if $n \ge \frac{4\tol \zbound^2}{\var_{\emp}(f)}$. As a consequence,
Theorem~\ref{theorem:uniform-variance-expansion} will follow if we can provide
a uniform lower bound on the sample variances $\var_{\emp}(f)$ that holds with
high enough probability. We use $C$ to denote a universal constant
whose value may change from line to line.
Noting that
$\var_{\emp}(f) = \E_{\emp}(f-\E[f])^2 - (\E_{\emp}(f- \E[f]))^2$,
we proceed in two parts. First, we give a lower bound for
$\E_{\emp}(f-\E[f])^2$.
\begin{lemma}
  \label{lemma:second-moment}
  Let $\fclass$ be a collection of bounded functions
  $f: \statdomain \to [\zbound_0, \zbound_1]$ with
  $\zbound \defeq \zbound_1 - \zbound_0$. Then, with probability at least
  $1 - e^{-t}$,
  for every $f \in \fclass$
  \begin{equation*}
    \var(f) \le 2 \E_{\emp}(f-\E[f])^2
    + C \left[ \worstrad(\mc{F})^2 \log^3(n\zbound)
      + \frac{\zbound^2}{n} \left( t + \log \log n \right) \right].
  \end{equation*}
\end{lemma}
\begin{proof}
  We follow
  the arguments of \citet{SrebroSrTe10} and
  \citet[Thm.~6.1]{Bousquet02thesis}.  For $x_1, \ldots, x_n \in \mc{X}$,
  let
  \begin{equation*}
    \mc{F}_{n,r}
    \defeq \left\{f - \E[f] \in \mc{F} \mid
    \E_{\emp}[(f- \E[f])^2] \le r \right\},
  \end{equation*}
  where $\emp$ is the empirical measure on $x_1, \ldots, x_n$.  Let
  $\psi^{\sup}_n$ be a sub-root upper bound on the worst-case Rademacher
  complexity
  \begin{equation*}
    \psi^{\sup}_n(r) \ge \worstrad(\mc{F}_{n,r}),
  \end{equation*}
  where implicitly in the right hand side we take the supremum over $x_1,
  \ldots, x_n$ definining $\mc{F}_{n,r}$ as well.
  The function $t \mapsto t^2$ has $2$-Lipschitz derivatives,
  so we may apply \citet[Lemma~2.2]{SrebroSrTe10} to obtain
  \begin{equation}
    \label{eqn:bound-worstrad}
    \worstrad(\mc{F}^2_{n,r}) \le
    C \sqrt{r} \worstrad(\mc{F}) \log^\frac{3}{2} n
  \end{equation}
  where we recall the notation that $\mc{G}^2 = \{g^2 \mid g \in \mc{G}\}$ for
  any function class $\mc{G}$. Thus we may take
  $\psi^{\sup}_n(r) = C \sqrt{r} \worstrad(\mc{F}) \log^\frac{3}{2} n$, which
  has fixed point $r^{\sup}_n = C^2 \worstrad(\mc{F})^2 \log^3 n$.  Since
  $f^2 \ge 0$, Theorem 6.1 of~\citet{Bousquet02thesis} yields that for all
  $f \in \mc{F}$,
  \begin{equation*}
    \E(f-\E[f])^2 \le 2 \E_{\emp}(f - \E[f])^2
    + C \left[\worstrad(\mc{F})^2 \log^3 n\zbound
      + \frac{\zbound^2}{n} \left( t+ \log \log n \right) \right]
  \end{equation*}
  with probability at least $1-e^{-t}$.
\end{proof}

Next, we give an upper bound for $(\E_{\emp}(f- \E[f]))^2$.  We use
the following version of Talagrand's inequality due to~\citet{Bousquet02,
  Bousquet03}.
(See also \citet[Thm 2.1]{BartlettBoMe05}.)
\begin{lemma}
  \label{lemma:talagrand}
  Let $r > 0$ and $\fclass$ be a class of functions that map $\statdomain$
  into $[a, b]$ such that for every $f \in \fclass$, $\var(f(X)) \le r$. Then,
  with probability at least $1 - e^{-t}$
  \begin{align*}
    \sup_{f \in \fclass} \{\E[f] - \E_{\emp}[f]\} \le \inf_{\alpha > 0}
    \left\{ 2 (1+\alpha) \E[\rad]
    + \sqrt{\frac{2rt}{n}} + \frac{t}{n}
    (b-a)\left( \frac{1}{3} + \frac{1}{\alpha} \right)
    \right\}.
  \end{align*}
  The same statement holds with $\sup_{f \in \fclass} (\E_{\emp}[f] - \E[f])$
  replacing the left-hand side of the inequalities.
\end{lemma}

Applying Lemma~\ref{lemma:talagrand} and letting $\alpha = \half$, with
probability at least $1-2e^{-t}$
\begin{equation*}
  |\E_{\emp}[f] - \E[f]|
  \le 3 \E[\rad] + 2\zbound\sqrt{\frac{2t}{n}}
\end{equation*}
holds for all $f \in \mc{F}$. Combining the above display with
Lemma~\ref{lemma:second-moment}, we obtain the desired result.

\section{Proof of Theorem~\ref{theorem:selection-by-robustness}}
\label{sec:proof-selection-by-robustness}

Before proving the theorem proper, we state a technical lemma
that provides uniform Bernstein-like bounds for the class $\fclass$
using empirical $\ell_\infty$-covering numbers.
\begin{lemma}[Maurer and Pontil~\cite{MaurerPo09}, Theorem 6]
  \label{lemma:vc-uniform-variance-bounds}
  Let $n \ge \frac{8\zbound^2}{t}$ and $t \ge \log 12$. Then with
  probability at least
  $1-6\covnum_\infty(\fclass, \epsilon, 2n) e^{-t}$, we have
  \begin{equation}
    \label{eqn:empirical-bernstein-simplified}
    \E[f] \le \E_{\emp}[f] + 3 \sqrt{\frac{2\var_{\emp} (f)t}{n}}
    + \frac{15\zbound t}{n}
    + 2\left(1 + 2 \sqrt{\frac{2t}{n}}\right)\epsilon
  \end{equation}
  for all $f \in \fclass$.
\end{lemma}

We return to the proof of
Theorem~\ref{theorem:selection-by-robustness}. Let $\event_1$ denote
that the event that the
inequalities~\eqref{eqn:empirical-bernstein-simplified} hold.
Then on $\event_1$ hold, uniformly over $f \in \fclass$ we have
\begin{align}
  \E[f]
  & \le \E_\emp[f]
    + \sqrt{\frac{18 \var_\emp(f(\statrv)) t}{n}}
    + \frac{15 \zbound t}{n}
    + 2\left(1 + 2\sqrt{\frac{2t}{n}}\right) \epsilon \nonumber \\
  & \stackrel{(i)}{\le} \sup_{P : \phidivs{P}{\emp} \le \frac{\tol}{n}}
    \E_P[f(\statrv)]
    + \sqrt{\frac{2 \tol \var_\emp(f(\statrv))}{n}} \nonumber \\
  & \qquad ~
    - \hinge{\sqrt{\frac{2 \tol \var_\emp(f(\statrv))}{n}}
    - \frac{2 \zbound \tol}{n}}
    + \frac{5 \zbound \tol}{3n}
    + 2\left(1 + 2\sqrt{\frac{2t}{n}}\right) \epsilon \nonumber \\
  & \le \sup_{P : \phidivs{P}{\emp} \le \frac{\tol}{n}}
    \E_P[f(\statrv)] + \frac{11}{3} \frac{ \zbound \tol}{n}
    + 2\left(1 + 2\sqrt{\frac{2t}{n}}\right) \epsilon
    ~\mbox{for~all~} f \in \fclass,
    \label{eqn:get-to-sup}
\end{align}
where inequality~$(i)$ follows from the
bounds~\eqref{eqn:sure-variance-bounds} in
Theorem~\ref{theorem:variance-expansion} and the fact that
$\tol \ge 9t$ by assumption.
This gives the first result~\eqref{eqn:coverage-F}.

For the second result~\eqref{eqn:fast-convergence-F}, we recall that
$\what{f} \in \argmin_{f \in \fclass} \sup_{P} \{ \E_P[f(\statrv)] :
\phidivs{P}{\emp} \le \frac{\tol}{n}\}$, and we bound the supremum term in
expression~\eqref{eqn:get-to-sup}. 
First, we note that because $\what{f}$ minimizes the supremum term
in expression~\eqref{eqn:get-to-sup}, we have
\begin{equation*}
  \E[\what{f}]
  \le \sup_{P : \phidivs{P}{\emp} \le \frac{\tol}{n}}
  \E_P[f(\statrv)] + \frac{ 11\zbound \tol}{3n}
    + 2\left(1 + 2\sqrt{\frac{2t}{n}}\right) \epsilon
    ~\mbox{for~all~} f \in \fclass.
\end{equation*}
Now fix $f \in \fclass$. As the function $f$ is fixed, by Bernstein's
inequality, we have
\begin{align*}
  \E_\emp[f]
  & \le \E[f] + \sqrt{\frac{2 \var(f) t}{n}}
    + \frac{2 Mt}{3n}
\end{align*}
with probability at least $1 - e^{-t}$. Similarly, we have by
Lemma~\ref{lemma:empirical-stdev-close} that
\begin{equation*}
  \sqrt{\var_\emp(f)}
  \le \sqrt{1 - n^{-1}} \sqrt{\var(f)} + \sqrt{\frac{2t \zbound^2}{n}}
\end{equation*}
with probability at least $1 - e^{-t}$. That is,
for any fixed $f \in \fclass$,
we have with probability at least $1 - 2 e^{-t}$ that
\begin{align*}
  \sup_{P : \phidivs{P}{\emp} \le \frac{\tol}{n}} \E_P[f(\statrv)]
  & \stackrel{(i)}{\le} \E_\emp[f] + \sqrt{\frac{2 \tol \var_\emp(f)}{n}} \\
  & \le \E[f] + \sqrt{\frac{2 \var(f) t}{n}} + \frac{2 \zbound}{3n} t
  + \sqrt{\frac{2 \tol \var(f)}{n}}
  + \frac{2 \sqrt{\zbound^2 \tol t}}{n} \\
  & \stackrel{(ii)}{\le}
    \E[f] + 2 \sqrt{\frac{2 \var(f) \tol}{n}}
    + \frac{8}{3}\frac{\zbound \tol}{n},
\end{align*}
where inequality~$(i)$ follows from the uniform upper
bound~\eqref{eqn:sure-variance-bounds} of
Theorem~\ref{theorem:variance-expansion} and inequality~$(ii)$ from
our assumption that $\tol \ge t$. Substituting this expression into
our earlier bound~\eqref{eqn:get-to-sup} yields that for any
$f \in \fclass$, with probability at least
\begin{equation*}
  1 - 2 (3\covnum_{\infty}\left(\fclass, \epsilon, 2n\right)+1) e^{-t},
\end{equation*}
we have
\begin{equation*}
  \E[\what{f}(\statrv)]
  \le \E[f(\statrv)]
  + 2 \sqrt{\frac{ 2 \tol \var(f(\statrv))}{n}}
  + \frac{19}{3}
  \frac{\zbound \tol}{n}
  + 2\left(1 + 2\sqrt{\frac{2t}{n}}\right) \epsilon.
\end{equation*}
This gives the theorem.



\section{Proof of Theorem~\ref{theorem:selection-by-robustness-localized}}
\label{section:proof-of-selection-by-robustness-localized}

We first show the following version of uniform Bernstein's inequality with
Rademacher complexities. The proof uses a peeling
technique~\cite{BartlettBoMe05,vandeGeer00}, in conjuction with Talagrand's
concentration inequality (Lemma~\ref{lemma:talagrand}).
\begin{lemma}
  \label{lemma:talagrand-bernstein}
  Let $r > 0$ and $\fclass$ be a collection of bounded functions
  $f: \statdomain \to [0, \zbound]$ with $\var(f(X)) \le r$. Then, with
  probability at least $1 - e^{-t}$, for every $f \in \fclass$
  \begin{equation*}
    \E[f] \le \E_{\emp}[f]
    + \sqrt{\frac{2e\var(f)}{n}
      \left( t + \log\ceil{\log \frac{nr}{\zbound^2 t}} \right) }
    + 6 \E[\rad]
    + \frac{7 \zbound}{n}
    \left( t + \log\ceil{\log \frac{nr}{\zbound^2 t }} \right).
  \end{equation*}
  The same statements hold with the roles of $\E[f]$ and $\E_{\emp}[f]$
  reversed.
\end{lemma}
\noindent We defer the proof to
section~\ref{section:proof-of-talagrand-bernstein} at the end of this
section. Because $\var(f) \le \zbound^2$ for all $f \in \fclass$,
Lemma~\ref{lemma:talagrand-bernstein} also holds if we replace
the terms  $\ceil{\log \frac{nr}{\zbound^2 t}}$ with 
$\ceil{\log \frac{n}{t}} \le 1 + \log \frac{n}{t}$.

Next, we show an important extension of
Lemma~\ref{lemma:talagrand-bernstein} that replaces the Rademacher
complexity term $\E[\rad]$ by a local quantity $r_n\opt$, the fixed point of
$\psi_n(r)$. To this end, we use another peeling argument and apply
Lemma~\ref{lemma:talagrand-bernstein} to the self-normalized class
\begin{equation*}
  \mc{G}_r
  \defeq \left\{ \sqrt{\frac{r}{\E[f^2] \vee r}}f : f \in \fclass \right\}
  \subseteq \left\{ cf: f \in \fclass, \E[c^2f^2] \le r, c \in [0, 1]\right\}.
\end{equation*}
This idea follows the techniques of \citet[Thm.~3.3]{BartlettBoMe05},
though we use a type of self-normalizing scale, that is, $f / \sqrt{\E[f^2]}$,
whereas they use a variance-normalizing scaling by studying classes
of functions of the form $f / \E[f^2]$. Our use of this 
alternative normalization is important in the next lemma, which
allows us to obtain bounds that apply to the robustly regularized risk.
\begin{lemma}
  \label{lemma:uniform-bernstein-local}
  Let $\fclass$ be a collection of bounded functions $f: \statdomain \to [0,
    \zbound]$ satisfying the localization inequality~\eqref{eqn:sub-root}
  for some sub-root function $\psi_n(\cdot)$ with root $r_n\opt$. Let
  $B_n = \frac{1}{n} \xprime$. Then with
  probability at least $1-e^{-t}$, for every $f \in \fclass$
  \begin{align*}
    \E[f] \le \E_{\emp}[f] 
    + \left(\sqrt{2e B_n} + 6 \sqrt{r_n\opt + 7 \zbound B_n / 3}\right)
    \sqrt{\E[f^2]}
    + 6 r_n\opt + 14 \zbound B_n.
  \end{align*}
  The same
  statement holds with the roles of $\E[f]$ and $\E_{\emp}[f]$ reversed.
\end{lemma}
\noindent See Section~\ref{section:proof-of-uniform-bernstein-local} for the
proof.


Next, we give an analogous result for $f^2$.
\begin{lemma}
  \label{lemma:variance-local}
  Let $\fclass$ be a collection of bounded functions $f: \statdomain \to [0,
    \zbound]$ satisfying the localization inequality~\eqref{eqn:sub-root}
  for some sub-root function $\psi_n(\cdot)$ with root $r_n\opt$.
  Let $\eta > 0$. Then, with probability at least $1
  -e^{-t}$, for every $f \in \fclass$
  \begin{equation*}
    \E[f^2] \le \E_{\emp}[f^2]
    + \frac{1}{\eta} \E_\emp[f^2]
    + 72 \zbound^2 (1 + \eta) r_n\opt
    + \frac{\zbound t}{n} \left( 4 + \frac{7}{3}\zbound\right).
  \end{equation*}
  Also, with probability at least $1 -e^{-t}$, for every $f \in \fclass$
  \begin{equation*}
    \E_{\emp}[f^2] \le \E[f^2]
    + \frac{\eta}{1 + \eta} \E[f^2]
    + 72 \zbound^2 (1 + \eta) r_n\opt
    + \frac{\zbound t}{n} \left(4 + \frac{7}{3} \zbound\right).
  \end{equation*}
\end{lemma}
\noindent See Section~\ref{section:proof-of-variance-local} for the proof.

Now, we make two additional pieces of shorthand notation.  Let
\begin{equation*}
  V_n
  = 4 ((2e + 84 \zbound)B_n + 36 r_n\opt).
\end{equation*}
Then, Lemma~\ref{lemma:uniform-bernstein-local} implies that
\begin{equation*}
  \E[f] \le \E_\emp[f] + \sqrt{V_n \E[f^2]} + 6 r_n\opt + 14 \zbound B_n
\end{equation*}
with probability at least $1 - e^{-t}$.
Applying Lemma~\ref{lemma:variance-local} to this
bound with the choice $\eta = 1$ immediately yields that
\begin{align*}
  \E[f] & \le \E_{\emp}[f] + \sqrt{
    2 V_n \E_\emp[f^2]
    + 144 \zbound^2 V_n r_n\opt
    + 7 V_n \zbound \max\{\zbound, 1\} t / n}
  + 6 r_n\opt + 14 \zbound B_n \\
  & \le \E_\emp[f]
  + \sqrt{2 V_n \E_\emp[f^2]}
  + 12 \zbound \sqrt{V_n\left(r_n\opt
  + \frac{7 \max\{\zbound, 1\}}{\zbound} \frac{t}{n}\right)}
  + 6 r_n\opt + 14 \zbound B_n
\end{align*}
for all $f \in \mc{F}$ with probability at least $1-2e^{-t}$.
Subtracting and adding $(\E_{\emp}[f])^2$ to the second term, we have
\begin{align*}
  \sqrt{2 V_n \E_\emp[f^2]}
  & = \sqrt{2 V_n \var_\emp(f) + 2 V_n \E_\emp[f]^2}
  \le \sqrt{2 V_n \var_\emp(f)} + \sqrt{2 V_n} \E_\emp[f],
\end{align*}
where we have used that $f \ge 0$. We thus obtain
\begin{align*}
  \E[f]
  & \le \left(1 + \sqrt{2 V_n}\right)\E_\emp[f]
  + \sqrt{2 V_n \var_\emp(f)}
  + 12 \zbound \sqrt{V_n \left(r_n\opt
    + \frac{7 \max\{\zbound,1\}}{\zbound} \frac{t}{n}\right)}
  + 6 r_n\opt + 14 \zbound B_n \\
  & \le
  \left(1 + \sqrt{2 V_n}\right)\E_\emp[f]
  + \sqrt{2 V_n \var_\emp(f)}
  + 6 \zbound V_n
  + 6 \zbound\left(r_n\opt + \frac{7 \max\{\zbound, 1\} t}{\zbound n}\right)
  + 6 r_n\opt + 14 \zbound B_n,
\end{align*}
where the second inequality follows because $\sqrt{ab} \le \half a + \half
b$ for $a, b \ge 0$.  Recalling the bound~\eqref{eqn:rho-bound}, which
implies $\tol \ge n V_n$,
$\tol \ge n (r_n\opt + \frac{7 \max\{\zbound, 1\} t}{\zbound n})$,
and $\tol/n \ge 6 r_n\opt + 14 \zbound B_n$, we obtain
that
\begin{equation*}
  \E[f]
  \le \left(1 + \sqrt{\frac{2 \tol}{n}}\right)
  \E_\emp[f] + \sqrt{\frac{2 \tol}{n} \var_\emp(f)}
  + \frac{13 \zbound \tol}{n}.
\end{equation*}
Theorem~\ref{theorem:variance-expansion} implies
$\E_{\emp}[f] + \sqrt{\frac{2\tol}{n} \var_{\emp}(f)} \le
\sup_{P : \phidivs{P}{\emp} \le \frac{\tol}{n}} \E_P[f(\statrv)]
  + \frac{2\zbound \tol}{n}$, so we immediately
we arrive at
\begin{align*}
  \E[f] \le   \left(1 + 2 \sqrt{\frac{2\tol}{n}} \right)
  \sup_{P : \phidivs{P}{\emp} \le \frac{\tol}{n}} \E_P[f(\statrv)]
  + \left(13 + 4 \sqrt{\frac{2\tol}{n}} \right) \frac{\zbound \tol}{n}
\end{align*}
for all $f \in \mc{F}$ with probability at least $1-2e^{-t}$.
This is the first result~\eqref{eqn:coverage-F-local}.

To show the second result, we simply apply Bernstein's inequality and the
concentration inequalities for the standard deviation in
Lemma~\ref{lemma:empirical-stdev-close}.  For any fixed $f \in \fclass$, by
Bernstein's inequality, we have
\begin{align*}
  \E_\emp[f]
  & \le \E[f] + \sqrt{\frac{2t \var(f) }{n}}
    + \frac{2 Mt}{3n}
\end{align*}
with probability at least $1 - e^{-t}$.  From
Lemma~\ref{lemma:empirical-stdev-close}, we have
\begin{equation*}
  \sqrt{\var_\emp(f)}
  \le \sqrt{1 - n^{-1}} \sqrt{\var(f)} + \sqrt{\frac{2t \zbound^2}{n}}
\end{equation*}
with probability at least $1 - e^{-t}$.

We thus obtain that
for any fixed $f$,
\begin{equation*}
  \sup_{P : \phidivs{P}{\emp} \le \frac{\tol}{n}}
  \E_P[f]
  \le
  \E_\emp[f]
  + \sqrt{\frac{2 \tol}{n} \var_\emp(f)}
  \le
  \E[f]
  + \sqrt{\frac{2t}{n} \var(f)}
  + \sqrt{\frac{2 \tol}{n} \var(f)}
  + \frac{2\zbound \sqrt{\tol t}}{n}
  + \frac{2 \zbound t}{3n}
\end{equation*}
with probability at least $1 - 2 e^{-t}$.  Noting that
$\tol \ge 45 \zbound t$
by assumption~\eqref{eqn:rho-bound}, so 
$\sqrt{\tol} + \sqrt{t}
\le \sqrt{46 \tol / 45 + 45 t} \le \sqrt{91 \tol / 45}$
and that always $2\sqrt{\tol t} \le 3 \tol + \frac{1}{3} t$,
we have that with probability at least $1 - 2e^{-t}$ that
\begin{equation*}
  \sup_{P : \phidivs{P}{\emp} \le \frac{\tol}{n}}
  \E_P[f]
  \le
  \E[f] + \sqrt{\frac{91 \tol}{45 n} \var(f)}
  + \frac{3\zbound \tol}{n} + \frac{\zbound t}{n}.
\end{equation*}
Noting that we could take $f$ to minimize the right hand side of the
preceding expression and that $\what{f}$ minimizes $\sup_{P :
  \phidivs{P}{\emp} \le \tol / n} \E_P[f]$, we have the
result~\eqref{eqn:second-fast-convergence-F-local}.

\subsection{Proof of Lemma~\ref{lemma:talagrand-bernstein}}
\label{section:proof-of-talagrand-bernstein}

We first show the claim for $g \in \fclass_{\rm centered} = \{ f - \E[f]:
f\in \fclass\}$. To see the claim for $g \in \fclass_{\rm centered}$, let us
fix $L \in \N$ to be chosen later, and for $l = 1,
\ldots, L - 1$ define the classes
\begin{equation*}
  \fclass_l \defeq
  \left\{ g \in \fclass_{\rm centered}: e^{-l}r < \E[g^2]
  \le e^{-(l-1)}r \right\},
  ~~ \fclass_L \defeq
  \left\{ g \in \fclass_{\rm centered}:  \E[g^2] \le e^{-L}r \right\}
\end{equation*}
so that $\fclass_{\rm centered} = \cup_{l=1}^L \fclass_l$. Let $z > 0$ be such
that $t \le z$.  Applying Lemma~\ref{lemma:talagrand} (with the
choice $\alpha = \half$) to
$\mc{F}_l$ for each $l = 1, \ldots, L-1$, we have with probability at least
$1-e^{-t}$, for every $g \in \mc{F}_l$
\begin{align*}
  \E[g] & \le \E_{\emp}[g] + \sqrt{\frac{2te^{-(l-1)}r}{n}}
  + 3 \E[\radcomp_n(\fclass_l)] + 5 \zbound \frac{t}{n} \\
  & \le \E_{\emp}[g] + \sqrt{\frac{2et}{n} \E[g^2]}
  + 3 \E[\radcomp_n(\fclass_l)] + 5 \zbound \frac{t}{n}
\end{align*}
where in the last line we have used $e^{-l}r \le \E[g^2]$ for
$g \in \fclass_l$. Similarly, applying
Lemma~\ref{lemma:talagrand} to $\mc{F}_L$,
then with probability at least $1 - e^{-t}$, for every
$g \in \mc{F}_L$
\begin{align*}
  \E[g] & \le \E_{\emp}[g] + \sqrt{\frac{2t e^{-L}r}{n}}
  + 3 \E[\radcomp_n(\fclass_L)] + 5 \zbound \frac{t}{n} \\
  & \le \E_{\emp}[g] + \sqrt{\frac{2et}{n} \E[g^2]}
  + \sqrt{\frac{2 t e^{-L} r}{n}}
  + 3 \E[\radcomp_n(\fclass_L)] + 5 \zbound \frac{t}{n}.
\end{align*}
Taking a union bound, we have with probability at least $1 - L e^{-t}$, for
every $g \in \fclass_{\rm centered}$
\begin{equation*}
  \E[g] \le \E_{\emp}[g] + \sqrt{\frac{2e t}{n} \E[g^2]}
  + 3 \E[\radcomp_n(\fclass_{\rm centered})] + 5 \zbound \frac{t}{n}
  + \sqrt{\frac{2t e^{-L} r}{n}}.
\end{equation*}
Noting that $\E[\radcomp_n(\fclass_{\rm centered})] \le
2 \E[\radcomp_n(\fclass)]$ by Jensen's inequality,
we take $L = \ceil{\log \frac{r n}{\zbound^2 t}}$
and map $t$ to $t + \log L$ to obtain the lemma.
The case when the roles
of $\E[f]$ and $\E_{\emp}[f]$ are reversed follows similarly.

\subsection{Proof of Lemma~\ref{lemma:uniform-bernstein-local}}
\label{section:proof-of-uniform-bernstein-local}

Let $r \ge r_n\opt$ be an arbitrary but fixed value to be choosen later. Using
this $r$, define the self-normalized class of functions
\begin{equation*}
  \mc{G}_r
  \defeq \left\{ \sqrt{\frac{r}{\E[f^2] \vee r}}f : f \in \fclass \right\}
  \subseteq \left\{ cf: f \in \fclass, \E[c^2f^2] \le r,
    c \in [0, 1]\right\}.
\end{equation*}
From the truncation by $r$, we have $\E[g^2] \le r$ for all $g \in
\mc{G}_{r}$. Lemma~\ref{lemma:talagrand-bernstein}
implies that with probability at least $1-e^{-t}$, uniformly over
$g \in \mc{G}_{r}$
\begin{align}
  \E[g] & \le \E_{\emp}[g] + \sqrt{\frac{2e}{n} \E[g^2] \xprime}
  + 6 \E[\radcomp_n(\mc{G}_{r})]
  + \frac{7 \zbound}{n} \xprime.  \label{eqn:talagrand-peel}
\end{align}
Using the sub-root property of $\psi_n$ and that $\psi_n(r_n\opt) =
r_n\opt$, we have the inequality
\begin{equation*}
  \psi_n(r) = \sqrt{r} \psi_n(r) / \sqrt{r}
  \le \sqrt{r} \psi_n(r_n\opt) / \sqrt{r_n\opt} = \sqrt{r r_n\opt}
\end{equation*}
for any $r \ge r_n\opt$, so
\begin{equation*}
  \E[\radcomp_n\mc{G}_{r}]
  \le \E[\radcomp_n \left\{ cf: f \in \fclass, \E[ c^2 f^2]  \le r,
    c \in [0, 1]\right\}]
  \le \psi_n (r)
  \le \sqrt{rr_n\opt}
\end{equation*}
Using this
upper bound in Eq.~\eqref{eqn:talagrand-peel}
and recalling the notation $B_n = \frac{1}{n} \xprime$, we get
\begin{align}
  \label{eqn:talagrand-peel-second}
  \E[g] \le \E_{\emp}[g] + \sqrt{2e B_n \E[g^2] } +
  6 \sqrt{r_n\opt r} + 7 \zbound B_n.
\end{align}

Now, we return to choose the value $r$ to optimize the
bound~\eqref{eqn:talagrand-peel-second}. let $r$ be the largest solution to
$6 \sqrt{r_n\opt r} + 7 \zbound B_n = 6 r$. The following elementary lemma
provides a bound on $r$.
\begin{lemma}
  \label{lemma:quad-formula}
  Let $x$ be the largest solution to $a x + b = \frac{x^2}{d}$ where
  $a, b, d > 0$. Then $a^2 d^2 \le x^2 \le a^2 d^2 + 2bd$.
\end{lemma}
\begin{proof}
  From the quadratic formula, we have
  $x = \frac{1}{2} \left( ad + \sqrt{a^2 d^2 + 4 b} \right)$ from which the
  lower bound follows. From convexity of $z \mapsto z^2$ and
  $\sqrt{z_1 + z_2} \le \sqrt{z_1} + \sqrt{z_2}$ for $z_1, z_2 > 0$, we obtain
  the upper bound.
\end{proof}
Lemma~\ref{lemma:quad-formula} immediately yields
\begin{equation*}
  r_n\opt \le r \le r_n\opt + \frac{7\zbound B_n}{3}.
\end{equation*}

For each $g \in \mc{G}_{r}$, there exists $f \in \fclass$ such that
$g = \sqrt{\frac{r}{\E[f^2] \vee r}}f$. If $\E[f^2] \le r$, we have $g =f $
and the bound~\eqref{eqn:talagrand-peel-second} yields
\begin{equation*}
  \E[f] \le \E_{\emp}[f] + \sqrt{2eB_n \E[f^2]}
  + 6 r_n\opt + 14 \zbound B_n.
\end{equation*}
If $\E[f^2] > r$, rescaling $g$ in the bound~\eqref{eqn:talagrand-peel-second}
and using the choice $6r = 6 \sqrt{r_n\opt r} + 7 \zbound B_n$
yields
\begin{align*}
  \E[f] & \le \E_{\emp}[f] + \sqrt{2eB_n \E[f^2]}
  + 6 \sqrt{r\E[f^2]} \\
  & \le \E_{\emp}[f] + \sqrt{2eB_n \E[f^2]}
  + 6\sqrt{(r_n\opt + 7 \zbound B_n / 3) \E[f^2]}
\end{align*}
instead. Combining the cases $\E[f^2] \lessgtr r$, we conclude that for all
$f \in \mc{F}$,
\begin{align*}
  \E[f] \le \E_{\emp}[f] 
  + \left(\sqrt{2e B_n} + 6 \sqrt{r_n\opt + 7 \zbound B_n / 3}\right)
  \sqrt{\E[f^2]}
  + 6 r_n\opt + 14 \zbound B_n
\end{align*}
with probability at least $1-e^{-t}$. Similarly, we can reverse the roles of
$\E[f]$ and $\E_{\emp}[f]$ to get the second result.

\subsection{Proof of Lemma~\ref{lemma:variance-local}}
\label{section:proof-of-variance-local}

We frequently use the Rademacher
contraction principle~\cite[Thm.~4.12]{LedouxTa91} in what follows.
\begin{lemma}
  \label{lemma:contraction}
  Let $\phi : \R \to \R$ be $L$-Lipschitz. Then, for every class $\mc{G}$
  \begin{equation*}
    \E_{\epsilon}[\radcomp_n(\phi \circ \mc{G})]
    \le L \E_{\epsilon}[\radcomp_n (\mc{G})]
  \end{equation*}
  where $\phi \circ \mc{G} = \{ \phi \circ f: f \in \mc{G} \}$.
\end{lemma}

As in Section~\ref{section:proof-of-uniform-bernstein-local}, define the
self-normalized functions in $\fclass$
\begin{equation*}
  \mc{G}_r
  \defeq \left\{ \sqrt{\frac{r}{\E[f^2] \vee r}}f : f \in \fclass \right\}
  \subseteq \left\{ cf: f \in \fclass, \E[c^2f^2] \le r, c \in [0, 1]\right\}
\end{equation*}
where $r \ge r_n\opt$ will be choosen later. Let $\mc{G}_r^2 = \{ g^2: g \in
\mc{G}_r\}$. From the truncation by $r$, we have that for all $g^2 \in
\mc{G}_r^2$, $\var(g^2) \le \E[g^4] \le \zbound^2 \E[g^2] \le \zbound^2
r$. Let $c_1 = 3$ and $c_2 = \frac{7}{3}$. Then by
Lemma~\ref{lemma:talagrand} applied to $\mc{G}_r^2$, with probability at
least $1 - e^{-t}$, for every $g \in \mc{G}_r$
\begin{align}
  \E[g^2]
  & \le \E_{\emp}[g^2]
    + c_1 \E[\radcomp_n(\mc{G}_r^2)] + \zbound \sqrt{\frac{2rt}{n}}
    + c_2 \frac{\zbound^2t}{n} \nonumber \\
  & \stackrel{(a)}{\le} \E_{\emp}[g^2] + 2c_1 \zbound \E[\radcomp_n(\mc{G}_r)]
    + \zbound \sqrt{\frac{2rt}{n}} + \frac{c_2\zbound^2t}{n} \nonumber \\
  & \stackrel{(b)}{\le} \E_{\emp}[g^2] + 2c_1 \zbound \sqrt{r r_n\opt}
    + \zbound \sqrt{\frac{2rt}{n}} + \frac{c_2\zbound^2t}{n}
    \label{eqn:talagrand-peel-var}
\end{align}
where in step $(a)$ we used the contraction principle
(Lemma~\ref{lemma:contraction}) and that
$x \mapsto x^2$ is $2\zbound$-Lipschitz on $[-\zbound, \zbound]$, and
in step $(b)$, we used that $\psi_n(r) \le \sqrt{r r_n\opt}$ as
in the proof of Lemma~\ref{lemma:uniform-bernstein-local} in
Section~\ref{section:proof-of-uniform-bernstein-local}.

Let $A = 2c_1 \zbound \sqrt{r_n\opt} + \zbound \sqrt{\frac{2t}{n}}$ and
$D = \frac{c_2\zbound^2t}{n}$. For any fixed $K > 1$, choose $r$ to
be the largest solution to $A\sqrt{r} + D = \frac{r}{K}$ so that the
bound~\eqref{eqn:talagrand-peel-var} becomes
\begin{equation*}
  \E[g^2] \le \E_{\emp}[g^2] + \frac{r}{D}.
\end{equation*}
From Lemma~\ref{lemma:quad-formula}, we have
\begin{equation*}
  K^2 A^2 \le r \le K^2 A^2 + 2 KD
\end{equation*}
and in particular, $r \ge K^2 A^2 \ge r_n\opt$. For each $g \in \mc{G}_r$,
there exists $f \in \fclass$ such that $g = \sqrt{\frac{r}{\E[f^2] \vee r}}f$.
If $\E[f^2] \le r$, rescaling the inequality~\eqref{eqn:talagrand-peel-var}
and using the upper bound on $r$, we obtain
\begin{equation*}
  \E[f^2] \le \E_{\emp}[f^2] + \frac{r}{K}
  \le \E_{\emp}[f^2] + KA^2 + 2D.
\end{equation*}
If $\E[f^2] > r$, rescaling instead yields
\begin{equation*}
  \E[f^2] \le \E_{\emp}[f^2] + \frac{\E[f^2]}{K}.
\end{equation*}
Combining the two cases, we obtain
\begin{equation*}
  \E[f^2] \le \frac{K}{K-1} \E_{\emp}[f^2] + K A^2 + 2D.
\end{equation*}
Noting that
$A \le 2 \left( 4 c_1^2 \zbound^2 r_n\opt + 2 \frac{\zbound^2t}{n} \right)$ by
convexity, we have the first result once we replace $K$ with
$\eta = K - 1 > 0$.
The second result similarly follows by
reversing the roles of $\E[f]$ and $\E_{\emp}[f]$ in the above argument.

\section{Proof of Theorem~\ref{theorem:fast-rates}}
\label{section:proof-of-fast-rates}

Recall our shorthand notation that $\pi(\theta) = \argmin_{\theta^* \in
  \solset} \{\ltwo{\theta - \theta^*}\}$ denotes the Euclidean projection of
$\theta$ onto $\solset$, which is a closed convex set.  Define also the
localized empirical deviation function
\begin{equation}
  \label{eqn:local-deviation}
  \Delta_n(\theta)
  \defeq \E\left[\loss(\theta; X)
    - \loss(\pi(\theta); X)\right]
  - \E_\emp\left[\loss(\theta; X)
    - \loss(\pi(\theta); X)\right].  
\end{equation}
We begin with the following
\begin{claim}
  \label{claim:local-gaps}
  If $\empsolset^\epsilon \not\subset \solset^{2\epsilon}$, then
  \begin{equation}
    \label{eqn:local-gaps}
    \sup_{\theta \in \solset^{2 \epsilon}}
    \left\{\Delta_n(\theta)
    + \sqrt{\frac{2\tol}{n}
      \var_{\emp}(\loss(\theta; X) - \loss(\pi(\theta); X))}
    \right\} \ge \epsilon.
  \end{equation}
\end{claim}
Deferring the proof of the claim,
let us prove the theorem. First, the growth condition~\eqref{eqn:growth}
shows that
\begin{equation*}
  \solset^{2\epsilon}
  \subset \left\{
  \theta \in \Theta :
  \ltwo{\theta - \pi(\theta)}
  \le \left(\frac{2 \epsilon}{\lambda}\right)^\frac{1}{\growthpow}
  \right\}
  = \left\{
  \theta \in \Theta :
  \dist(\theta, \solset)
  \le \left(\frac{2 \epsilon}{\lambda}\right)^\frac{1}{\growthpow}
  \right\}.
\end{equation*}
Therefore, we have for all $\theta \in \solset^{2 \epsilon}$ that
\begin{equation*}
  \var_\emp(\loss(\theta; X) - \loss(\pi(\theta); X))
  \le L^2 \dist(\theta, \solset)^2
  \le L^2 \left(\frac{2\epsilon}{\lambda} \right)^{\frac{2}{\growthpow}},
\end{equation*}
and so by the assumption~\eqref{eqn:fast-rate-rademacher-bound} that
$\epsilon \ge (\frac{8 L^2 \tol}{n})^\frac{\growthpow}{2(\growthpow - 1)}
(\frac{2}{\lambda})^\frac{1}{\growthpow - 1}$, we have
\begin{equation*}
  \sqrt{\frac{2\tol}{n}
    \var_\emp(\loss(\theta; X) - \loss(\pi(\theta); X))}
  \le L \sqrt{\frac{2 \tol}{n}}
  \left(\frac{2 \epsilon}{\lambda}\right)^\frac{1}{\growthpow}
  \le \frac{\epsilon}{2}.
\end{equation*}
In particular, if the event~\eqref{eqn:local-gaps} holds then
\begin{equation*}
  \sup_{\theta \in \solset^{2\epsilon}}
  \Delta_n(\theta) \ge \frac{\epsilon}{2},
\end{equation*}
and recalling the definition~\eqref{eqn:local-deviation}
of $\Delta_n$, it then follows that
\begin{align}
  \label{eqn:prob-local-gaps}
  \P\left( \empsolset^{\epsilon} \not\subset \solset^{2\epsilon} \right)
  \le \P\left(
  \sup_{\theta \in \solset^{2\epsilon}}
  \Delta_n(\theta)
  \ge \frac{\epsilon}{2}
  \right).
\end{align}

To bound the probability~\eqref{eqn:prob-local-gaps}, we use standard
bounded difference and symmetrization arguments~\cite[e.g.][Theorem
  6.5]{BoucheronLuMa13}. Letting $f(X_1, \ldots, X_n) \defeq \sup_{\theta
  \in \solset^{2\epsilon}} \Delta_n(\theta)$, the function $f$ satisfies
bounded differences:
\begin{align*}
  & \sup_{x, x' \in \mc{X}} | f(X_1, \cdots, X_{j-1}, x, X_{j+1}, \cdots, X_n)
  - f(X_1, \cdots, X_{j-1}, x', X_{j+1}, \cdots, X_n)| \\
  & \le \sup_{x, x' \in \mc{X}}
  \sup_{\theta  \in \solset^{2\epsilon}}
  \left| \frac{1}{n} (\loss(\theta; x) - \loss(\pi(\theta); x))
  - \frac{1}{n} (\loss(\theta; x') - \loss(\pi(\theta); x')) \right| \\
  & \le \frac{2L}{n} \sup_{\theta \in \solset^{2\epsilon}}
  \dist(\theta, \solset)
  \le \frac{2L}{n} \left(\frac{2\epsilon}{\lambda} \right)^{\frac{1}{\growthpow}}
\end{align*}
for $j=1,\ldots,n$. Using the standard symmetrization inequality
$\E[\sup_{\theta \in \solset^{2\epsilon}}
  \Delta_n(\theta)] \le 2\E[\radcomp_n(\solset^{2\epsilon})]$ and the bounded
differences inequality~\cite[Theorem 6.5]{BoucheronLuMa13}, we have
\begin{align*}
  \P\left(
  \sup_{\theta \in \solset^{2\epsilon}}
  \Delta_n(\theta) \ge 2\E[\radcomp_n(\solset^{2\epsilon})] + t \right)
  & \le \exp\left(-\frac{nt^2}{2L^2} \left(\frac{\lambda}{2\epsilon}
    \right)^{\frac{2}{\growthpow}} \right)
\end{align*}
for all $t \ge 0$.  Letting $u = \frac{nt^2}{2L^2}
\left(\frac{\lambda}{2\epsilon} \right)^{\frac{2}{\growthpow}}$ above and
recalling the assumption~\eqref{eqn:fast-rate-rademacher-bound} upper
bounding $\E[\radcomp_n(\solset^{2\epsilon})]$, we have $\P(\sup_{\theta \in
  \solset^{2\epsilon}} \Delta_n(\theta) \ge \frac{\epsilon}{2}) \le e^{-u}$.
The theorem follows from the
bound~\eqref{eqn:prob-local-gaps}.

\paragraph{Proof of Claim~\ref{claim:local-gaps}}
If $\empsolset^\epsilon \not\subset \solset^{2\epsilon}$, then
certainly it is the case that there is some $\theta \in \Theta \setminus
\solset^{2\epsilon}$ such that
\begin{equation*}
  \risk_n(\theta, \mc{P}_n) \le
  \inf_{\theta \in \Theta} \risk_n(\theta, \mc{P}_n) + \epsilon
  \le \risk_n(\pi(\theta), \mc{P}_n) + \epsilon.
\end{equation*}
Using the convexity of $\risk_n$, we have for all $t \in [0, 1]$
that
\begin{equation*}
  \risk_n(t \theta + (1 - t) \pi(\theta), \mc{P}_n)
  \le
  t \risk_n(\theta, \mc{P}_n) + (1 - t) \risk_n(\pi(\theta),
  \mc{P}_n)
  \le 
  \risk_n(\pi(\theta),
  \mc{P}_n) + t \epsilon.
\end{equation*}
For all $t \in [0, 1]$, we have by definition of orthogonal projection
(because the vector $\theta - \pi(\theta)$ belongs to the normal cone to
$\solset$ at $\pi(\theta)$; cf.~\cite[Prop.~III.5.3.3]{HiriartUrrutyLe93ab})
that $\pi(t \theta + (1 - t) \pi(\theta)) = \pi(\theta)$. Thus, choosing $t$
appropriately, there exists $\theta' \in \bd \solset^{2 \epsilon}$ with
$\theta' = t \theta + (1 - t) \pi(\theta)$, $\pi(\theta') = \pi(\theta)$,
and $\risk_n(\theta', \mc{P}_n) \le \risk_n(\pi(\theta'), \mc{P}_n) +
\epsilon$.

Adding and subtracting the risk $\risk(\theta)$ and
$\risk(\pi(\theta))$, we have that
for some $\theta \in \bd \solset^{2 \epsilon}$ that
\begin{equation*}
  \risk_n(\theta, \mc{P}_n)
  - \risk(\theta) +
  \risk(\pi(\theta))
  - \risk_n(\pi(\theta), \mc{P}_n)
  \le \risk(\pi(\theta)) - \risk(\theta)
  + \epsilon
  \le -\epsilon,
\end{equation*}
where we have used that $\risk(\theta) = \risk(\pi(\theta))
+ 2\epsilon$ by construction.
Multiplying by $-1$ on each side of the preceding
display and taking suprema, we find that
\begin{align*}
  \epsilon & \le
             \sup_{\theta \in \solset^{2\epsilon}}
             \left\{\risk(\theta) - \risk_n(\theta, \mc{P}_n)
             - (\risk(\pi(\theta)) - \risk_n(\pi(\theta),
             \mc{P}_n))\right\} \\
           & \le \sup_{\theta \in \solset^{2\epsilon}}
             \sup_{P : \phidivs{P}{\emp}
             \le \tol/n}
             \left\{\risk(\theta) -
             \risk(\pi)
             + \E_P\left[
             \loss(\pi(\theta); X)
             - \loss(\theta; X)\right]\right\}.
\end{align*}
Applying the upper bound in inequality~\eqref{eqn:sure-variance-bounds} of
Theorem~\ref{theorem:variance-expansion} gives the claim.


\section{Proof of Theorem~\ref{theorem:asymptotic-convergence}}
\label{sec:proof-asymptotic-convergence}

We begin by establishing a few technical lemmas, after which the proof of the
theorem follows essentially standard arguments in asymptotics.
To prove Theorem~\ref{theorem:asymptotic-convergence}, we first
show that (eventually) we have the exact expansion
\begin{equation*}
  \risk_n(\theta, \mc{P}_n) = \E_\emp[\loss(\theta, \statrv)]
  + \sqrt{\frac{2 \tol \var_\emp(\loss(\theta, \statrv))}{n}}
\end{equation*}
for all $\theta$ in a neighborhood of $\theta\opt$. As in the proof of
Theorem~\ref{theorem:variance-expansion}, this exact equality holds once there
is suitable variability in the values $\loss(\theta, \statrv_i)$ over
$i = 1, \ldots, n$, however, we require a bit more care as the values
$\loss(\theta, \statrv_i)$ may be unbounded below and above.  Heuristically,
however, assuming that we have this exact expansion and that
$\robsol - \theta\opt = O_P(n^{-\half})$, then we can write the
expansions
\begin{align*}
  0 & = \nabla_\theta \risk_n(\robsol, \mc{P}_n) \\
    & = \nabla \frac{1}{n} \sum_{i = 1}^n \loss(\theta\opt, \statrv_i)
      + \nabla^2 \bigg(\frac{1}{n} \sum_{i = 1}^n \loss(\theta\opt, \statrv_i)\bigg)
      (\robsol - \theta\opt)
      + \nabla \sqrt{\frac{2 \tol \var_\emp(\loss(\robsol, \statrv))}{n}}
      + o_P(n^{-\half}) \\
    & = \frac{1}{n} \sum_{i = 1}^n \nabla \loss(\theta\opt, \statrv_i)
      + \nabla^2 \risk(\theta\opt)
      (\robsol - \theta\opt)
      + \nabla \sqrt{\frac{2 \tol \var(\loss(\theta\opt, \statrv))}{n}}
      + o_P(n^{-\half}).
\end{align*}
Multiplying by $\sqrt{n}$ and solving for $\robsol$ in the preceding
expression, computing $\nabla \sqrt{\var(\loss(\theta\opt, \statrv))}$
then yields the theorem.

The remainder of the proof makes this heuristic rigorous, and the outline
is as follows:
\begin{enumerate}[1.]
\item We show that there is a uniform expansion of the
  form~\eqref{eqn:uniform-exact-variance-expansion} in a neighborhood of
  $\theta\opt$.  (See Section~\ref{sec:local-variance-expansion}.)
\item Using the uniform expansion, we can then leverage standard techniques
  for asymptotic analysis of finite-dimensional estimators (see, e.g.\
  \citet{VanDerVaartWe96} or \citet{LehmannCa98}), which proceed by performing
  a Taylor expansion of the objective in a neighborhood of the optimum and
  using local asymptotic normality arguments. (See
  Section~\ref{sec:asymptotics-taylors}.)
\end{enumerate}

\subsection{The uniform variance expansion}
\label{sec:local-variance-expansion}

To lighten notation, we define a few quantities similar to those used in the
proof of Theorem~\ref{theorem:variance-expansion}. Let
\begin{equation*}
  Z(\theta) \defeq \loss(\theta, \statrv) - \E[\loss(\theta, \statrv)]
\end{equation*}
be the deviation of $\loss(\theta, \statrv)$ around its mean (the risk), and
similarly let $Z_i(\theta)$ be the version of this quantity for observation
$\statrv_i$. In addition, let $s_n^2(\theta) = \var_\emp(Z(\theta))$ be the
empirical variance of $Z(\theta)$, which
is identical to the empirical variance of $\loss(\theta, \statrv)$.

\providecommand{\eventmax}{\event_{n,\min}}

Now, recall the problem
\begin{equation*}
  \maximize_P ~ \E_P[Z(\theta)] ~~ \subjectto
  \phidivs{P}{\emp} \le \frac{\tol}{n},
\end{equation*}
and for each $\theta \in \Theta$, let $p(\theta) = \argmax_{p \in \mc{P}_n}
\sum_{i = 1}^n p_i Z_i(\theta)$ be the solution (probability) vectors.
Following expression~\eqref{eqn:when-choosing-u-is-possible} we see
for any $\epsilon \ge 0$ that
\begin{equation*}
  \min_{i \in [n]}
  \frac{\sqrt{2 \tol} (Z_i(\theta) - \wb{Z}(\theta))}{\sqrt{n} s_n(\theta)}
  \ge -1
  ~~~ \mbox{for~all~} \theta \in \theta\opt + \epsilon \ball
\end{equation*}
is sufficient for the exact variance expansion to hold. We now show that this
is indeed likely.  Let $\epsilon > 0$ be small enough that
Assumption~\ref{assumption:smoothness-assumptions} holds, that is, the random
Lipschitz function $\lipobj(\statrv)$ satisfies $|\loss(\theta, \statval) -
\loss(\theta', \statval)| \le \lipobj(\statval) \norms{\theta - \theta'}$ for
$\theta, \theta' \in \theta\opt + \epsilon \ball$.  Then because
\begin{align*}
  \left|\sqrt{n} s_n(\theta) - \sqrt{n} s_n(\theta')\right|
  & \le \sup_{u : \ltwo{u} \le 1}
  \sum_{i = 1}^n u_i \left(\loss(\theta, \statrv_i)
  - \loss(\theta', \statrv_i)\right) \\
  & \le \sup_{u : \ltwo{u} \le 1}
  \sum_{i = 1}^n u_i \lipobj(\statrv_i) \norm{\theta - \theta'}
  \le \sqrt{\sum_{i = 1}^n \lipobj^2(\statrv_i)} \norm{\theta - \theta'}
\end{align*}
so $\theta \mapsto s_n(\theta)$ is $\sqrt{\frac{1}{n} \sum_{i = 1}^n
  \lipobj(\statrv_i)^2}$-Lipschitz for $\theta \in \theta\opt + \epsilon
\ball$, we have
\begin{equation*}
  \inf_{\theta \in \theta\opt + \epsilon \ball}
  \min_{i \in [n]}
  \left\{\frac{\sqrt{2 \tol} (Z_i(\theta) - \wb{Z}(\theta))}{
    \sqrt{n} s_n(\theta)}\right\}
  \ge \min_{i \in [n]}
  \frac{\sqrt{2 \tol} (Z_i(\theta\opt) - \wb{Z}(\theta\opt)
    - 2 \epsilon \lipobj(\statrv_i))}{
    \sqrt{n \left(s_n(\theta\opt) - \epsilon \sqrt{\frac{1}{n}
        \sum_{j = 1}^n \lipobj(\statrv_j)^2}\right)}}.
\end{equation*}
Summarizing our development thus far, we have the following lemma.
\begin{lemma}
  \label{lemma:sufficient-uniform-expansion}
  Let the conditions of the previous paragraph hold. Then
  \begin{equation*}
    \min_{i \in [n]}
    \left\{\sqrt{2 \tol}
    (Z_i(\theta\opt) - \wb{Z}(\theta\opt)
    - 2 \epsilon \lipobj(\statrv_i))\right\}
    \ge \sqrt{n} \sqrt{s_n(\theta\opt) - \epsilon
      \bigg(\frac{1}{n} \sum_{i = 1}^n \lipobj(\statrv_i)^2\bigg)^\half}
  \end{equation*}
  implies that
  \begin{equation*}
    \risk_n(\theta, \mc{P}_n)
    = \E_\emp[\loss(\theta, \statrv)]
    + \sqrt{\frac{2 \tol}{n} \var_\emp(\loss(\theta, \statrv))}
    ~~ \mbox{for~all~} \theta \in \theta\opt + \epsilon \ball.
  \end{equation*}
\end{lemma}

Now, we use the following standard result to show that the conditions of
Lemma~\ref{lemma:sufficient-uniform-expansion} eventually hold with
probability one.
\begin{lemma}[Owen~\cite{Owen90}, Lemma 3]
  \label{lemma:two-moments-as-0}
  Let $Y_i$ be independent random variables with $\sup_i \E[Y_i^2] < \infty$.
  Then $n^{-\half} \max_{1 \le i \le n} |Y_i| \cas 0$.
\end{lemma}
\noindent
Based on Lemma~\ref{lemma:two-moments-as-0} and
the strong law of large numbers, we see immediately that
\begin{equation*}
  \frac{1}{\sqrt{n}} \max_{1 \le i \le n} |Z_i(\theta\opt)| \cas 0,
  ~~ \mbox{and} ~~
  \frac{1}{\sqrt{n}} \max_{1 \le i \le n} \lipobj(\statrv_i) \cas 0,
\end{equation*}
because $\E[Z(\theta\opt)^2] < \infty$ and
$\E[\lipobj(\statrv_i)^2] < \infty$.
Applying the strong law of large numbers to obtain
\begin{equation*}
  s_n(\theta\opt) \cas \sqrt{\var(\loss(\theta\opt, \statrv))}
  ~~ \mbox{and} ~~
  \sqrt{\frac{1}{n} \sum_{i = 1}^n \lipobj(\statrv_i)^2}
  \cas \sqrt{\E[\lipobj(\statrv)^2]},
\end{equation*}
we see immediately that for small enough $\epsilon > 0$, the
condition of Lemma~\ref{lemma:sufficient-uniform-expansion} holds
eventually with probability 1.
That is, the following uniform expansion holds.
\begin{lemma}
  \label{lemma:uniform-max-event-eventually}
  There exists $\epsilon > 0$ such that, with probability $1$,
  there exists an $N$ (which may be random) such that $n \ge N$ implies
  \begin{equation*}
    \risk_n(\theta, \mc{P}_n)
    = \E_\emp[\loss(\theta, \statrv)]
    + \sqrt{\frac{2 \tol \var_\emp(\loss(\theta, \statrv))}{n}}
    ~~ \mbox{for~all~} \theta \in \theta\opt + \epsilon \ball.
  \end{equation*}
\end{lemma}

\subsection{Asymptotics and Taylor expansions}
\label{sec:asymptotics-taylors}

Let $\exactevent$ be the event that the exact variance expansion of
Lemma~\ref{lemma:uniform-max-event-eventually} occurs
for $\theta \in \theta\opt + \epsilon \ball$. Now that we
know that $\P(\exactevent ~\mbox{eventually}) = 1$, we may perform a
few asymptotic expansions of the variance-regularized objective to
provide the convergence guarantees specified by the theorem. We use
the following lemma.
\begin{lemma}
  \label{lemma:as-convergence-theta-hat}
  Let the conditions of the theorem hold. If
  \begin{equation}
    \label{eqn:as-convergence-theta-hat}
    \robsol \in \argmin_\theta \risk_n(\theta, \mc{P}_n)
    ~~ \mbox{then} ~~
    \robsol \cas \theta\opt.
  \end{equation}
\end{lemma}
\noindent
The proof is standard, but for completeness we include it in
Section~\ref{sec:proof-of-as-convergence-theta-hat}. 

By combining Lemmas~\ref{lemma:uniform-max-event-eventually}
and~\ref{lemma:as-convergence-theta-hat}, we see that
with probability 1, for any $\epsilon > 0$, we eventually have both
\begin{equation*}
  \ltwos{\robsol - \theta\opt} < \epsilon
  ~~ \mbox{and} ~~
  \risk_n(\robsol, \mc{P}_n)
  = \E_\emp[\loss(\robsol, \statrv)]
  + \sqrt{\frac{2 \tol}{n} \var_\emp(\loss(\robsol, \statrv))}.
\end{equation*}
Assume for the remainder of the argument that both of these conditions hold.
Standard results on subdifferentiability of maxima of collections of convex
functions~\cite[Chapter X]{HiriartUrrutyLe93ab} give that
$\risk_n(\theta, \mc{P}_n)$ is differentiable near $\theta\opt$, and thus
\begin{align}
  \nonumber
  0 & = \nabla \risk_n(\robsol, \mc{P}_n)
      = \E_\emp[\nabla \loss(\robsol, \statrv)]
      + \nabla \sqrt{\frac{2 \tol}{n} \var_\emp(\loss(\robsol, \statrv))}
  \\
    & = \frac{1}{n} \sum_{i = 1}^n \nabla \loss(\robsol, \statrv_i)
      + \sqrt{\frac{2 \tol}{n}} \frac{\E_\emp\left[
      (\nabla \loss(\robsol, \statrv) - \E_\emp[\nabla
      \loss(\robsol, \statrv)]) (\loss(\robsol, \statrv)
      - \E_\emp[\loss(\robsol, \statrv)])\right]}{
      \sqrt{\var_\emp(\loss(\robsol, \statrv))}}.
    \label{eqn:zero-gradient-risk}
\end{align}
Because $\robsol \cas \theta\opt$, by the continuous mapping theorem
and local uniform convergence of the empirical expectations
$\E_\emp[\cdot]$ to $\E[\cdot]$,
the second term of expression~\eqref{eqn:zero-gradient-risk} satisfies
\begin{equation*}
  \frac{\E_\emp\left[
      (\nabla \loss(\robsol, \statrv) - \E_\emp[\nabla
      \loss(\robsol, \statrv)]) (\loss(\robsol, \statrv)
      - \E_\emp[\loss(\robsol, \statrv)])\right]}{
    \sqrt{\var_\emp(\loss(\robsol, \statrv))}}
  = \underbrace{\frac{\cov(\nabla \loss(\theta\opt, \statrv), \loss(\theta\opt,
      \statrv))}{\sqrt{\var(\loss(\theta\opt, \statrv))}}}_{\eqdef b(\theta\opt)}
  + o_P(1).
\end{equation*}
For simplicity, we let $b(\theta\opt)$ denote the final term, which we shall
see becomes an asymptotic bias.  Thus, performing a Taylor expansion of the
terms $\nabla \loss(\robsol, \statrv_i)$ around $\theta\opt$ in
equality~\eqref{eqn:zero-gradient-risk}, there
exist (random) error matrices $E_n(X_i)$, where
$\norm{E_n(X_i)} \le \liphess(X_i) \norms{\robsol - \theta\opt}$
by Assumption~\ref{assumption:smoothness-assumptions},
such that
\begin{align*}
  0 & = \E_\emp[\nabla \loss(\theta\opt, \statrv)]
  + \frac{1}{n} \sum_{i = 1}^n \left(\nabla^2 \loss(\theta\opt, \statrv_i)
  + E_n(X_i)\right)
  (\robsol - \theta\opt)
  + \sqrt{\frac{2 \tol}{n}} (b(\theta\opt) + o_P(1))  \\
  & = \E_\emp[\nabla \loss(\theta\opt, \statrv)]
  + \left(\nabla^2 \risk(\theta\opt)
  + o_P(1)\right)
  (\robsol - \theta\opt)
  + \sqrt{\frac{2 \tol}{n}} (b(\theta\opt) + o_P(1)).
\end{align*}
Multiplying both sides by $\sqrt{n}$, using that
$\nabla^2 \risk(\theta\opt) + o_P(1)$ is eventually invertible, and applying
the continuous mapping theorem, we have
\begin{equation*}
  \sqrt{n}(\robsol - \theta\opt)
  = -(\nabla^2 \risk(\theta\opt) + o_P(1))^{-1}
  \frac{1}{\sqrt{n}} \sum_{i = 1}^n \nabla \loss(\theta\opt, \statrv_i)
  - \sqrt{2 \tol} b(\theta\opt) + o_P(1).
\end{equation*}
The first term on the right side of the above display converges in
distribution to a $\normal(0, \Sigma)$ distribution, where
\begin{equation*}
  \Sigma = (\nabla^2 \risk(\theta\opt))^{-1}
  \cov(\nabla \loss(\theta\opt, \statrv))
  (\nabla^2 \risk(\theta\opt))^{-1},
\end{equation*}
so that
\begin{equation*}
  \sqrt{n}(\robsol - \theta\opt)
  \cd \normal\left(-\sqrt{2 \tol} \, b(\theta\opt),
    \Sigma\right)
\end{equation*}
as claimed in the theorem statement.


\section{Proofs of Technical Lemmas}
\label{appendix:sub-gaussians}

\subsection{Proof of Inequality~\eqref{eqn:rademacher-eigenvalues}}
\label{sec:proof-rademacher-eigenvalues}

Define the Gaussian complexity
\begin{equation}
  \label{eqn:loss-complexity-hilbert}
  \mathfrak{G}_n (\{\loss \circ \mc{H}\}_{\le r})
  \defeq \E\left[
    \sup_{h \in \ball_{\mc{H}}, c \in [0, 1]} \sum
    g_i c \loss(h(x_i), y_i)
    \mid \E[\loss(h(X), Y)^2] \le r / c^2 \right],
\end{equation}
where $g_i \simiid \normal(0, 1)$ (here we recall the standard
result~\cite{BartlettMe02} that Gaussian complexity upper bounds
Rademacher complexities up to a constant).
Now, the set $h - h\opt$ such that $h \in \ball_{\mc{H}}$ is
contained in $2 \ball_{\mc{H}}$, which is convex. Moreover,
we have $\E[\loss(h(X), Y)^2] =
\E[(h(X) - h\opt(X))^2] + \sigma^2$, and so we
have for any $c$ that
\begin{equation*}
  \{h \in \ball_{\mc{H}}
  \mid c^2 \E[\loss(h(X), Y)^2] \le r\}
  \subset
  \{h \in \ball_{\mc{H}}
  \mid \E[(h(X) - h\opt(X))^2] \le r / c^2 \},
\end{equation*}
and $\E[\loss(h(X), Y)^2] \le r / c^2$ also implies
$\sigma^2 \le r / c^2$.
Returning to expression~\eqref{eqn:loss-complexity-hilbert}
and enlarging the sets over which we take suprema,
we thus obtain
\begin{align*}
  \mathfrak{G}_n(\loss \circ \mc{H})
  & \le
  \E\left[
    \sup_{h \in \ball_{\mc{H}}, c_1, c_2 \in [0, 1]}
    \sum_{i=1}^n g_i |c_1 (h(x_i) - h\opt(x_i)) - c_2 \xi_i|
    \mid \E[(h(X) - h\opt(X))^2] \le \frac{r}{c_1^2},
    \sigma^2 \le \frac{r}{c_2^2} \right] \\
  & \le 
  \E\left[
    \sup_{f \in 2 \ball_{\mc{H}}, c \in [0, 1]}
    \sum_{i=1}^n g_i |f(x_i) - c \xi_i|
    \mid \E[f(X)^2] \le r,
    \sigma^2 \le r / c^2 \right],
\end{align*}
where we have used that $h - h\opt \in 2 \ball_{\mc{H}}$ and
that the set $\ball_{\mc{H}}$ is convex to obtain the
second inequality.
We now upper bound the
final display using the classical
Sudakov-Fernique comparison theorem~\cite[e.g.][]{Chatterjee05}.  Indeed,
define the two Gaussian processes indexed by $f \in \mc{H}$ and $c \in [0,
  1]$ by $Y_{f,c} = \sum_{i = 1}^n g_i |f(x_i) - c \xi_i|$ and
$Z_{f,c} = \sum_{i = 1}^n g_i f(x_i) + c \sum_{i = 1}^n w_i \xi_i$,
where $g_i \simiid \normal(0, 1)$ and $w_i \simiid \normal(0, 1)$.
Then we have for any $f_1, f_2 \in \mc{H}$ and
$c_1, c_2 \in [0, 1]$ that
\begin{align*}
  \E[(Y_{f_1, c_1} - Y_{f_2, c_2})^2]
  & = \sum_{i = 1}^n \left(|f_1(x_i) - c_1 \xi_i|
  - |f_2(x_i) - c_2 \xi_i|\right)^2 \\
  & \le \sum_{i = 1}^n \left(f_1(x_i) - f_2(x_i)
  + (c_2 - c_1) \xi_i\right)^2 \\
  & \le 2 \sum_{i = 1}^n (f_1(x_i) - f_2(x_i))^2
  + 2 (c_2 - c_1)^2 \sum_{i = 1}^n \xi_i^2.
\end{align*}
Moreover, $\E[(Z_{f_1, c_1} - Z_{f_2, c_2})^2]
= \sum_{i = 1}^n (f_1(x_i) - f_2(x_i))^2
+ (c_1 - c_2)^2 \sum_{i = 1}^n \xi_i^2$. Thus,
the Sudakov-Fernique inequality guarantees that
$\E[\sup_{f,c} Y_{f,c}] \le
\sqrt{2} \E[\sup_{f,c} Z_{f,c}]$, and
\begin{align*}
  \mathfrak{G}_n(\loss \circ \mc{H})
  & \lesssim
  \E\left[
    \sup_{f \in 2 \ball_{\mc{H}}}
    \sum_{i=1}^n g_i f(x_i)
    \mid \E[f(X)^2] \le r
    \right]
  + \E\left[\sup_{c \in [0, 1]}
    c \sum_{i = 1}^n w_i \xi_i
    \mid 
    c^2 \sigma^2 \le r\right].
\end{align*}
The last term in the expression has bound
$\sqrt{n r}$ by Jensen's inequality and the relaxation that
$c \in [-1, 1]$.
For the first term, \citet[Thm.~2.1]{Mendelson03} shows that
for RKHS with kernel eigenvalues $\lambda_1, \lambda_2, \ldots$,
we have
\begin{equation*}
  \E\left[
    \sup_{f \in 2 \ball_{\mc{H}}}
    \sum_{i=1}^n g_i f(X_i)
    \mid \E[f(X)^2] \le r
    \right]
  \lesssim
  \sqrt{n} \left(\sum_{j = 1}^\infty \min \{\lambda_j, r \}
  \right)^\half,
\end{equation*}
which yields our desired claim~\eqref{eqn:rademacher-eigenvalues}.

\subsection{Proof of Lemma~\ref{lemma:erm-sucks}}
\label{sec:proof-erm-sucks}

Defining
$N_y \defeq \card \{i \in [n] : \statrv_i = y\}$ for $y \in \{-1, 0, 1\}$, we immediately
obtain
\begin{equation*}
  \E_\emp[\loss(\theta; \statrv)]
  = \frac{1}{n} \left[
    N_{-1} |\theta + 1| + N_1 |\theta - 1| + N_0 |\theta|
    - (n - N_0) \right],
\end{equation*}
because $N_1 + N_{-1} + N_0 = n$. In particular, we find that the empirical
risk minimizer $\theta$ satisfies
\begin{equation*}
  \ermsol \defeq \argmin_{\theta \in \R} \E_\emp[\loss(\theta; \statrv)]
  = \begin{cases} 1 & \mbox{if}~ N_1 > N_0 + N_{-1} \\
    -1 & \mbox{if~} N_{-1} > N_0 + N_1 \\
    \in [-1, 1] & \mbox{otherwise.}
  \end{cases}
\end{equation*}
On the events $N_1 > N_{-1} + N_0$ or $N_{-1} > N_0 + N_1$, which are disjoint,
then, we have
\begin{equation*}
  \risk(\ermsol) = \delta = \risk(\theta\opt) + \delta.
\end{equation*}
Let us give a lower bound on the probability of this event. Noting that
marginally $N_1 \sim \binomial(n, \frac{1 - \delta}{2})$ and using
$N_0 + N_{-1} = n - N_1$, we have $N_1 > N_0 + N_{-1}$ if and only if
$N_1 > \frac{n}{2}$, and we would like to lower bound
\begin{equation*}
  \P\left(N_1 > \frac{n}{2}\right)
  = \P\left(\binomial\left(n, \frac{1 - \delta}{2}\right) > \frac{n}{2} \right)
  = \P\left(\binomial\left(n, \frac{1 + \delta}{2}\right) < \frac{n}{2} \right).
\end{equation*}
Letting $\Phi(t) = \frac{1}{\sqrt{2\pi}} \int_{-\infty}^t e^{-u^2/  2} du$ denote
the standard Gaussian CDF, then \citet{ZubkovSe13} show that
\begin{equation*}
  \P\left(N_1 \ge \frac{n}{2}\right)
  \ge \Phi\left(-\sqrt{2 n \dkl{\half}{\frac{1 + \delta}{2}}}\right)
\end{equation*}
where $\dkl{p}{q} = p \log \frac{p}{q} + (1 - p) \log \frac{1 - p}{1 - q}$
denotes the binary KL-divergence.  We have by standard bounds on
the KL-divergence~\cite[Lemma 2.7]{Tsybakov09} that
$\dkls{\half}{\frac{1 + \delta}{2}} \le \frac{\delta^2}{2 (1 - \delta^2)}$,
so that
\begin{equation*}
  \P\left(N_1 > \frac{n}{2}
    ~ \mbox{or} ~ N_{-1} > \frac{n}{2}\right)
  \ge 2 \Phi\left(-\sqrt{\frac{n \delta^2}{1 - \delta^2}}\right)
  - 2 \P\left(N_1 = \frac{n}{2}\right).
\end{equation*}
For $n$ odd, the final probability is 0, while for $n$ even, we have
\begin{equation*}
  \P\left(N_1 = \frac{n}{2}\right)
  = 2^{-n} \binom{n}{n/2} (1 - \delta^2)^{n/2}
  \le (1 - \delta^2)^{n/2}
  \sqrt{\frac{2}{\pi n}},
\end{equation*}
where the inequality uses that $\binom{2n}{n} \le \frac{4^n}{\sqrt{\pi n}}$ by
Stirling's approximation. Summarizing, we find that
\begin{equation*}
  \P\left(N_1 > \frac{n}{2}
    ~ \mbox{or} ~ N_{-1} > \frac{n}{2}\right)
  \ge 2 \Phi\left(-\sqrt{\frac{n \delta^2}{1 - \delta^2}}\right)
  - (1 - \delta^2)^{n/2} \sqrt{\frac{8}{\pi n}}.
\end{equation*}

\subsection{Proof of Lemma~\ref{lemma:as-convergence-theta-hat}}
\label{sec:proof-of-as-convergence-theta-hat}
Under the conditions of the theorem, the compactness of
$\theta\opt + \epsilon \ball$ guarantees that
\begin{equation*}
  \sup_{\theta \in \theta\opt + \epsilon \ball}
  |\E_\emp[\loss(\theta, \statrv)] - \risk(\theta)| \cas 0,
\end{equation*}
as the functions $\theta \mapsto \loss(\theta, \statval)$ are Lipschitz in a
neighborhood of $\theta\opt$ by
Assumption~\ref{assumption:smoothness-assumptions}. Similarly,
\begin{equation*}
  \sup_{\theta \in \theta\opt + \epsilon \ball}
  \left|\var_\emp(\loss(\theta, \statrv)) - \var(\loss(\theta,
    \statrv))\right|
  \cas 0,
\end{equation*}
using the local Lipschitzness of $\nabla^2 \loss$.  (See, for example, the
Glivenko-Cantelli results in Chapters 2.4--2.5 of \citet{VanDerVaartWe96}.)
Thus, using the two-sided bounds~\eqref{eqn:sure-variance-bounds} of
Theorem~\ref{theorem:variance-expansion}, we have that
\begin{align*}
  & \sup_{\theta \in \theta\opt + \epsilon \ball}
  \left|\risk_n(\theta, \mc{P}_n) - \risk(\theta)\right| \\
  & \le \sup_{\theta \in \theta\opt + \epsilon \ball}
    \left|\E_\emp[\loss(\theta, \mc{P}_n)] - \risk(\theta)\right|
    + \sqrt{\frac{2\tol}{n}} \sup_{\theta \in \theta\opt + \epsilon \ball}
  \sqrt{\var_\emp(\loss(\theta, \statrv))}
  \cas 0.
\end{align*}

Now, we use the fact that $\nabla^2 \risk(\theta\opt) \succ 0$, and that
$\theta \mapsto \nabla^2 \risk(\theta)$ is continuous in a neighborhood of
$\theta\opt$. Fix $\epsilon > 0$ small enough that the preceding uniform
convergence guarantees hold over $\theta\opt + 2 \epsilon \ball$ and
$\nabla^2 \risk(\theta) \succeq \lambda I$ for some $\lambda > 0$
and all $\theta \in \theta\opt + 2 \epsilon \ball$. Let
$\theta \not \in \theta\opt + \epsilon \ball$, but
$\theta \in \theta\opt + 2 \epsilon \ball$. Then for sufficently large
$n$, we have that
\begin{align*}
  \risk_n(\theta, \mc{P}_n)
  \ge \E_\emp[\loss(\theta, \statrv)]
  & \stackrel{(i)}{\ge} \risk(\theta) - \frac{\lambda}{4} \epsilon^2  \\
  & \stackrel{(ii)}{\ge} \risk(\theta\opt) + \frac{\lambda}{2}
    \ltwo{\theta - \theta\opt}^2 - \frac{\lambda}{4} \epsilon^2
    \stackrel{(iii)}{\ge} \risk(\theta\opt) + \frac{\lambda}{4} \epsilon^2 \\
  & \stackrel{(iv)}{\ge}
    \E_\emp[\loss(\theta\opt, \statrv)] + \frac{\lambda}{4} \epsilon^2
    - \frac{\lambda}{8} \epsilon^2
    = \E_\emp[\loss(\theta\opt, \statrv)] + \frac{\lambda}{8} \epsilon^2,
\end{align*}
where inequalities $(i)$ and $(iv)$ follow from the uniform convergence
guarantee,
inequality $(ii)$ from the strong convexity of $\risk$ near $\theta\opt$,
and $(iii)$ because $\ltwo{\theta - \theta\opt} \ge \epsilon$.
Finally, we have that
\begin{equation*}
  \E_\emp[\loss(\theta\opt, \statrv)]
  \ge \risk_n(\theta\opt, \mc{P}_n) - \underbrace{\sqrt{\frac{2 \tol}{n}
      \var_\emp(\loss(\theta\opt, \statrv))}}_{\cas 0},
\end{equation*}
so that eventually $\risk_n(\theta, \mc{P}_n) > \risk_n(\theta\opt,
\mc{P}_n)$
for all $\theta \in \theta\opt + 2 \epsilon \ball \setminus \epsilon
\ball$. By convexity, then this inequality holds for
all $\theta \not\in \theta\opt + \epsilon \ball$. Thus
if $\robsol \in \argmin_\theta \risk_n(\theta, \mc{P}_n)$,
then for any $\epsilon > 0$ we must eventually have
$\ltwos{\robsol - \theta\opt} < \epsilon$.

\section{Efficient solutions to computing the robust expectation}
\label{appendix:efficient-alg}

In this appendix, we give a detailed description of the procedure we use to
compute the supremum problem~\eqref{eqn:simple-problem}. In particular, our
procedure requires time $O(n \log n + \log \frac{1}{\epsilon} \log n)$, where
$\epsilon$ is the desired solution accuracy.  Let us reformulate this as a
minimization problem in a variable $p \in \R^n$ for simplicity. Then we wish
to solve
\begin{equation*}
  \minimize ~ p^\top z ~~ \subjectto ~
  \frac{1}{2n} \ltwo{np - \ones}^2 \le \tol,
  ~ p \ge 0,
  ~ p^\top \ones = 1.
\end{equation*}
We take a partial dual of this minimization problem, then maximize this dual
to find the optimizing $p$. Introducing the dual variable $\lambda \ge 0$ for
the constraint that $\half \ltwos{p - \frac{1}{n} \ones}^2 \le \frac{\tol}{n}$
and performing the standard min-max swap~\cite{BoydVa04} (strong duality
obtains for this problem because the Slater condition is satisfied by
$p = \frac{1}{n} \ones$) yields the maximization problem
\begin{equation}
  \label{eqn:partial-dual}
  \maximize_{\lambda \ge 0}
  ~ f(\lambda) \defeq
  \inf_p \left\{\frac{\lambda}{2} \ltwoBig{p - \frac{1}{n} \ones}^2
    - \frac{\lambda \tol}{n} + p^\top z
  \mid p \ge 0, ~ \ones^\top p = 1\right\}.
\end{equation}
If we can efficiently compute the infimum~\eqref{eqn:partial-dual}, then it is
possible to binary search over $\lambda$. Recall the standard
fact~\cite[Chapter VI.4.4]{HiriartUrrutyLe93ab} that for a collection
$\{f_p\}_{p \in \mc{P}}$ of concave functions, if the infimum
$f(x) = \inf_{p \in \mc{P}} f_p(x)$ is attained at some $p_0$ then any vector
$\nabla f_{p_0}(x)$ is a supergradient of $f(x)$.  Thus, letting $p(\lambda)$
be the (unique) minimizing value of $p$ for any $\lambda > 0$, the
objective~\eqref{eqn:partial-dual} becomes
$f(\lambda) = \frac{\lambda}{2} \ltwos{p(\lambda) - \frac{1}{n} \ones}^2 -
\frac{\lambda \tol}{n} + p(\lambda)^\top z$,
whose derivative with respect to $\lambda$ (holding $p$ fixed) is
$f'(\lambda) = \half \ltwos{p(\lambda) - \frac{1}{n} \ones}^2 -
\frac{\tol}{n}$.

Now we use well-known results on the Euclidean projection of a vector to the
probability simplex~\cite{DuchiShSiCh08} to provide an efficient computation
of the infimum~\eqref{eqn:partial-dual}. First, we assume with no loss of
generality that $z_1 \le z_2 \le \cdots \le z_n$ and that $\ones^\top z = 0$,
because neither of these changes the original optimization problem (as
$\ones^\top p = 0$ and the objective is symmetric). Then we define the two
vectors $s, \sigma^2 \in \R^n$, which we use for book-keeping in the algorithm,
by
\begin{equation*}
  s_i = \sum_{j \le i} z_j,
  ~~
  \sigma^2_i = \sum_{j \le i} z_j^2,
\end{equation*}
and we let $z^2$ be the vector whose entries are $z_i^2$. 
The infimum problem~\eqref{eqn:partial-dual} is equivalent to projecting the vector
$v(\lambda) \in \R^n$ defined by
\begin{equation*}
  v_i = \frac{1}{n} - \frac{1}{\lambda} z_i
\end{equation*}
onto the probability simplex. Notably~\cite{DuchiShSiCh08}, the projection
$p(\lambda)$ has the form $p_i(\lambda) = \hinge{v_i - \eta}$ for some
$\eta \in \R$, where $\eta$ is chosen such that
$\sum_{i = 1}^n p_i(\lambda) = 1$.  Finding such a value $\eta$ is
equivalent~\cite[Figure 1]{DuchiShSiCh08} to finding the unique index $i$ such
that
\begin{equation*}
  \sum_{j = 1}^i (v_j - v_i) < 1
  ~~ \mbox{and} ~~
  \sum_{j = 1}^{i + 1} (v_j - v_{i + 1}) \ge 1,
\end{equation*}
taking $i = n$ if no such index exists (the sum $\sum_{j = 1}^i (v_j - v_i)$
is increasing in $i$ and $v_1 - v_1 = 0$). Given the index $i$, algebraic
manipulations show that
$\eta = \frac{1}{n} - \frac{1}{i} - \frac{1}{i} \sum_{j = 1}^i z_j / \lambda =
\frac{1}{n} - \frac{1}{i} - \frac{1}{i} s_i / \lambda$
satisfies the equality $\sum_{i = 1}^n \hinge{v_i - \eta} = 1$ and that
$v_j - \eta \ge 0$ for all $j \le i$ while $v_j - \eta \le 0$ for $j > i$.  Of
course, given the index $i$ and $\eta$, we may calculate the derivative
$\frac{\partial}{\partial \lambda} f(\lambda)$ efficiently as well:
\begin{align*}
  f'(\lambda)
  & = \frac{\partial}{\partial\lambda}
    \left\{\frac{\lambda}{2} \ltwo{p(\lambda) - n^{-1} \ones}^2
    - \frac{\lambda \tol}{n} + p(\lambda)^\top z\right\} \\
  & = \half \ltwo{p(\lambda) - n^{-1} \ones}^2
    - \frac{\tol}{n}
    = \half \sum_{j = 1}^i
    (v_j - \eta - n^{-1})^2
    + \half \sum_{j = i + 1}^n
    \frac{1}{n^2}
    - \frac{\tol}{n} \\
  & = \half \sum_{j = 1}^i \left(\frac{1}{\lambda} z_j + \eta \right)^2
    + \frac{n - i}{2 n^2} - \frac{\tol}{n}
    = \frac{\sigma^2_i}{2 \lambda^2}
  + \frac{i \eta^2}{2} + \frac{s_i \eta}{\lambda}
  + \frac{n - i}{2 n^2}
  - \frac{\tol}{n}.
\end{align*}
Finding the index optimal $i$ can be done by a binary search, which requires
$O(\log n)$ time, and $f'(\lambda)$ is then computable in $O(1)$ time using
the vectors $s$ and $\sigma^2$. It is
then possible to perform a binary search over $\lambda$ using $f'(\lambda)$,
which which requires $\log \frac{1}{\epsilon}$ iterations to find $\lambda$
within accuracy $\epsilon$, from which it is easy to compute $p(\lambda)$ via
$p_i(\lambda) = \hinge{v_i - \eta} = \hinge{n^{-1} - \lambda^{-1} z_i -
  \eta}$.

We summarize this discussion with pseudo-code in
Figures~\ref{fig:find-optimal-p} and~\ref{fig:find-shifting-inds}, which
provide a main routine and sub-routine for finding the optimal vector
$p$. These routines show that, once provided the sorted vector $z$ with
$z_1 \le z_2 \le \cdots \le z_n$ (which requires $n \log n$ time to compute),
we require only $O(\log \frac{1}{\epsilon} \cdot \log n)$ computations.

\begin{figure}
  \algbox{Sorted vector $z \in \R^n$ with $\ones^\top z = 0$,
    parameter $\tol > 0$, solution accuracy $\epsilon$}{
    \textsc{Set} $\lambda_{\min} = 0$ and
    $\lambda_{\max} = \lambda_{\infty}
    = \max\{n \linf{z}, \sqrt{n / 2 \tol} \ltwo{z}\}$

    \textsc{Set} $s_i = \sum_{j \le i} z_j$ and $\sigma^2_i = \sum_{j \le i} z_j^2$

    \textsc{While} $|\lambda_{\max} - \lambda_{\min}| > \epsilon \lambda_{\infty}$

    \hspace{.5cm} \textsc{Set} $\lambda = \frac{\lambda_{\max} + \lambda_{\min}}{2}$

    \hspace{.5cm} \textsc{Set} $(\eta, i) = \textsc{FindShift}(z, \lambda, s)$
    ~~~ // (Figure~\ref{fig:find-shifting-inds})

    \hspace{.5cm} \textsc{Set}
    $f'(\lambda) = \frac{1}{2 \lambda^2} \sigma_i^2
    + \frac{\eta^2}{2} i^2 + \frac{\eta}{\lambda} s_i
    + \frac{n - i}{2 n^2} - \frac{\tol}{n}$

    \hspace{.5cm} \textsc{If} $f'(\lambda) > 0$

    \hspace{.8cm} \textsc{Set} $\lambda_{\min} = \lambda$

    \hspace{.5cm} \textsc{Else}

    \hspace{.8cm} \textsc{Set} $\lambda_{\max} = \lambda$

    \textsc{Set} $\lambda = \half (\lambda_{\max} + \lambda_{\min})$,
    $(\eta, i) = \textsc{FindShift}(z, \lambda, s)$
      
    \textsc{Set} $p_i = \hinge{\frac{1}{n} - \frac{1}{\lambda} z_i - \eta}$
    and \textsc{return} $p$
    }
  \caption{\label{fig:find-optimal-p} Procedure
  $\textsc{FindP}$ to find the vector $p$ minimizing
  $\sum_{i = 1}^n p_i z_i$ subject to the constraint
  $\frac{1}{2n} \ltwo{n p - \ones}^2\le \tol$.
  Method takes $\log \frac{1}{\epsilon}$ iterations of the loop.
}
\end{figure}

\begin{figure}[h!]
  \algbox{Sorted vector $z$ with $\ones^\top z = 0$,
    $\lambda > 0$, vector $s$ with
    $s_i = \sum_{j \le i} z_j$}{
    \textsc{Set} $\ilow = 1, \ihigh = n$

    \textsc{If} $\frac{1}{n} - \frac{z_n}{\lambda} \ge 0$

    \hspace{.5cm} \textsc{Return} $(\eta = 0, i = n)$
    
    \textsc{While} $\ilow \neq \ihigh$

    \hspace{.5cm} $i = \half (\ilow + \ihigh)$

    \hspace{.5cm} $s_{\rm left} = \frac{1}{\lambda} (i z_i - s_i)$
    ~~~~~~~~ // (this is $s_{\rm left} = \sum_{j = 1}^i (v_j - v_i)$)

    \hspace{.5cm} $s_{\rm right} = \frac{1}{\lambda} ((i + 1) z_{i + 1} - s_{i + 1})$
    ~~~ // (this is $s_{\rm right} = \sum_{j = 1}^{i + 1} (v_j - v_{i + 1})$)

    \hspace{.5cm} \textsc{If} $s_{\rm right} \ge 1$ \textsc{and} $s_{\rm left} < 1$

    \hspace{.8cm} \textsc{Set} $\eta = \frac{1}{n} - \frac{1}{i} - \frac{1}{\lambda i} s_i$
    and \textsc{return} $(\eta, i)$

    \hspace{.5cm} \textsc{Else if} $s_{\rm left} \ge 1$

    \hspace{.8cm} \textsc{Set} $\ihigh = i - 1$

    \hspace{.5cm} \textsc{Else}

    \hspace{.8cm} \textsc{Set} $\ilow = i + 1$

    \textsc{Set} $i = \ilow$ and $\eta = \frac{1}{n} - \frac{1}{i} - \frac{1}{\lambda i} s_i$
    and
    \textsc{return} $(\eta, i)$
  }
  \caption{\label{fig:find-shifting-inds} Procedure \textsc{FindShift} to find
    index $i$ and parameter $\eta$ such that, for the definition $v_i =
    \frac{1}{n} - \frac{1}{\lambda} z_i$, we have $v_j - \eta \ge 0$ for $j
    \le i$, $v_j - \eta \le 0$ for $j > i$, and $\sum_{j = 1}^n \hinge{v_j -
      \eta} = 1$.  Method requires time $O(\log n)$.}
\end{figure}

\ifdefined\usejmlrstyle

\bibliography{bib}

\else

\setlength{\bibsep}{.2em}
\bibliography{bib}
\bibliographystyle{abbrvnat}

\fi

\end{document}